\documentclass{amsart}
\usepackage{graphicx}
\usepackage{amssymb}
\usepackage{amsmath}
\usepackage{mathrsfs}
\usepackage{subfigure}
\usepackage{verbatim}
\usepackage{graphicx}
\usepackage{epsfig}
\usepackage{longtable}
\usepackage{epsfig,indentfirst,latexsym,bm,amsmath,amssymb,amsfonts,lscape,rawfonts,eufrak}
\usepackage{color}
\newtheorem{theorem}{Theorem}[section]

\newtheorem{remark}[theorem]{Remark}

\newcommand{\beq}{\begin{equation}}
\newcommand{\eeq}{\end{equation}}
\newcommand{\beqa}{\begin{eqnarray}}
\newcommand{\eeqa}{\end{eqnarray}}
\newcommand{\beqas}{\begin{eqnarray*}}
\newcommand{\eeqas}{\end{eqnarray*}}
\newcommand{\ba}{\begin{array}}
\newcommand{\ea}{\end{array}}
\newcommand{\bi}{\begin{itemize}}
\newcommand{\ei}{\end{itemize}}
\newcommand{\gap}{\hspace*{2em}}

\numberwithin{equation}{section}

\def\Argmin{{\rm Argmin}}

\def\vgap{\vspace*{.1in}}
\def\QED{\ifhmode\unskip\nobreak\fi\ifmmode\ifinner\else\hskip5pt\fi\fi
  \hbox{\hskip5pt\vrule width5pt height5pt depth1.5pt\hskip1pt}}
\def\Arg{{\rm Arg}}

\def\bJ{{\bar J}}

\def\cS{{\mathcal S}}

\def\cX{{\mathcal X}}
\def\cY{{\mathcal Y}}

\def\eps{{\epsilon}}
\def\feas{{\rm feas}}

\def\ph{{\partial h}}
\def\pg{{\partial g}}

\def\tu{{\tilde u}}
\def\vrho{{\varrho}}

\def\R{{{\mathbb R}}}

\def\cI{\mathcal{I}}

\begin{document}

\title{$\ell_0$ Minimization for Wavelet Frame Based Image Restoration
}

\author{Yong Zhang}
\address{Department of Mathematics,
    Simon Fraser University, Burnaby, BC, V5A 1S6,
    Canada.} \email{yza30@sfu.ca}

\author{Bin Dong}
\address{Department of Mathematics,
The University of Arizona, 617 N. Santa Rita Ave., Tucson,
Arizona, 85721-0089} \email{dongbin@math.arizona.edu}

\author{Zhaosong Lu}
\address{Department of Mathematics,
    Simon Fraser University, Burnaby, BC, V5A 1S6,
    Canada.} \email{zhaosong@sfu.ca}



\thanks{The first and third authors were supported in part by NSERC Discovery Grant.}

\maketitle

\begin{abstract}

The theory of (tight) wavelet frames has been extensively studied in the past
twenty years and they are currently widely used for image restoration and
other image processing and analysis problems. The success of wavelet frame
based models, including balanced approach \cite{chan2003wavelet,CCSS} and
analysis based approach \cite{cai2009split,elad2005simultaneous,
starck2005image}, is due to their capability of sparsely approximating
piecewise smooth functions like images. Motivated by the balanced approach
and analysis based approach, we shall propose a wavelet frame based $\ell_0$
minimization model, where the $\ell_0$ ``norm" of the frame coefficients is
penalized. We adapt the penalty decomposition (PD) method of
\cite{LuZhangTech2010} to solve the proposed optimization problem. Some
convergence analysis of the adapted PD method will also be provided.
Numerical results showed that the proposed model solved by the PD method can
generate images with better quality than those obtained by either analysis
based approach or balanced approach in terms of restoring sharp features as
well as maintaining smoothness of the recovered images.

\end{abstract}

\section{Introduction} \label{introduction}

Mathematics has been playing an important role in the modern developments of
image processing and analysis. Image restoration, including image denoising,
deblurring, inpainting, tomography, etc., is one of the most important areas
in image processing and analysis. Its major purpose is to enhance the quality
of a given image that is corrupted in various ways during the process of
imaging, acquisition and communication; and enable us to see crucial but
subtle objects residing in the image. Therefore, image restoration is an
important step to take towards accurate interpretations of the physical world
and making optimal decisions.

\subsection{Image Restoration}

Image restoration is often formulated as a linear inverse problem.
For the simplicity of the notations, we denote the images as
vectors in $\R^n$ with $n$ equals to the total number of pixels. A
typical image restoration problem is formulated as \beq
\label{model}
 f= Au+\eta,
\eeq where $f \in \R^d$ is the observed image (or measurements),
$\eta$ denotes white Gaussian noise with variance $\sigma^2$, and
$A \in \R^{d \times n}$ is some linear operator. The objective is
to find the unknown true image $u\in\R^n$ from the observed image
$f$. Typically, the linear operator in \eqref{model} is a
convolution operator for image deconvolution problems, a
projection operator for image inpainting and partial Radon
transform for computed tomography.

To solve $u$ from \eqref{model}, one of the most natural choices
is the following least square problem
$$\min_{u\in\R^n} \|Au-f\|_2^2,$$ where $\|\cdot\|_2$ denotes the $\ell_2$-norm.
This is, however, not a good idea in general. Taking image
deconvolution problem as an example, since the matrix $A$ is
ill-conditioned, the noise $\eta$ possessed by $f$ will be
amplified after solving the above least squares problem.
Therefore, in order to suppress the effect of noise and also
preserve key features of the image, e.g., edges, various
regularization based optimization models were proposed in the
literature. Among all regularization based models for image
restoration, variational methods and wavelet frames based
approaches are widely adopted and have been proven successful.

The trend of variational methods and partial differential equation
(PDE) based image processing started with the refined
Rudin-Osher-Fatemi (ROF) model \cite{ROF} which
penalizes the total variation (TV) of $u$. Many of the current PDE
based methods for image denoising and decomposition utilize TV
regularization for its beneficial edge preserving property (see
e.g., \cite{meyer2001oscillating,sapiro2001geometric,OF}). The ROF
model is especially effective on restoring images that are
piecewise constant, e.g., binary images. Other types of variational
models were also proposed after the ROF model. We refer the
interested readers to
\cite{geman1995nonlinear,chambolle1997image,meyer2001oscillating,OF,
chan2006total,aubert2006mathematical,ChanShen,wang2008new} and
the references therein for more details.

Wavelet frame based approaches are relatively new and came from a different
path. The basic idea for wavelet frame based approaches is that images can be
sparsely approximated by properly designed wavelet frames, and hence, the
regularization used for wavelet frame based models is the $\ell_1$-norm of
frame coefficients. Although wavelet frame based approaches take similar
forms as variational methods, they were generally considered as different
approaches than variational methods because, among many other reasons,
wavelet frame based approaches is defined for discrete data, while
variational methods assume all variables are functions. Some studies in the
literature (see for example \cite{steidl2005equivalence}) indicated that
there was a relation between Haar wavelet and total variation. However, it
was not clear if there exists a general relation between wavelet frames and
variational models (with general differential operators) in the context of
image restorations. In a recent paper \cite{CDOS2011}, the authors
established a rigorous connection between one of the wavelet frame based
approaches, namely the analysis based approach, and variational models. It
was shown in \cite{CDOS2011} that the analysis based approach can be regarded
as a finite difference approximation of a certain type of general variational
model, and such approximation will be exact when image resolution goes to
infinity. Furthermore, through Gamma-convergence, the authors showed that the
solutions of the analysis based approach also approximate the solutions of
the corresponding variational model. Such connections not only grant
geometric interpretation to wavelet frame based approaches, but also lead to
even wider applications of them, e.g., image segmentation \cite{Dong2010Seg}
and 3D surface reconstruction from unorganized point sets
\cite{Dong2011SurfRec}. On the other hand, the discretization provided by
wavelet frames was shown, in e.g.,
\cite{chan2003wavelet,chan2005frameletreport,
cai2009linearized,cai2009split,CDOS2011,Dong2010IASNotes}, to be superior
than the standard discretizations for some of the variational models, due to
the multiresolution structure and redundancy of wavelet frames which enable
wavelet frame based models to adaptively choose a proper differential
operators in different regions of a given image according to the order of the
singularity of the underlying solutions. For these reasons, as well as the
fact that digital images are always discrete, we use wavelet frames as the
tool for image restoration in this paper.

\subsection{Wavelet Frame Based Approaches}
\label{Wavelet}

We now briefly introduce the concept of tight frames and tight
wavelet frame, and then recall some of the frame based image
restoration models. Interesting readers should consult
\cite{ron1997affine,Dau,Daubechies2003} for theories of frames and
wavelet frames, \cite{ShenICM2010} for a short survey on theory
and applications of frames, and \cite{Dong2010IASNotes} for a more
detailed survey.

A countable set $X\subset L_2(\mathbb{R})$ is called a tight frame
of $L_2(\mathbb{R})$ if
\begin{equation*}
f=\sum_{h\in X}\langle f,h \rangle h \quad \forall f \in
L_2(\mathbb{R}),
\end{equation*}
where $\langle \cdot,\cdot \rangle$ is the inner product of
$L_2(\mathbb{R})$. The tight frame $X$ is called a tight wavelet
frame if the elements of $X$ are generated by dilations and
translations of finitely many functions called framelets. The
construction of framelets can be obtained by the unitary extension
principle (UEP) of \cite{ron1997affine}. In our implementations,
we will mainly use the piecewise linear B-spline framelets constructed by
\cite{ron1997affine}. Given a 1-dimensional framelet system for
$L_2(\mathbb{R})$, the $s$-dimensional tight wavelet frame system
for $L_2(\mathbb{R}^s)$ can be easily constructed by using tensor
products of 1-dimensional framelets (see e.g., \cite{Dau,
Dong2010IASNotes}).

In the discrete setting, we will use $W\in \R^{m\times n}$ with
$m\ge n$ to denote fast tensor product framelet decomposition and
use $W^\top$ to denote the fast reconstruction. Then by the
unitary extension principle \cite{ron1997affine}, we have $W^\top
W=I$, i.e., $u=W^\top Wu$ for any image $u$. We will further denote
an $L$-level framelet decomposition of $u$ as
\begin{equation*}
Wu=\left(\ldots, {W}_{l,j}u,\ldots\right)^\top \quad\mbox{for }
0\le l\le L-1, j\in \cI,
\end{equation*}
where $\cI$ denotes the index set of all framelet bands and
$W_{l,j}u\in\R^n$. Under such notation, we have $m=L\times
|\cI|\times n$. We will also use $\alpha\in\R^m$ to denote the
frame coefficients, i.e., $\alpha=Wu$, where
$$\alpha=\left(\ldots, \alpha_{l,j},\ldots\right)^\top,\quad\mbox{with} \
\alpha_{l,j}=W_{l,j}u.$$ More details on discrete algorithms
of framelet transforms can be found in \cite{Dong2010IASNotes}.

Since tight wavelet frame systems are redundant systems (i.e.,
$m>n$), the representation of $u$ in the frame domain is
not unique. Therefore, there are mainly three formulations
utilizing the sparseness of the frame coefficients, namely,
analysis based approach, synthesis based approach, and balanced
approach. Detailed and integrated descriptions of these three
methods can be found in \cite{Dong2010IASNotes}.

The wavelet frame based image processing started from
\cite{chan2003wavelet,chan2004tight} for high-resolution image
reconstructions, where the proposed algorithm was later analyzed
in \cite{CCSS}. These work lead to the following balanced approach
\cite{cai2008simultaneous} \beq \label{l1b}
\min\limits_{\alpha \in \R^m} \frac{1}{2}\|AW^\top\alpha-f\|^2_D +
\frac{\kappa}{2}\|(I-WW^\top)\alpha\|^2_2
 +\left\|\sum_{l=0}^{L-1} \left(\sum_{j\in\cI}\lambda_{l,j}|\alpha_{l,j}|^p\right)^{1/p}\right\|_1,
\eeq where $p=1$ or 2, $0\leq\kappa\leq\infty$,
$\lambda_{l,j}\ge0$ is a scalar parameter, and $\|\cdot\|_D$
denotes the weighted $\ell_2$-norm with $D$ positive definite. This
formulation is referred to as the balanced approach because it
balances the sparsity of the frame coefficient and the smoothness
of the image. The balanced approach \eqref{l1b} was applied to
various applications in
\cite{chai2007deconvolution,chan2007frameletvideo,shenAPG2009,Xun2010FrameCBCT}.

When $\kappa=0$, only the sparsity of the frame coefficient is
penalized. This is called the synthesis based approach, as the
image is synthesized by the sparsest coefficient vector(see e.g.,
\cite{daubechies2007iteratively,fadili5sparse,fadili2009inpainting,figueiredo2003algorithm,figueiredo2005bound}).
When $\kappa=+\infty$, only the sparsity of canonical wavelet
frame coefficients, which corresponds to the smoothness of the
underlying image, is penalized. For this case, problem
\eqref{l1b} can be rewritten as \beq \label{l1a}
\min\limits_{u \in \R^n} \frac{1}{2}\|Au-f\|^2_D +
\left\|\sum_{l=0}^{L-1}
\left(\sum_{j\in\cI}\lambda_{l,j}|W_{l,j}u|^p\right)^{1/p}\right\|_1.
\eeq This is called the analysis based approach, as the
coefficient is in range of the analysis operator (see, for example,
\cite{cai2009split,elad2005simultaneous, starck2005image}).

Note that if we take $p=1$ for the last term of \eqref{l1b} and
\eqref{l1a}, it is known as the anisotropic $\ell_1$-norm of the
frame coefficients, which is the case used for earlier frame based
image restoration models. The case $p=2$, called isotropic
$\ell_1$-norm of the frame coefficients, was proposed in
\cite{CDOS2011} and was shown to be superior than anisotropic
$\ell_1$-norm. Therefore, we will choose $p=2$ for our
simulations.

\subsection{Motivations and Contributions}

For most of the variational models and wavelet frame based
approaches, the choice of norm for the regularization term is the
$\ell_1$-norm. Taking wavelet frame based approaches for example,
the attempt of minimizing the $\ell_1$-norm of the frame
coefficients is to increase their sparsity, which is the right
thing to do since piecewise smooth functions like images can be
sparsely approximated by tight wavelet frames. Although the
$\ell_1$-norm of a vector does not directly correspond to its
cardinality in contrast to $\ell_0$ ``norm", it can be regarded as a
convex approximation to $\ell_0$ ``norm". Such approximation is also
an excellent approximation for many cases. It was shown by
\cite{candes2010compressed}, which generalizes the exciting
results of compressed sensing \cite{CRT,CT1,CT,Do2}, that for a
given wavelet frame, if the operator $A$ satisfies certain
conditions, and if the unknown true image can be sparsely
approximated by the given wavelet frame, one can robustly recover
the unknown image by penalizing the $\ell_1$-norm of the frame
coefficients.

For image restoration, however, the conditions on $A$ as required
by \cite{candes2010compressed} are not generally satisfied, which
means penalizing $\ell_0$ ``norm" and $\ell_1$-norm may produce
different solutions. Although both the balanced approach
\eqref{l1b} and analysis based approach \eqref{l1a} can generate
restored images with very high quality, one natural question is
whether using $\ell_0$ ``norm" instead of $\ell_1$-norm can
further improve the results.

On the other hand, it was observed, in e.g.,
\cite{Dong2010IASNotes} (also see Figure \ref{image4} and Figure
\ref{image5}), that balanced approach \eqref{l1b} generally
generates images with sharper features like edges than the
analysis based approach \eqref{l1a}, because balanced approach
emphasizes more on the sparsity of the frame coefficients. However,
the recovered images from balanced approach usually contains more
artifact (e.g., oscillations) than analysis based approach, because
the regularization term of the analysis based approach has a
direct link to the regularity of $u$ (as proven by
\cite{CDOS2011}) comparing to balanced approach. Although such
trade-off can be controlled by the parameter $\kappa$ in the
balanced approach \eqref{l1b}, it is not very easy to do in
practice. Furthermore, when a large $\kappa$ is chosen, some of
the numerical algorithms solving \eqref{l1b} will converge slower
than choosing a smaller $\kappa$ (see e.g.,
\cite{shenAPG2009,Dong2010IASNotes}).

Since penalizing $\ell_1$-norm of $Wu$ ensures smoothness while not as much
sparsity as balanced approach, we propose to penalize $\ell_0$ ``norm" of
$Wu$ instead. Intuitively, this should provide us a balance between sharpness
of the features and smoothness for the recovered images. The difficulty here
is that $\ell_0$ minimization problems are generally hard to solve. Recently,
penalty decomposition (PD) methods were proposed by \cite{LuZhangTech2010}
for a general $\ell_0$ minimization problem that can be used to solve our
proposed model due to its generality. Computational results of
\cite{LuZhangTech2010} demonstrated that their methods generally outperform
the existing methods for compressed sensing problems, sparse logistic
regression and sparse inverse covariance selection problems in terms of
quality of solutions and/or computational efficiency. This motivates us to
adapt one of their PD methods to solve our proposed $\ell_0$ minimization
problem. Same as proposed in \cite{LuZhangTech2010}, the block coordinate
descent (BCD) method is used to solve each penalty subproblem of the PD
method. However, the convergence analysis of the BCD method was missing from
\cite{LuZhangTech2010} when $\ell_0$ ``norm'' appears in the objective
function. Indeed, the convergence of the BCD method generally requires the
continuity of the objective function as discussed in
\cite{tseng2001convergence}. In addition, the BCD method for the optimization
problem with the nonconvex objective function has only been proved to
converge to a stationary point which is not a local minimizer in general (see
\cite{tseng2001convergence} for details).

\subsection*{Contributions}
The main contributions of this paper are summarized as follows.
\begin{itemize}
\item[1)] We propose a new wavelet frame based model for image
    restoration problems that penalizes the $\ell_0$ ``norm'' of the
    wavelet frame coefficients. Numerical simulations show that the PD
    method that solves the proposed model generates recovered images with
    better quality than those obtained by either balanced approach and
    analysis based approach.
\item[2)] Given the discontinuity and nonconvexity of the $\ell_0$
    ``norm'' term in the objective function, we have proved some
    convergence results for the BCD method which is missing from the
    literature.
\end{itemize}

We now leave the details of the model and algorithm to Section \ref{method} and
details of simulations to Section \ref{results}.

\section{Model and Algorithm}\label{method}

We start by introducing some simple notations. The space of
symmetric $n \times n$ matrices will be denoted by $\cS^n$. If $X
\in \cS^n$ is positive definite, we write $X \succ 0$. We denote
by $I$ the identity matrix, whose dimension should be clear from
the context. Given an index set $J \subseteq \{1,\ldots,n\}$,
$x_J$ denotes the sub-vector formed by the entries of $x$ indexed
by $J$. For any real vector, $\|\cdot\|_0$ and $\|\cdot\|_2$ denote
the cardinality (i.e., the number of nonzero entries) and the
Euclidean norm of the vector, respectively. In addition, $\|x\|_D$
denotes the weighted $\ell_2$-norm defined by $\|x\|_D =\sqrt{x^\top Dx}$
with $D \succ 0$.

\subsection{Model}

We now propose the following optimization model for image restoration
problems, \beq \label{l0A} \min\limits_{u \in \cY} \frac{1}{2}\|Au-f\|^2_D +
\sum_{\bm{i}} \lambda_{\bm{i}}\|(Wu)_{\bm{i}}\|_0, \eeq where $\cY$ is some
convex subset of $\R^n$. Here we are using the multi-index $\bm{i}$ and
denote $(Wu)_{\bm{i}}$ (similarly for $\lambda_{\bm{i}}$) the value of $Wu$
at a given pixel location within a certain level and band of wavelet frame
transform. Comparing to the analysis based model, we are now penalizing the
number of nonzero elements of $Wu$. As mentioned earlier that if we emphasize
too much on the sparsity of the frame coefficients as in the balanced
approach or synthesis based approach, the recovered image will contain
artifacts, although features like edges will be sharp; if we emphasize too
much on the regularity of $u$ like in analysis based approach, features in
the recovered images will be slightly blurred, although artifacts and noise
will be nicely suppressed. Therefore, by penalizing the $\ell_0$ ``norm" of
$Wu$ as in \eqref{l0A}, we can indeed achieve a better balance between
sharpness of features and smoothness of the recovered images.

Given that the $\ell_0$ ``norm" is an integer-valued, discontinuous and
nonconvex function, problem \eqref{l0a} is generally hard to solve. Some
algorithms proposed in the literature, e.g., iterative hard thresholding
algorithms \cite{BD08,blumensath2009iterative,he06sp}, cannot be directly
applied to the proposed model \eqref{l0A} unless $W=I$. Recently, Lu and
Zhang \cite{LuZhangTech2010} proposed a penalty decomposition (PD) method to
solve the following general $\ell_0$ minimization problem: \beq \label{l0}
\min\limits_{x \in \cX} f(x) + \nu \|x_J\|_0 \eeq for some $\nu > 0$
controlling the sparsity of the solution, where $\cX$ is a closed convex set
in $\R^n$, $f: \R^n \to \R$ is a continuously differentiable function, and
$\|x_J\|_0$ denotes the cardinality of the subvector formed by the entries of
$x$ indexed by $J$. In view of \cite{LuZhangTech2010}, we reformulate
\eqref{l0A} as \beq \label{l0a} \min\limits_{u \in \cY, \alpha = Wu}
\frac{1}{2}\|Au-f\|^2_D + \sum_{\bm{i}} \lambda_{\bm{i}}\|\alpha_{\bm{i}}\|_0
\eeq and then we can adapt the PD method of \cite{LuZhangTech2010} to tackle
problem \eqref{l0A} directly. Same as proposed in \cite{LuZhangTech2010}, the
BCD method is used to solve each penalty subproblem of the PD method. In
addition, we apply the non-monotone gradient projection method proposed in
\cite{birgin2000nonmonotone} to solve one of the subproblem in the BCD
method.

\subsection{Algorithm for Problem
\eqref{l0a}}

In this section, we discuss how the PD method proposed in
\cite{LuZhangTech2010} solving \eqref{l0} can be adapted to solve problem
\eqref{l0a}. Letting $x = ( u_1, \ldots, u_n, \alpha_1, \ldots,\alpha_m)$, $J
=\{n+1,\ldots,n+m\}$, $\bJ = \{1,\ldots,n\}$, $f(x) = \frac{1}{2} \|Ax_{\bJ}
-f\|^2_D$ and $\cX = \{x \in \R^{n+m}:x_{J} = Wx_{\bJ} \ \mbox{and} \ x_{\bJ}
\in \cY\}$, we can clearly see that the problem \eqref{l0a} takes the same
form as \eqref{l0}. In addition, there obviously exists a feasible point
$(u^{\feas},\alpha^{\feas})$ for problem \eqref{l0a} when $\cY \neq
\emptyset$, i.e. there exist $(u^{\feas},\alpha^{\feas})$ such that
$Wu^{\feas}=\alpha^{\feas}$ and $u^{\feas}\in\cY$. In particular, we can
choose $(u^{\feas},\alpha^{\feas})=(0,0)$, which is the choice we make for
our numerical studies. We now discuss the implementation details of the PD
method when solving the proposed wavelet frame based model \eqref{l0a}.

Given a penalty parameter $\vrho>0$, the associated quadratic
penalty function for \eqref{l0a} is defined as \beq \label{l0rp}
p_{\vrho}(u,\alpha) := \frac{1}{2}\|Au-f\|^2_D + \sum_{\bm{i}}
\lambda_{\bm{i}}\|\alpha_{\bm{i}}\|_0 + \frac{\vrho}{2} \|Wu -
\alpha\|^2_2. \eeq Then we have the following PD method for
problem \eqref{l0a} where each penalty subproblem is approximately
solved by a BCD method (see \cite{LuZhangTech2010} for details).

\gap

\noindent
{\bf Penalty Decomposition (PD) Method for \eqref{l0a}:}  \\
[5pt] Let $\vrho_0 >0$, $\delta > 1$ be given. Choose an arbitrary
$\alpha^{0,0} \in \R^m$ and a constant $\Upsilon$ such that
$\Upsilon \ge \max\{\frac{1}{2}\|Au^\feas-f\|^2_D + \sum_{\bm{i}}
\lambda_{\bm{i}}\|\alpha^\feas_{\bm{i}}\|_0, \min_{u \in \cY}
p_{\vrho_0}(u,\alpha^{0,0})\}$. Set $k=0$.
\begin{itemize}
\item[1)] Set $q=0$ and apply the BCD method to find an
approximate solution $(u^k, \alpha^k) \in \cY \times \R^m$ for the
penalty subproblem \beq \label{inner-prob2}
\min\{p_{\vrho_k}(u,\alpha): \ u \in \cY, \ \alpha \in \R^m\} \eeq
by performing steps 1a)-1d): \bi \item[1a)] Solve $u^{k,q+1} \in
\Arg\min\limits_{u \in \cY} p_{\vrho_k}(u,\alpha^{k,q})$.

\item[1b)] Solve $\alpha^{k,q+1} \in \Arg\min\limits_{\alpha \in
\R^n} p_{\vrho_k}(u^{k,q+1},\alpha)$. \item[1c)] If $(u^{k,q+1},
\alpha^{k,q+1})$ satisfies the stopping criteria of the BCD method, set $(u^k,
\alpha^k) := (u^{k,q+1},\alpha^{k,q+1})$ and go to step 2).
\item[1d)] Otherwise, set $q \leftarrow q+1$ and go to step 1a). \ei
    \item[2)] If $(u^k,\alpha^k)$ satisfies the stopping criteria of the
    PD method, stop and output $u^k$. Otherwise, set$\vrho_{k+1} :=
    \delta\vrho_k$.
    \item[3)] If $\min\limits_{u \in \cY} p_{\vrho_{k+1}}(u,\alpha^k) >
    \Upsilon$, set $\alpha^{k+1,0} := \alpha^{\feas}$. Otherwise, set
    $\alpha^{k+1,0} := \alpha^k$.
\item[4)] Set $k \leftarrow k+1$ and go to step 1).
\end{itemize}
\noindent
{\bf end}
\vgap

\begin{remark}
In the practical implementation, we terminate the inner iterations of the BCD method
based on the relative progress of $p_{\vrho_k}(u^{k,q},\alpha^{k,q})$
which can be described as follows: \beq \label{inner-term}
\frac{|p_{\vrho_k}(u^{k,q},\alpha^{k,q})-p_{\vrho_k}(u^{k,q+1},\alpha^{k,q+1})|}
           {\max(|p_{\vrho_k}(u^{k,q+1},\alpha^{k,q+1})|,1)} \ \leq \  \eps_I. 
\eeq
Moreover, we terminate the outer iterations of the PD method once
\beq \label{outer-term}
\frac{\|Wu^k-\alpha^k\|_2}{\max(|p_{\vrho_k}(u^k,\alpha^k)|,1)} \ \leq \ \eps_O. 
\eeq
\end{remark}
Next we discuss how to solve two subproblems arising in step 1a)
and 1b) of the PD method.

\subsubsection{The BCD subproblem in step 1a)} The BCD subproblem
in step 1a) is in the form of \beq \label{sub2} \min\limits_{u \in
\cY} \frac{1}{2}\langle u, Qu \rangle - \langle c, u \rangle \eeq
for some $Q \succ 0$ and $c \in \R^n$. Obviously, when $\cY =
\R^n$, problem \eqref{sub2} is an unconstrained quadratic
programming problem that can be solved by the conjugate gradient
method. Nevertheless, the pixel values of an image are usually
bounded. For example, the pixel values of a CT image should be
always greater than or equal to zero and the pixel values of a
grayscale image is between $[0,255]$. Then the corresponding $\cY$
of these two examples are $\cY=\{x \in \R^n:x_i \geq lb \ \forall
i=1,\ldots,n\}$ with $lb=0$ and $\cY=\{x \in \R^n:lb\leq x_i \leq
ub \ \forall i=1,\ldots,n\}$ with $lb=0$ and $ub=255$. To solve
these types of the constrained quadratic programming problems, we
apply the nonmonotone projected gradient method proposed in
\cite{birgin2000nonmonotone} and terminate it using the duality
gap and dual feasibility conditions (if necessary).

For $\cY=\{x \in \R^n:x_i \geq lb \ \forall i=1,\ldots,n\}$, given
a Lagrangian multiplier $\beta \in \R^n$, the associated
Lagrangian dual function of \eqref{sub2} can be written as:
\[
 L(u,\beta) = w(u) + \beta^\top (lb-u),
\]
where $w(u)=\frac{1}{2}\langle u, Qu \rangle - \langle c, u
\rangle$. Based on the Karush-Kuhn-Tucker (KKT) conditions, for an
optimal solution $u^*$ of \eqref{sub2}, there exists a Lagrangian
multiplier $\beta^*$ such that
\[
\ba{l}
 Qu^* - c - \beta^*= 0, \\
 \beta_i^* \geq 0 \ \forall i =1,\ldots,n, \\
  (lb-u_i^*)\beta_i^* = 0 \ \forall i =1,\ldots,n.
\ea
\]
Then at the $s$th iteration of the projected gradient method, we
let $\beta^s=Qu^s - c$. As $\{u^s\}$ approaches the solution $u^*$
of \eqref{sub2}, $\{\beta^s\}$ approaches the Lagrangian
multiplier $\beta^*$ and the corresponding duality gap at each
iteration is given by $\sum_{i=1}^n\beta^s_i(lb-u^s_i)$.
Therefore, we terminate the projected gradient method when
\[
\frac{ |\sum_{i=1}^n\beta^s_i(lb-u^s_i)|}{\max(|w(u^s)|,1)} \ \leq \eps_D \
\mbox{and} \ \frac{ -\min(\beta^s,0)}{\max(\|\beta^s\|_2,1)} \ \leq \eps_F
\]
for some tolerances $\eps_D,\eps_F > 0$.

For $\cY=\{x \in \R^n:lb\leq x_i \leq ub \ \forall
i=1,\ldots,n\}$, given Lagrangian multipliers $\beta, \gamma \in
\R^n$, the associated Lagrangian function of \eqref{sub2} can be
written as:
\[
 L(u,\beta,\gamma) = w(u) + \beta^\top (lb-u) + \gamma^\top (u-ub),
\]
where $w(u)$ is defined as above. Based on the KKT conditions, for
an optimal solution $u^*$ of \eqref{sub2}, there exist Lagrangian
multipliers $\beta^*$ and $\gamma^*$ such that
\[
\ba{l}
 Qu^* - c - \beta^* + \gamma^*= 0, \\
\beta_i^* \geq 0 \ \forall i =1,\ldots,n, \\
 \gamma_i^* \geq 0 \ \forall i =1,\ldots,n, \\
  (lb-u_i^*)\beta_i^* = 0 \ \forall i =1,\ldots,n, \\
(u_i^*-ub)\gamma_i^* = 0 \ \forall i =1,\ldots,n.
\ea
\]
Then at the $s$th iteration of the projected gradient method, we
let $\beta^s= \max(Qu^s - c,0)$ and $\gamma^s=-\min(Qu^s - c,0)$.
As $\{u^s\}$ approaches the solution $u^*$ of \eqref{sub2},
$\{\beta^s\}$ and $\{\gamma^s\}$ approach Lagrangian multipliers
$\beta^*$ and $\gamma^*$. In addition, the corresponding duality
gap at each iteration is given by
$\sum_{i=1}^n(\beta^s_i(lb-u^s_i) + \gamma^s_i(u^s_i-ub))$ and the
duality feasibility is automatically satisfied. Therefore, we
terminate the projected gradient method when
\[
\frac{ |\sum_{i=1}^n(\beta^s_i(lb-u^s_i) + \gamma^s_i(u^s_i-ub))|}{\max(|w(u^s)|,1)} \ \leq \eps_D
\]
for some tolerance $\eps_D > 0$.

\subsubsection{The BCD subproblem in step 1b)}
For $\lambda_{\bm{i}}\ge0$, $\vrho>0$ and $c \in \R^m$, the BCD
subproblem in step 1b) is in the form of
\[
\min_{\alpha \in \R^m}\sum_{\bm{i}}
\lambda_{\bm{i}}\|\alpha_{\bm{i}}\|_0 + \frac{\vrho}{2}
\sum_{{\bm{i}}}(\alpha_{\bm{i}} -c_{\bm{i}})^2.
\]
By \cite[Proposition 2.2]{LuZhangTech2010} (see also
\cite{antoniadis2001regularization,BD08} for example), the solutions of the
above subproblem forms the following set:
\begin{equation}\label{Eqn:1}
\alpha^* \in H_{\tilde\lambda}\left( c
\right)
\quad\mbox{with }
\tilde\lambda_{\bm{i}}:=\sqrt{\frac{2\lambda_{\bm{i}}}{\vrho}} \ \mbox{for all} \ \bm{i},
\end{equation}
where $H_\gamma(\cdot)$ denotes a component-wise hard thresholding operator with
threshold $\gamma$: \beq \label{D:HT} [H_{\gamma}(x)]_{\bm{i}} =
\left\{\ba{ll}
0    & \mbox{if} \ |x_{\bm{i}}| < \gamma_{\bm{i}}, \\
\{0,  x_{\bm{i}}\} & \mbox{if} \ |x_{\bm{i}}| = \gamma_{\bm{i}}, \\
x_{\bm{i}}  & \mbox{if} \ |x_{\bm{i}}| > \gamma_{\bm{i}}. \ea\right. \eeq
Note that $H_\gamma$ is defined as a set-valued mapping \cite[Chapter
5]{rockafellar2004variational} which is different (only when $|x_{\bm{i}}| =
\gamma_{\bm{i}}$) from the conventional definition of hard thresholding
operator.

\subsection{Convergence of the BCD method}
In this subsection, we establish some convergence results
regarding the inner iterations, i.e., Step 1), of the PD method. In
particular, we will show that the fixed point of the BCD method is
a local minimizer of \eqref{inner-prob2}. Moreover, under certain
conditions, we prove that the sequence $\{(u^{k,q}, \alpha^{k,q})\}$
generated by the BCD method converges and the limit is a local
minimizer of \eqref{inner-prob2}.

For convenience of presentation, we omit the index $k$ from
\eqref{inner-prob2} and consider the BCD method for solving the
following problem: \beq \label{inner-prob}
\min\{p_{\vrho}(u,\alpha): \ u \in \cY, \ \alpha \in \R^m\}. \eeq
Without loss of generality, we assume that $D=I$. We now relabel
and simplify the BCD method described in step 1a)-1c) in the PD
method as follows.
\begin{equation}\label{Algorithm:BCD}
\begin{cases}
{u}^{q+1} = \arg\min_{u\in\cY}\ \frac12\| {Au}-
{f}\|_2^2+\frac{\vrho}{2}\| {Wu}- {\alpha}^q\|_2^2,\cr
{\alpha}^{q+1} \in \Arg\min_{ {\alpha}}\
\sum_{\bm{i}}\lambda_{\bm{i}}\|{\alpha}_{\bm{i}}\|_0+\frac{\vrho}{2}\|
{\alpha}- {Wu}^{q+1}\|_2^2.
\end{cases}
\end{equation}
We first show that the fixed point of the above BCD method is a local minimizer
of \eqref{inner-prob2}.

\begin{theorem}
\label{fixed} Given a fixed point of the BCD method
\eqref{Algorithm:BCD}, denoted as $(u^*, \alpha^*)$, then
$(u^*,\alpha^*)$ is a local minimizer of $p_\vrho(u,\alpha)$.
\end{theorem}

\begin{proof}
We first note that the first subproblem of \eqref{Algorithm:BCD}
gives us
\begin{equation}
\label{Eqn:Opt:1}
\langle A^\top(Au^*-f)+\vrho W^\top(Wu^*-\alpha^*),
v-u^*\rangle\ge0\quad\mbox{for all } v\in\cY.
\end{equation}
By applying \eqref{Eqn:1}, the second subproblem of \eqref{Algorithm:BCD} leads to:
\beq
\label{Eqn:2}
\alpha^* \in H_{\tilde\lambda}\left( Wu^*
\right).
\eeq
Define index sets
$$\Gamma_0:=\{\bm{i}: \alpha^*_{\bm{i}}=0\}\quad\mbox{and}\quad \Gamma_1:=\{\bm{i}:
\alpha^*_{\bm{i}}\ne0\}.$$ It then follows from \eqref{Eqn:2} and \eqref{D:HT} that
\begin{equation}\label{Eqn:Opt:2}
\begin{cases}
|(Wu^*)_{\bm{i}}|\le \tilde\lambda_{\bm{i}}\quad\mbox{for
}\bm{i}\in\Gamma_0\cr (Wu^*)_{\bm{i}} = \alpha^*_{\bm{i}}
\quad\mbox{for }\bm{i}\in\Gamma_1,
\end{cases}
\end{equation}
where $(Wu^*)_{\bm{i}}$ denotes ${\bm{i}}$th entry of $Wu^*$.

Consider a small deformation vector $(\ph,\pg)$ such that
$u^*+\ph\in\cY$. Using \eqref{Eqn:Opt:1}, we have
\begin{eqnarray*}
p_\vrho(u^*+\ph,\alpha^*+\pg)&=&\frac12\|Au^*+A\ph-f\|_2^2+\sum_{\bm{i}}\lambda_{\bm{i}}\|(\alpha^*+\pg)_{\bm{i}}\|_0
+  \frac{\vrho}{2}\| \alpha^* + \pg-W(u^*+\ph) \|_2^2\\
 &=&\frac12\|Au^*-f\|_2^2+\langle A\ph,Au^*-f
 \rangle+\frac12\|A\ph\|_2^2 +
 \sum_{\bm{i}}\lambda_{\bm{i}}\|(\alpha^*+\pg)_{\bm{i}}\|_0\\
& &+ \frac{\vrho}{2}\| \alpha^* -Wu^* \|_2^2 + \vrho\langle
\alpha^*-Wu^*,\pg-W\ph
 \rangle+ \frac{\vrho}{2}\| \pg -W\ph \|_2^2\\
 &=&\frac12\|Au^*-f\|_2^2 + \sum_{\bm{i}}\lambda_{\bm{i}}\|(\alpha^*+\pg)_{\bm{i}}\|_0
+ \frac{\vrho}{2}\| \alpha^* -Wu^* \|_2^2+\frac12\|A\ph\|_2^2\\
&& +\langle \ph,A^\top(Au^*-f)+\vrho W^\top(Wu^*-\alpha^*)
 \rangle + \vrho\langle \pg,\alpha^*-Wu^*\rangle + \frac{\vrho}{2}\| \pg -W\ph \|_2^2\\
&\ge&\frac12\|Au^*-f\|_2^2 +
\sum_{\bm{i}}\lambda_{\bm{i}}\|(\alpha^*+\pg)_{\bm{i}}\|_0
+ \frac{\vrho}{2}\| \alpha^* -Wu^* \|_2^2\\
&& +\langle \ph,A^\top(Au^*-f)+\vrho W^\top(Wu^*-\alpha^*)
 \rangle+ \vrho \langle \pg,\alpha^*-Wu^*\rangle \\
(\mbox{By \eqref{Eqn:Opt:1}}) &\ge&\frac12\|Au^*-f\|_2^2 +
\sum_{\bm{i}}\lambda_{\bm{i}}\|(\alpha^*+\pg)_{\bm{i}}\|_0
+ \frac{\vrho}{2}\| \alpha^* -Wu^* \|_2^2 + \vrho \langle \pg,\alpha^*-Wu^*\rangle \\
  &=&\frac12\|Au^*-f\|_2^2+\frac{\vrho}{2}\| \alpha^* -Wu^* \|_2^2 +
 \sum_{\bm{i}}\Big(\lambda_{\bm{i}}\|\alpha^*_{\bm{i}}+\pg_{\bm{i}}\|_0 + \vrho
 \pg_{\bm{i}}(\alpha^*_{\bm{i}}- (Wu^*)_{\bm{i}})\Big).
\end{eqnarray*}
Splitting the summation in the last equation with respect to index
sets $\Gamma_0$ and $\Gamma_1$ and using \eqref{Eqn:Opt:2}, we
have
\begin{equation*}
p_\vrho(u^*+\ph,\alpha^*+\pg)\ge\frac12\|Au^*-f\|_2^2+\frac{\vrho}{2}\|
\alpha^* -Wu^* \|_2^2+
 \sum_{\bm{i}\in\Gamma_0}\Big(\lambda_{\bm{i}}\|\pg_{\bm{i}}\|_0 - \vrho
 \pg_{\bm{i}}(Wu^*)_{\bm{i}}\Big)+
 \sum_{\bm{i}\in\Gamma_1}\lambda_{\bm{i}}\|\alpha^*_{\bm{i}}+\pg_{\bm{i}}\|_0.
\end{equation*}
Notice that when $|\pg_{\bm{i}}|$ is small enough, we then have
$$\|\alpha^*_{\bm{i}}+\pg_{\bm{i}}\|_0=\|\alpha^*_{\bm{i}}\|_0\quad\mbox{for
}\bm{i}\in\Gamma_1.$$ Therefore, we have
\begin{eqnarray*}
p_\vrho(u^*+\ph,\alpha^*+\pg)&\ge&\frac12\|Au^*-f\|_2^2+\frac{\vrho}{2}\|
\alpha^* -Wu^* \|_2^2+
 \sum_{\bm{i}\in\Gamma_0}\Big(\lambda_{\bm{i}}\|\pg_{\bm{i}}\|_0 - \vrho
 \pg_{\bm{i}}(Wu^*)_{\bm{i}}\Big)+
 \sum_{\bm{i}\in\Gamma_1}\lambda_{\bm{i}}\|\alpha^*_{\bm{i}}\|_0\\
 &=&p_\vrho(u^*,\alpha^*)+\sum_{\bm{i}\in\Gamma_0}\Big(\lambda_{\bm{i}}\|\pg_{\bm{i}}\|_0 - \vrho
 \pg_{\bm{i}}(Wu^*)_{\bm{i}}\Big).
\end{eqnarray*}
We now show that, for $\bm{i}\in\Gamma_0$ and $\|\pg\|$ small
enough,
\begin{equation}\label{E:WantToShow}
\lambda_{\bm{i}}\|\pg_{\bm{i}}\|_0 -\vrho
 \pg_{\bm{i}}(Wu^*)_{\bm{i}}\ge0.
\end{equation}
For the indices $\bm{i}$ such that $\lambda_{\bm{i}}=0$, first
inequality of \eqref{Eqn:Opt:2} implies that $(Wu^*)_{\bm{i}}=0$
and hence \eqref{E:WantToShow} holds. Therefore, we only need to
consider indices $\bm{i}\in\Gamma_0$ such that
$\lambda_{\bm{i}}\ne0$. Then obviously as long as
$|\pg_{\bm{i}}|\le\frac{\lambda_{\bm{i}}}{\vrho|(Wu^*)_{\bm{i}}|}$,
we will have \eqref{E:WantToShow} hold. We now conclude that there
exists $\varepsilon>0$ such that for all $(\ph,\pg)$ satisfying
$\max(\|\ph\|_{\infty},\|\pg\|_{\infty})<\varepsilon$, we have
$p_{\vrho}(u^*+\ph,\alpha^*+\pg)\ge p_{\vrho}(u^*,\alpha^*)$.
\end{proof}
\vgap

We next show that under some suitable assumptions, the sequence
$\{(u^q,\alpha^q)\}$ generated by \eqref{Algorithm:BCD} converges
to a fixed point of the BCD method.

\begin{theorem}
\label{converge} Assume that $\cY = \R^n$ and $A^\top A \succ 0$.
Let $\{(u^q,\alpha^q)\}$ be the sequence generated by the BCD method
described in \eqref{Algorithm:BCD}. Then, the sequence $\{(u^q,\alpha^q)\}$
is bounded. Furthermore, any cluster point of the sequence $\{(u^q,\alpha^q)\}$ is
a fixed point of \eqref{Algorithm:BCD}.
\end{theorem}

\begin{proof}
In view of $\cY=\R^n$ and the optimality condition of the first
subproblem of \eqref{Algorithm:BCD}, one can see that
\beq \label{uq}
u^{q+1} = (A^\top A+\vrho I)^{-1}A^\top
f+\vrho (A^\top A+\vrho I)^{-1}W^\top\alpha^q.
\eeq
Let $x:=(A^\top A+\vrho I)^{-1}A^\top f$, $P:=\vrho (A^\top
A+\vrho I)^{-1}$, equation \eqref{uq} can be rewritten as
\beq \label{uqr}
u^{q+1} = x +PW^\top\alpha^q.
\eeq
Moreover, by the assumption $A^\top A \succ 0$, we have
$0\prec P\prec I$.

Using \eqref{uqr} and \eqref{D:HT}, we observe from the second subproblem of
\eqref{Algorithm:BCD} that
\begin{equation}\label{IHT:0}
\alpha^{q+1} \in H_{\tilde\lambda}(Wu^{q+1})=H_{\tilde\lambda}\left(Wx
+WPW^\top\alpha^q\right).
\end{equation}
Let $Q:=I- WPW^\top$, then \eqref{IHT:0} can be rewritten as
\begin{equation}\label{IHT}
\alpha^{q+1} \in H_{\tilde\lambda}\left(\alpha^q+Wx-Q\alpha^q\right).
\end{equation}
In addition, from $W^{\top}W =I $ we can easily show that
$0 \prec Q \preceq I$.

Let $F(\alpha,\beta):=\frac12\langle \alpha, Q \alpha\rangle -
\langle Wx,\alpha \rangle +
\sum_{\bm{i}}\bar\lambda_{\bm{i}}\|\alpha_{\bm{i}}\|_0 -\frac12\langle
\alpha-\beta,Q(\alpha-\beta) \rangle+\frac12\|\alpha-\beta\|_2^2$
where $\bar\lambda = \frac{\lambda}{\rho}$. Then we have
\begin{equation}\label{F} \Argmin_{\alpha}F(\alpha,\alpha^q) =
\Argmin_{\alpha} \frac12 \|\alpha-(\alpha^q+Wx-Q\alpha^q)\|_2^2 +\sum_{\bm{i}}\bar\lambda_{\bm{i}}\|\alpha_{\bm{i}}\|_0.
\end{equation}
In view of equation \eqref{IHT} and \eqref{F} and the definition of the
hard thresholding operator, we can easily observe that
$\alpha^{q+1} \in \Argmin_{\alpha}F(\alpha,\alpha^q)$. By following
similar arguments as in \cite[Lemma 1, Lemma D.1]{BD08}, we have
\begin{eqnarray*}
F(\alpha^{q+1},\alpha^{q+1})&\le&
F(\alpha^{q+1},\alpha^{q+1})+\frac12\|\alpha^{q+1}-\alpha^q\|_2^2-\frac12\langle
\alpha^{q+1}-\alpha^q, Q(\alpha^{q+1}-\alpha^q)\rangle\\
 &=&F(\alpha^{q+1},\alpha^q)\\
 &\le&F(\alpha^{q},\alpha^q),
\end{eqnarray*}
which leads to
$$\|\alpha^{q+1}-\alpha^q\|_2^2-\langle
\alpha^{q+1}-\alpha^q, Q(\alpha^{q+1}-\alpha^q)\rangle\le
2F(\alpha^{q},\alpha^{q})-2F(\alpha^{q+1},\alpha^{q+1}).$$
Since
$P\succ 0$, we have
\begin{eqnarray*}
\|W^\top(\alpha^{q+1}-\alpha^q)\|_2^2&\le&
\frac1{C_1}\langle
W^\top(\alpha^{q+1}-\alpha^q),PW^\top (\alpha^{q+1}-\alpha^q) \rangle \\
&=&\frac1{C_1} \langle \alpha^{q+1}-\alpha^q, (I-Q)(\alpha^{q+1}-\alpha^q) \rangle\\
&=&\frac1{C_1}\left( \|\alpha^{q+1}-\alpha^q\|_2^2-\langle
\alpha^{q+1}-\alpha^q, Q(\alpha^{q+1}-\alpha^q)\rangle \right) \\
&\le&\frac2{C_1}F(\alpha^{q},\alpha^{q})-\frac2{C_1}F(\alpha^{q+1},\alpha^{q+1})
\end{eqnarray*}
for some $C_1>0$. Telescoping on the above inequality and using
the fact that $ \sum_{\bm{i}}\lambda_{\bm{i}}\|\alpha_{\bm{i}}\|_0
\geq 0$, we have
\begin{eqnarray*}
\sum_{q=0}^N\|W^\top(\alpha^{q+1}-\alpha^q)\|_2^2&\le&
\frac2{C_1}F(\alpha^0,\alpha^0)-\frac2{C_1}F(\alpha^{N+1},\alpha^{N+1}) \nonumber \\
&\le& \frac2{C_1}\left(F(\alpha^0,\alpha^0) - (\frac12\langle \alpha^{N+1}, Q \alpha^{N+1}\rangle - \langle Wx,\alpha^{N+1} \rangle)\right) \nonumber \\
&\le& \frac2{C_1}\left(F(\alpha^0,\alpha^0) - K\right), \label{cauchy}
\end{eqnarray*}
where $K$ is the optimal value of $\min\limits_y \{\frac{1}{2}
\langle y, Qy\rangle - \langle Wx, y \rangle\}$. Since $Q \succ 0$,
we have $K > -\infty$. Then the last inequality implies that
$\lim_{q\to \infty}\|W^\top(\alpha^{q+1}-\alpha^q) \|_2\to0$.

By using \eqref{uqr} and  $P \prec I$, we see that
\begin{eqnarray*}
\|u^{q+1} - W^\top\alpha^{q+1}\|_2 &=& \|x + PW^\top\alpha^q - W^\top\alpha^{q+1} +W^\top\alpha^q - W^\top\alpha^q \|_2 \\
&=& \|x+(P-I)W^\top\alpha^q - W^\top(\alpha^{q+1}- \alpha^q)\|_2 \\
&\geq & \|x+(P-I)W^\top\alpha^q\|_2 - \|W^\top(\alpha^{q+1}- \alpha^q)\|_2  \\
&= & \|(I-P)W^\top\alpha^q-x\|_2 - \|W^\top(\alpha^{q+1}- \alpha^q)\|_2  \\
&\geq & \|(I-P)W^\top\alpha^q\|_2-\|x\|_2 - \|W^\top(\alpha^{q+1}- \alpha^q)\|_2  \\
&\geq & C_2\|W^\top\alpha^q\|_2-\|x\|_2 - \|W^\top(\alpha^{q+1}- \alpha^q)\|_2
\end{eqnarray*}
for some $C_2>0$. Then by rearranging the above inequality and using the fact
$W^\top W=I$, we have
\begin{eqnarray*}
\|W^\top\alpha^q\|_2 &\leq& \frac1{C_2}(\|u^{q+1} - W^\top\alpha^{q+1}\|_2 + \|x\|_2+\|W^\top(\alpha^{q+1}- \alpha^q)\|_2) \\
&=& \frac1{C_2}(\|W^\top(Wu^{q+1} - \alpha^{q+1})\|_2 + \|x\|_2+\|W^\top(\alpha^{q+1}- \alpha^q)\|_2) \\
&\leq & \frac1{C_2}(\|Wu^{q+1} - \alpha^{q+1}\|_2 + \|x\|_2+\|W^\top(\alpha^{q+1}- \alpha^q)\|_2).
\end{eqnarray*}
By the definition of the hard thresholding operator and \eqref{IHT:0}, we can easily
see that $\|Wu^{q+1} - \alpha^{q+1}\|_2$ is bounded. In addition, notice that $\|x\|_2$ is
a constant and $\lim_{q\to \infty}\|W^\top(\alpha^{q+1}-\alpha^q) \|_2\to0$. Thus
$\|W^\top\alpha^q\|_2$ is also bounded. By using \eqref{uqr} and the definition of the hard
thresholding operator again, we can immediately see that both $\{u^{q+1}\}$ and $\{\alpha^{q+1}\}$ are
bounded as well.


Suppose that $(u^*,\alpha^*)$ is a cluster point of the sequence $\{(u^q,\alpha^q)\}$. Therefore,
there exists a subsequence $\{(u^{q_l},\alpha^{q_l})\}_l$ converging to $(u^*,\alpha^*)$.
Using \eqref{IHT:0} and the definition of the hard thresholding operator, we can observe that
\[
\alpha^* = \lim_{l\to \infty} \alpha^{q_{l+1}} \in
H_{\tilde\lambda}(\lim_{l\to \infty} Wu^{q_{l+1}}) =  H_{\tilde\lambda}(Wu^*).
\]
In addition, it follows from \eqref{uq} that
\[
u^* = (A^\top A+\vrho I)^{-1}A^\top
f+\vrho (A^\top A+\vrho I)^{-1}W^\top\alpha^*.
\]
In view of the above two relations, one can immediately conclude
that $\{(u^*,\alpha^*)\}$ is a fixed point of
\eqref{Algorithm:BCD}.
\end{proof}
\vgap In the view of Theorems \ref{fixed}, \ref{converge} and
under some suitable assumptions, we can easily observe the following
convergence of the BCD method.

\begin{theorem}
Assume that $\cY = \R^n$ and $A^\top A \succ 0$. Then,
the sequence $\{(u^q,\alpha^q)\}$ generated by the BCD method
has at least one cluster point. Furthermore, any cluster point of
the sequence $\{(u^q,\alpha^q)\}$ is a local minimizer of
\eqref{inner-prob}.
\end{theorem}

For the PD method itself, similar arguments as in the proof of
\cite[Theorem 3.2]{LuZhangTech2010} will lead to that every
accumulation point of the sequence $\{(u^k,\alpha^k)\}$ is a
feasible point of \eqref{l0a}. Although it is not clear whether
the accumulation point is a local minimizer of \eqref{l0a}, our
numerical results show that the solutions obtained by the PD
method are superior than those obtained by the balanced approach
and the analysis based approach.

\section{Numerical results}
\label{results}

In this section, we conduct numerical experiments to test the
performance of the PD method for problem \eqref{l0a} presented in
Section \ref{method} and compare the results with the balanced
approach \eqref{l1b} and the analysis based approach \eqref{l1a}.
We use the accelerated proximal gradient (APG) algorithm
\cite{shenAPG2009} (see also \cite{beck2009fast}) to solve the
balanced approach; and we use the split Bregman algorithm
\cite{GoldO,cai2009split} to solve the analysis based approach.

For APG algorithm that solves balanced approach \eqref{l1b}, we
shall adopt the following stopping criteria:
\[
 \min \left\{ \frac{\|\alpha^k-\alpha^{k-1}\|_2} {\max\{1,\|\alpha^k\|_2\}},
\frac{\|AW^\top\alpha^k-f\|_D}{\|f\|_2} \right\} \leq \eps_P.  
\]
For split Bregman algorithm that solves the analysis based
approach \eqref{l1a}, we shall use the following stopping
criteria:
\[
 \frac{\|Wu^{k+1}-\alpha^{k+1}\|_2}{\|f\|_2} \leq \eps_S.
\]

Throughout this section, the codes of all the algorithms are written in
MATLAB and all computations below are performed on a workstation with Intel
Xeon E5410 CPU (2.33GHz) and 8GB RAM running Red Hat Enterprise Linux (kernel
2.6.18). If not specified, the piecewise linear B-spline framelets
constructed by \cite{ron1997affine} are used in all the numerical
experiments. We also take $D=I$ for all three methods for simplicity. For the
PD method, we choose $\eps_I=10^{-4}$ and $\eps_O=10^{-3}$ and set
$\alpha^{0,0}$, $\alpha^{\feas}$ and $u^{\feas}$ to be zero vectors. In
addition, we choose\cite[Algorithm 2.2]{birgin2000nonmonotone} and set
$M=20$, $\eps_D=5\times 10^{-5}$ and $\eps_F= 10^{-4}$ (if necessary) for the
projected gradient method applied to one of subproblems arising in the BCD method
$(\mbox{i.e., step 1a) in the PD method})$.

\subsection{Experiments on CT Image Reconstruction}
\label{CT}

In this subsection, we apply the PD method stated in Section
\ref{method} to solve problem \eqref{l0a} on CT images and compare
the results with the balanced approach \eqref{l1b} and the
analysis based approach \eqref{l1a}. The matrix $A$ in
\eqref{model} is taken to be a projection matrix based on fan-beam
scanning geometry using Siddon's algorithm \cite{siddon1985fast},
and $\eta$ is generated from a zero mean Gaussian distribution
with variance $\sigma=0.01\|f\|_\infty$. In addition, we pick
level of framelet decomposition to be 4 for the best quality of
the reconstructed images. For balanced approach, we set $\kappa=2$
and take $\eps_P= 1.5\times 10^{-2}$ for the stopping criteria of
the APG algorithm. We set $\eps_S= 10^{-5}$ for the stopping
criteria of the split Bregman algorithm when solving the analysis
based approach. Moreover, we take $\cY=\{x \in \R^n:x_i \geq 0 \
\forall i=1,\ldots,n\}$ for model \eqref{l0a}, and take
$\delta=10$ and $\vrho_0=10$ for the PD method. To measure quality
of the restored image, we use the PSNR value defined by
\[
 \mbox{PSNR}:=-20\log_{10}\frac{\|u-\tu\|_2}{n},
\]
where $u$ and $\tu$ are the original and restored images
respectively, and $n$ is total number of pixels in $u$.

Table \ref{table1} summarizes the results of all three models when applying
to the CT image restoration problem and the corresponding images and their
zoom-in views are shown in Figure \ref{image1} and Figure \ref{image12}.  In
Table \ref{table1}, the CPU time (in seconds) and PSNR values of all three
methods are given in the first and second row, respectively. In order to
fairly compare the results, we have tuned the parameter $\lambda$ to achieve
the best quality of the restoration images for each individual method. We
observe that based on the PSNR values listed in Table \ref{table1} the
analysis based approach and the PD method obviously achieve better
restoration results than the balanced approach. Nevertheless, the APG
algorithm for the balanced approach is the fastest algorithm in this
experiment. In addition, the PD method is faster and achieves larger PSNR
than the split Bergman algorithm for the analysis based approach. Moreover,
we can observe from Figure \ref{image12} that the edges are recovered better
by the PD method and the balanced approach.

\begin{table}[t!]
\caption{Comparisons: CT image reconstruction} \centering
\label{table1}
\begin{tabular}{|l||c |c |c|}
\hline
\multicolumn{1}{|l||}{} &
\multicolumn{1}{c|}{Balanced approach} & \multicolumn{1}{c|}{Analysis based approach} & \multicolumn{1}{c|}{PD method}
\\
\hline
Time & 56.0  & 204.8 & 147.6 \\
PSNR & 56.06 & 59.90 & 60.22     \\
\hline
\end{tabular}
\\
\end{table}

\begin{figure}[htp]
\centering
    \includegraphics[width=1.5in]{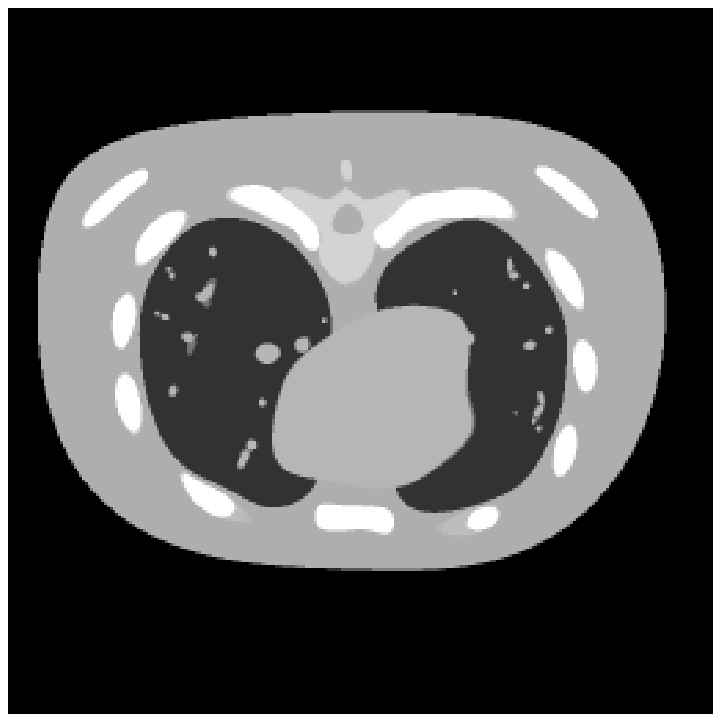}
    \includegraphics[width=1.5in]{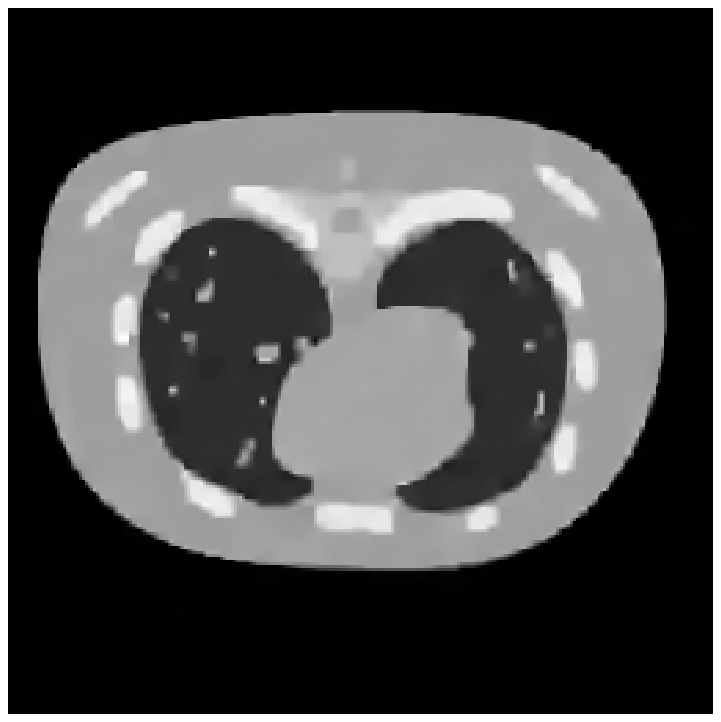}
    \includegraphics[width=1.5in]{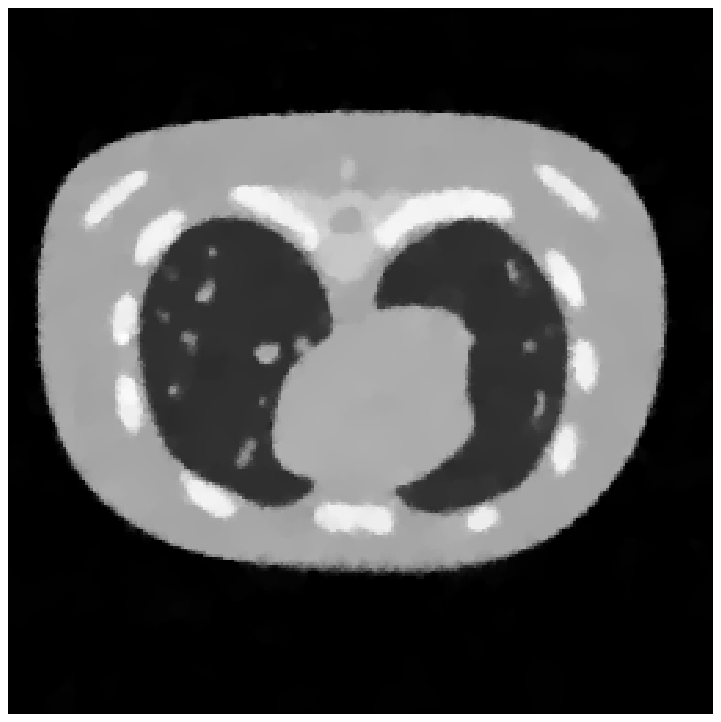}
    \includegraphics[width=1.5in]{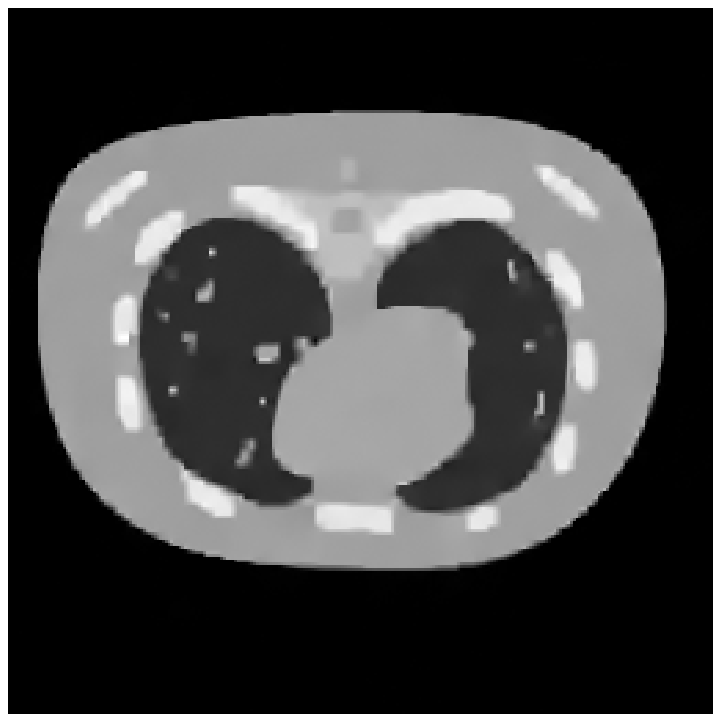}
\caption{CT image reconstruction. Images from left to right are:
original CT image, reconstructed image by balanced approach,
reconstructed image by analysis based approach and reconstructed
image by PD method.}\label{image1}
\end{figure}
\begin{figure}[htp]
\centering
    \includegraphics[width=1.5in]{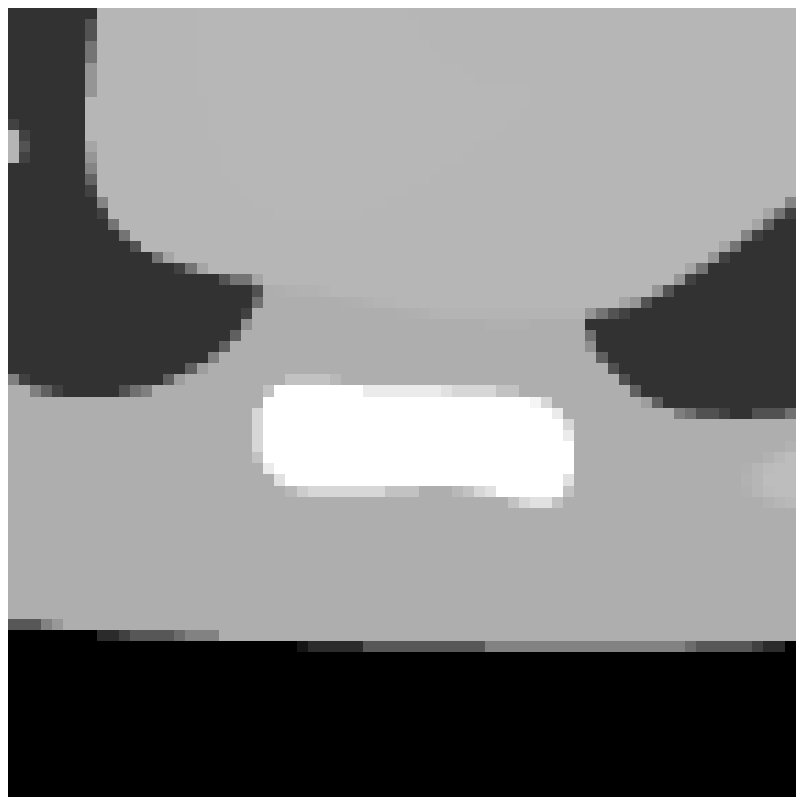}
    \includegraphics[width=1.5in]{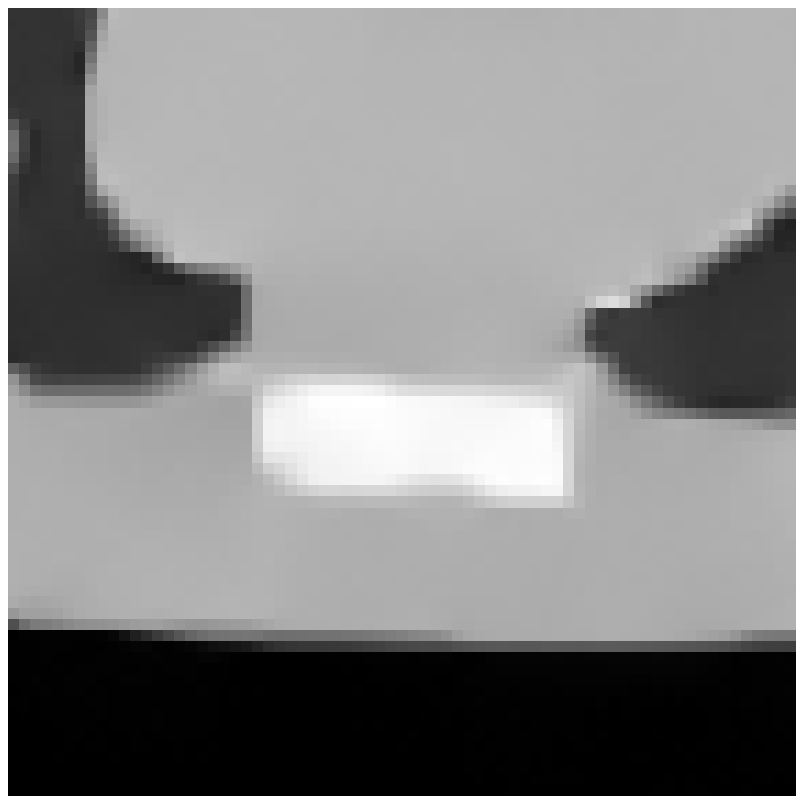}
    \includegraphics[width=1.5in]{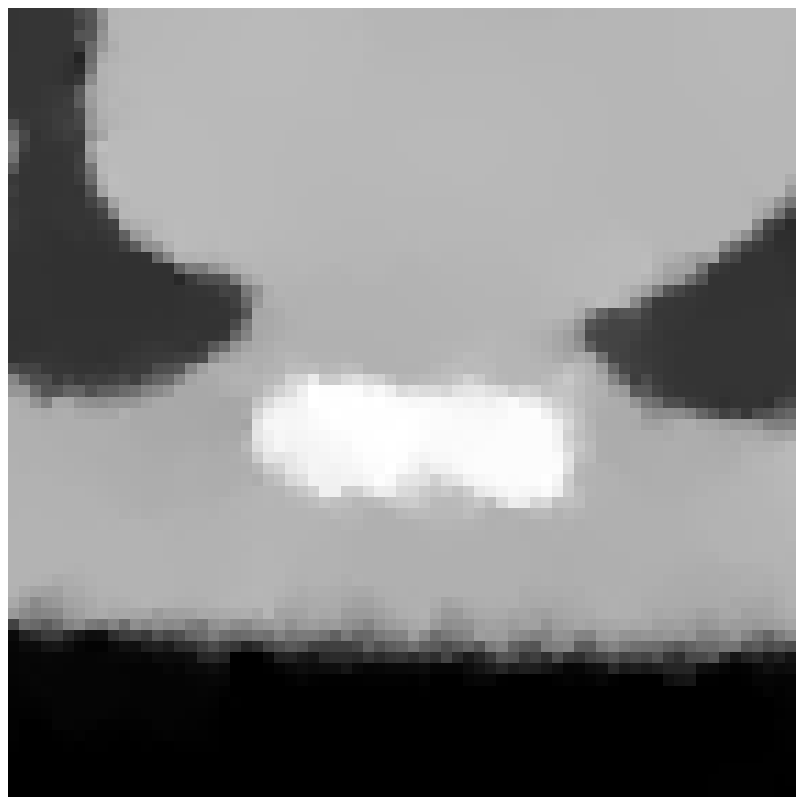}
    \includegraphics[width=1.5in]{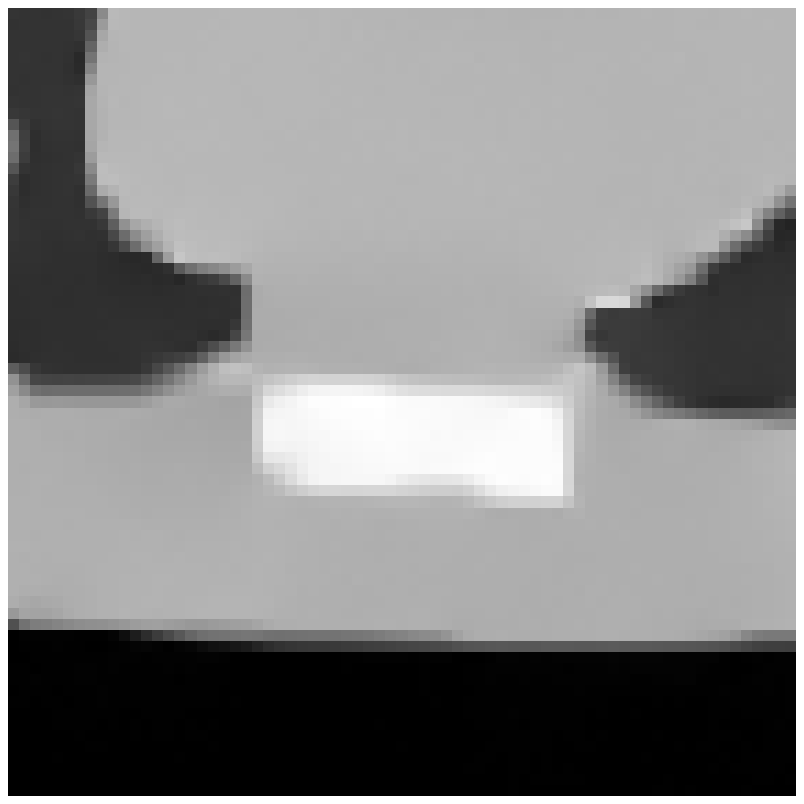}
\caption{Zoom-in views of the CT image reconstruction. Images from left to right are:
original CT image, reconstructed image by balanced approach,
reconstructed image by analysis based approach and reconstructed
image by PD method.}\label{image12}
\end{figure}

\subsection{Experiments on image deconvolution}
\label{image}

In this subsection, we apply the PD method stated in Section
\ref{method} to solve problem \eqref{l0a} on image deblurring
problems and compare the results with the balanced approach
\eqref{l1b} and the analysis based approach \eqref{l1a}. The
matrix $A$ in \eqref{l0a} is taken to be a convolution matrix with
corresponding kernel a Gaussian function (generated in MATLAB by
``fspecial(`gaussian',9,1.5);'') and $\eta$ is generated from a
zero mean Gaussian distribution with variance $\sigma$. If not
specified, we choose $\sigma=3$ in our experiments. In
addition, we pick level of framelet decomposition to be 4 for the
best quality of the reconstructed images. We set $\kappa=1$ for
balanced approach and choose both $\eps_P$ and $\eps_S$ to be
$10^{-4}$ for the stopping criteria of both APG algorithm and the
split Bregman algorithm. Moreover, we set $\cY=\{x \in \R^n:0\leq
x_i \leq 255 \ \forall i=1,\ldots, n\}$ for model \eqref{l0a}, and
take $\delta=10$ and $\vrho_0=10^ {-3}$ for the PD method. To
measure quality of restored image, we use the PSNR value defined
by
\[
 \mbox{PSNR}:=-20\log_{10}\frac{\|u-\tu\|_2}{255n}.
\]

We first test all three methods on twelve different images by using piecewise
linear wavelet and summarize the results in Table \ref{table2}. The names and
sizes of images are listed in the first two columns. The CPU time (in
seconds) and PSNR values of all three methods are given in the rest six
columns. In addition, the zoom-in views of original images, observed images and
recovered images are shown in Figure \ref{image4}-\ref{image5}. In order to fairly compare the
results, we have tuned the parameter $\lambda$ to achieve the best quality of
the restoration images for each individual method and each given image.

We first observe that in Table \ref{table2}, the PSNR values obtained by the
PD method are generally better than those obtained by other two approaches.
Although for some of the images (i.e. ``Downhill", ``Bridge", ``Duck" and
``Barbara"), the PSNR values obtained by the PD methods are comparable to
those of balanced and analysis based approaches, the quality of the restored
images can not only be judged by their PSNR values. Indeed, the zoom-in views
of the recovered images in Figure \ref{image4} and Figure \ref{image5} show
that for all tested images, the PD method produces visually superior results
than the other two approaches in terms of both sharpness of edges and
smoothness of regions away from edges. Takeing the image ``Barbara" as an
example, the PSNR value of the PD method is only slightly greater than that
obtained by the other two approaches. However, the zoom-in views of
``Barbara'' in Figure \ref{image5} show that the face of Barbara and the
textures on her scarf are better recovered by the PD method than the other
two approaches. This confirms the observation that penalizing $\ell_0$
``norm" of $Wu$ should provide good balance between sharpness of features and
smoothness of the reconstructed images. We finally note that the PD method is
slower than other two approaches in these experiments but the processing time
of the PD method is still acceptable.

We next compare all three methods on ``portrait I'' image by using three different
tight wavelet frame systems, i.e., Haar framelets, piecewise linear framelets
and piecewise cubic framelets constructed by \cite{ron1997affine}. We
summarize the results in Table \ref{table3}. The names of three wavelets are
listed in the first column. The CPU time (in seconds) and PSNR values of all
three methods are given in the rest six columns. In Table \ref{table3}, we can see
that the quality of the restored images by using the piecewise linear framelets and
the piecewise cubic framelets is better than that by using the Haar framelets. In addition,
all three methods are generally faster when using Haar framelets and slower when using
piecewise cubic framelets. Overall, all three approaches when using the piecewise linear
have balanced performance in terms of time and quality (i.e., the PSNR value).
Finally, we observe that  the PD method consistently achieves the best quality of
restored images among all the approaches for all three different tight wavelet
frame systems.

Finally, we test how different noise levels affect the restored images
obtained from all the three methods. We choose three different noise levels
(i.e., $\sigma = 3,5,7$) for image ``Portrait I'', and test all the three
methods by using piecewise linear framelets. We summarize the results in
Table \ref{table4}. The variances of noises are listed in the first column.
The CPU time (in seconds) and PSNR values of all three methods are given in
the rest six columns. We observe that the qualities of the restored images by
all three methods degrade when the noise level increases. Nevertheless, the
PD method still outperforms other two methods.

\begin{figure}[t!]
\centering
\subfigure{\includegraphics[width=31mm]{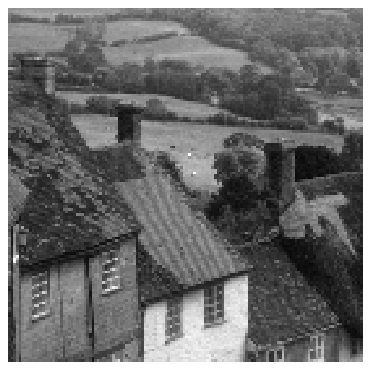}}
\subfigure{\includegraphics[width=31mm]{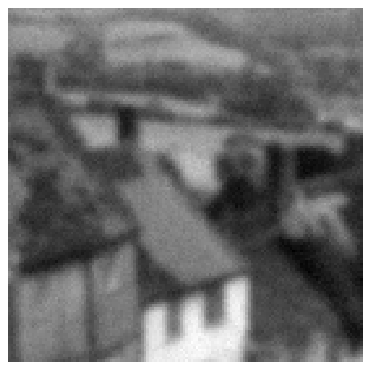}}
\subfigure{\includegraphics[width=31mm]{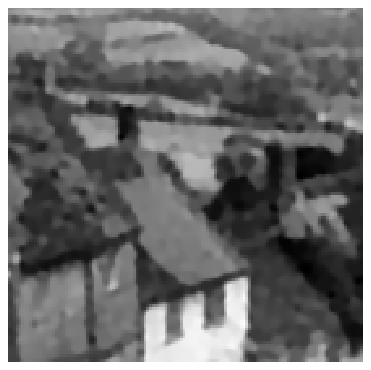}}
\subfigure{\includegraphics[width=31mm]{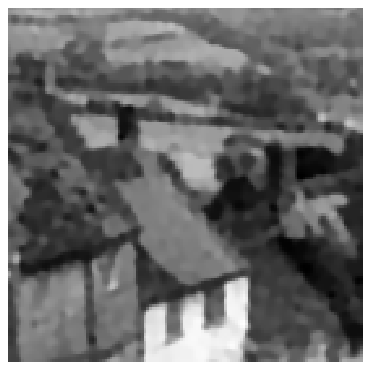}}
\subfigure{\includegraphics[width=31mm]{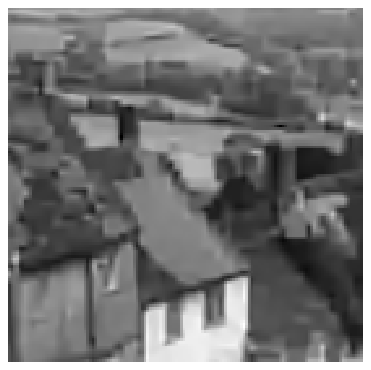}}
\\
\subfigure{\includegraphics[width=31mm]{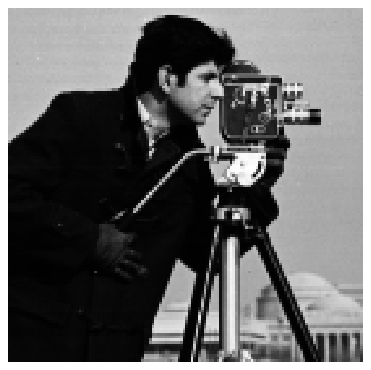}}
\subfigure{\includegraphics[width=31mm]{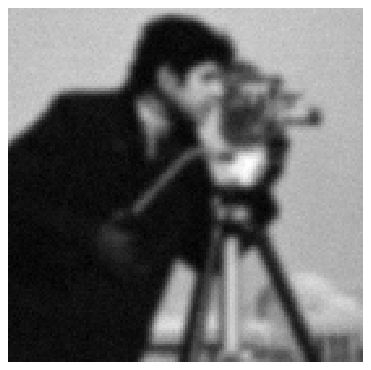}}
\subfigure{\includegraphics[width=31mm]{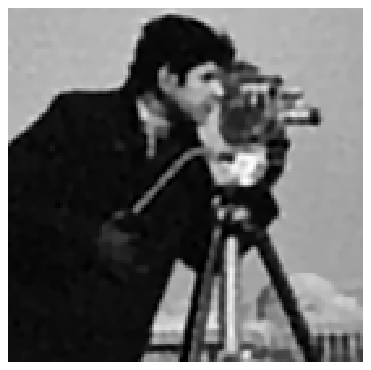}}
\subfigure{\includegraphics[width=31mm]{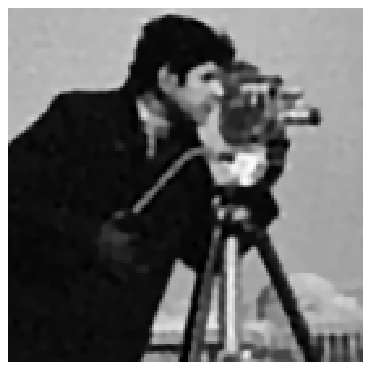}}
\subfigure{\includegraphics[width=31mm]{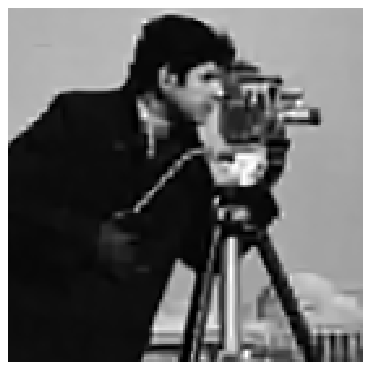}}
\\
\subfigure{\includegraphics[width=31mm]{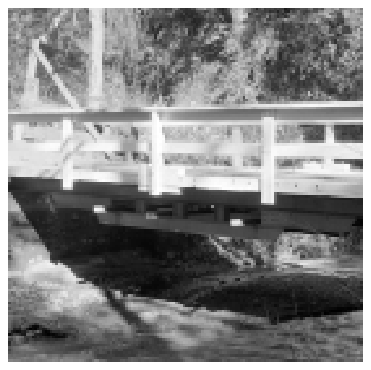}}
\subfigure{\includegraphics[width=31mm]{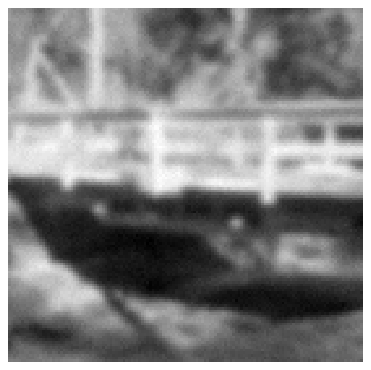}}
\subfigure{\includegraphics[width=31mm]{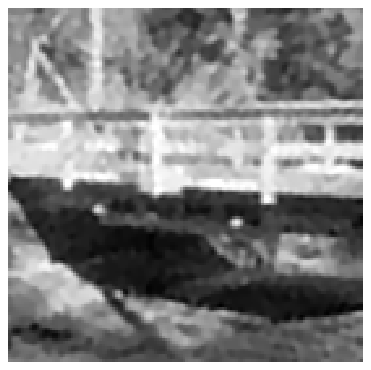}}
\subfigure{\includegraphics[width=31mm]{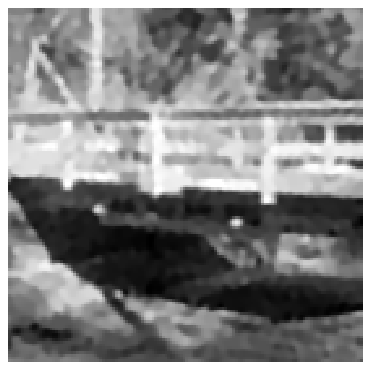}}
\subfigure{\includegraphics[width=31mm]{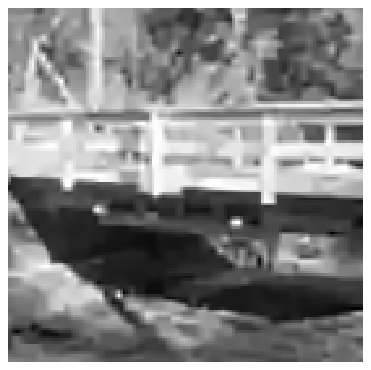}}
\\
\subfigure{\includegraphics[width=31mm]{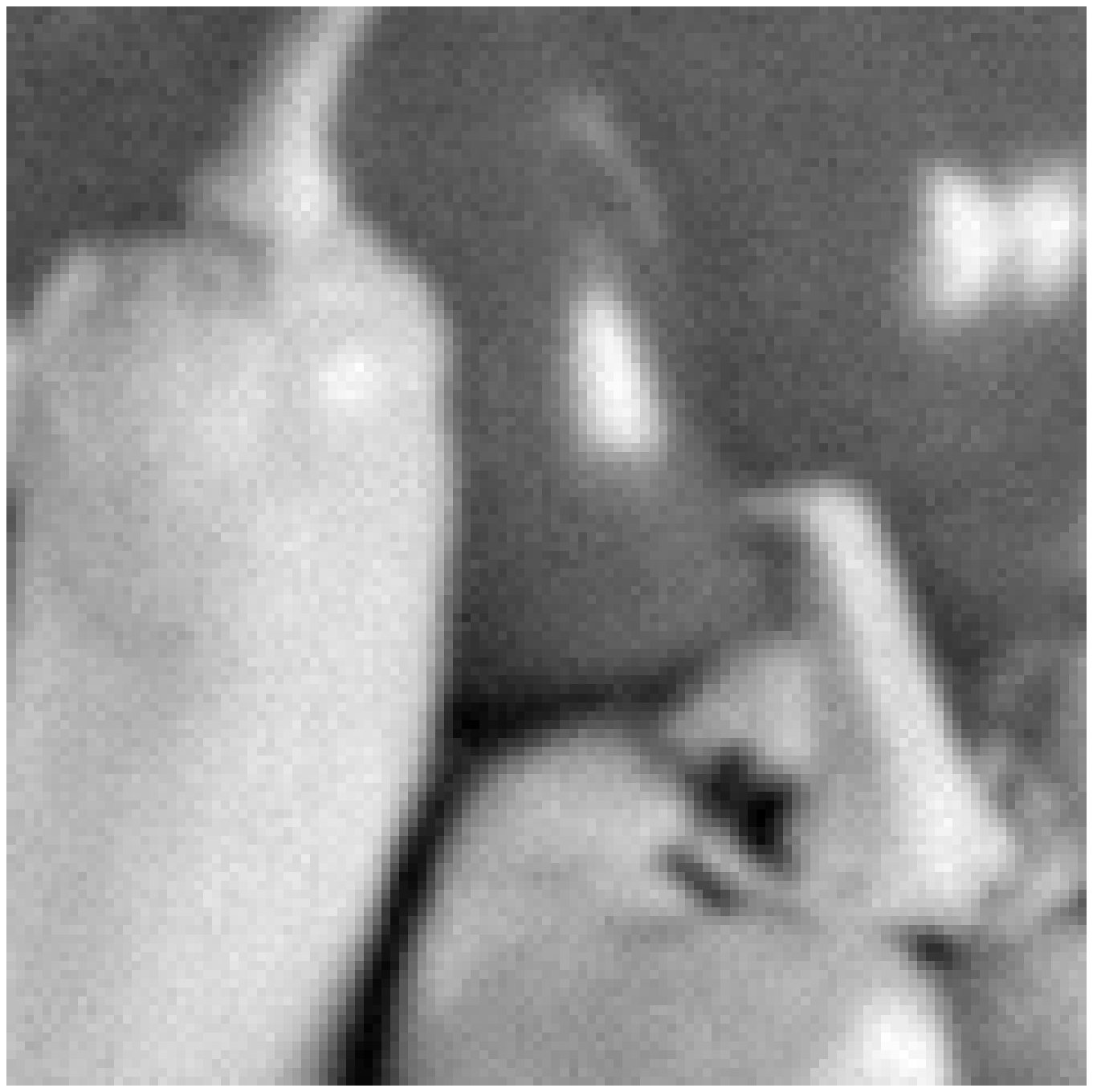}}
\subfigure{\includegraphics[width=31mm]{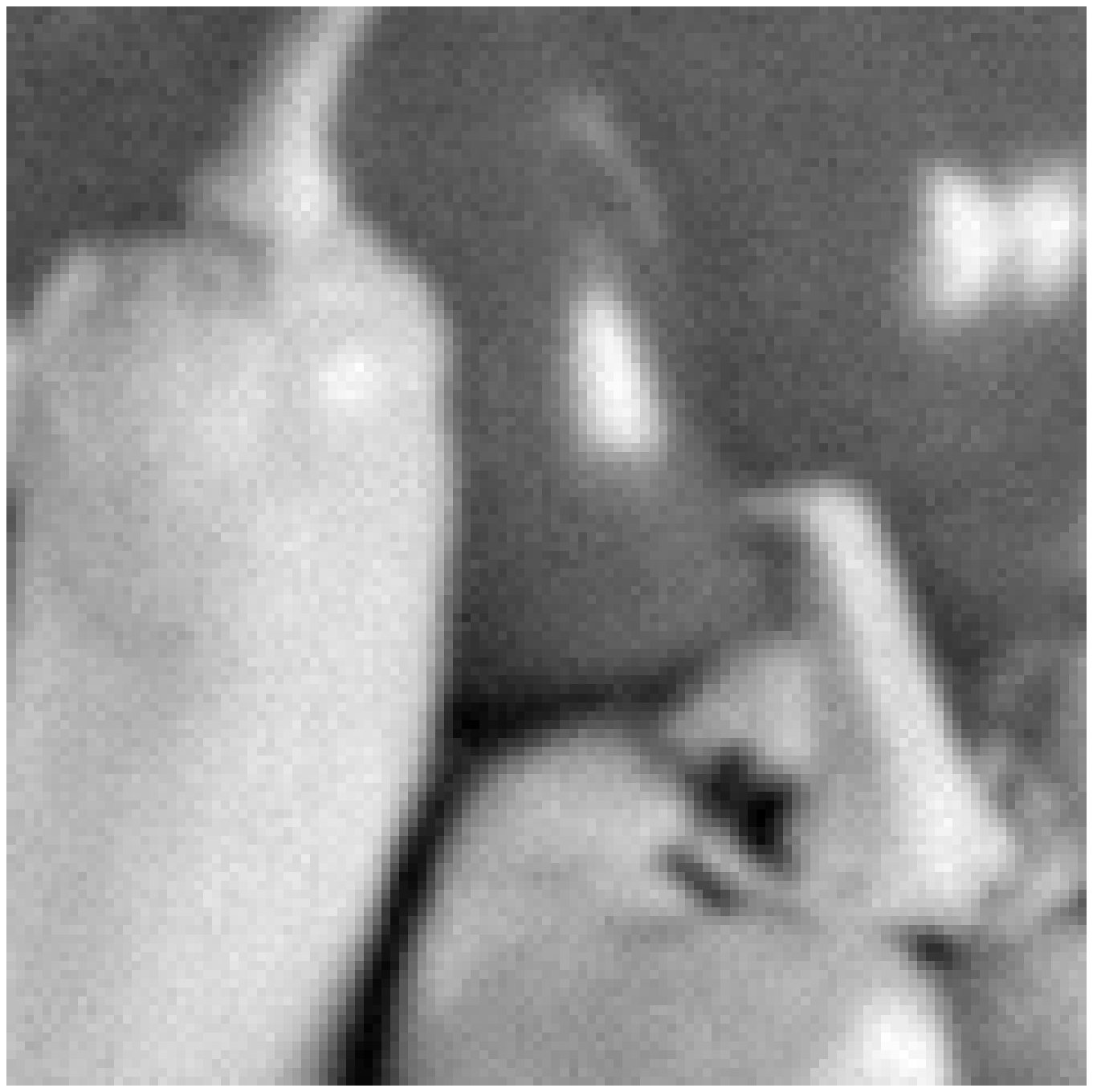}}
\subfigure{\includegraphics[width=31mm]{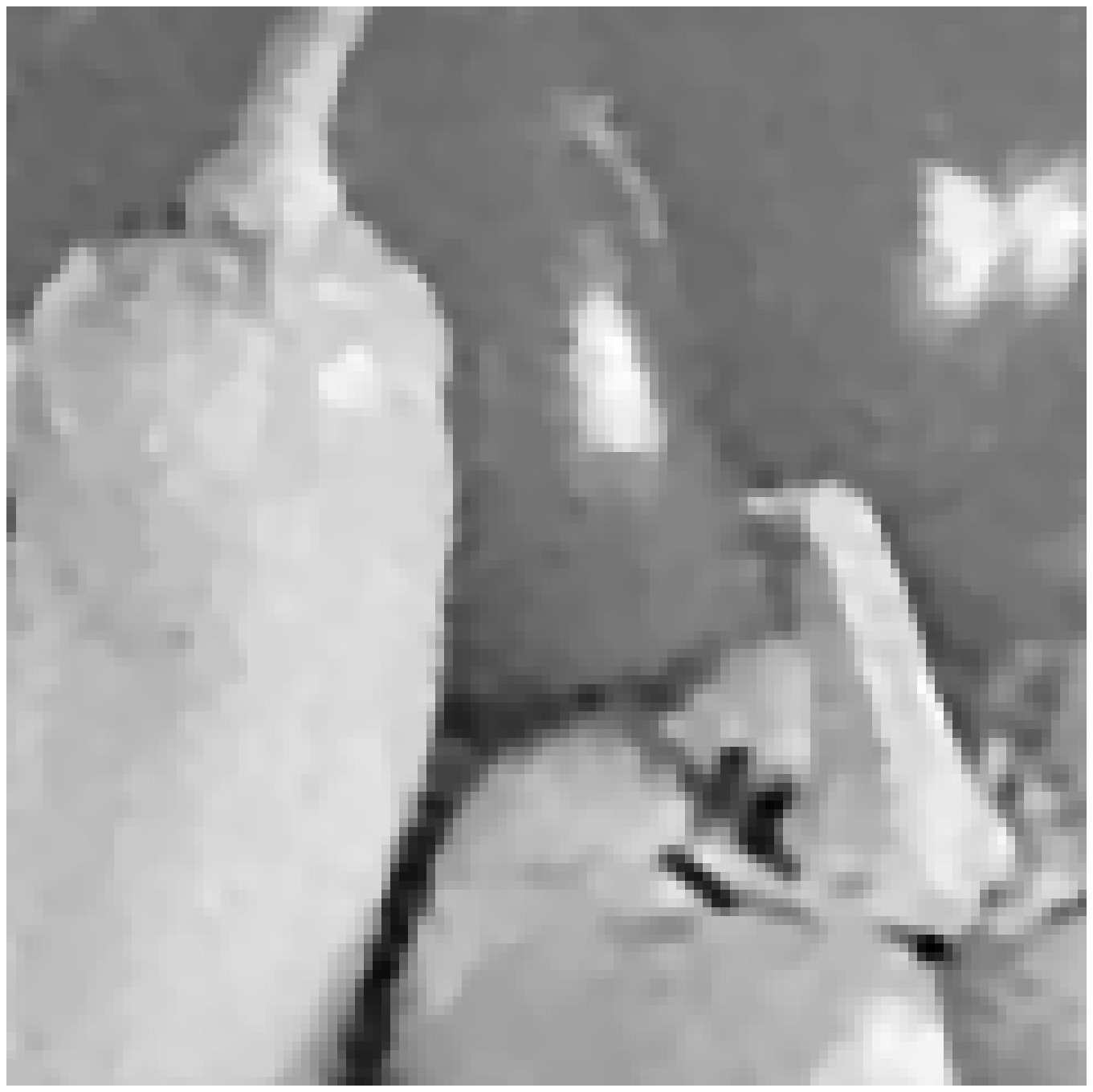}}
\subfigure{\includegraphics[width=31mm]{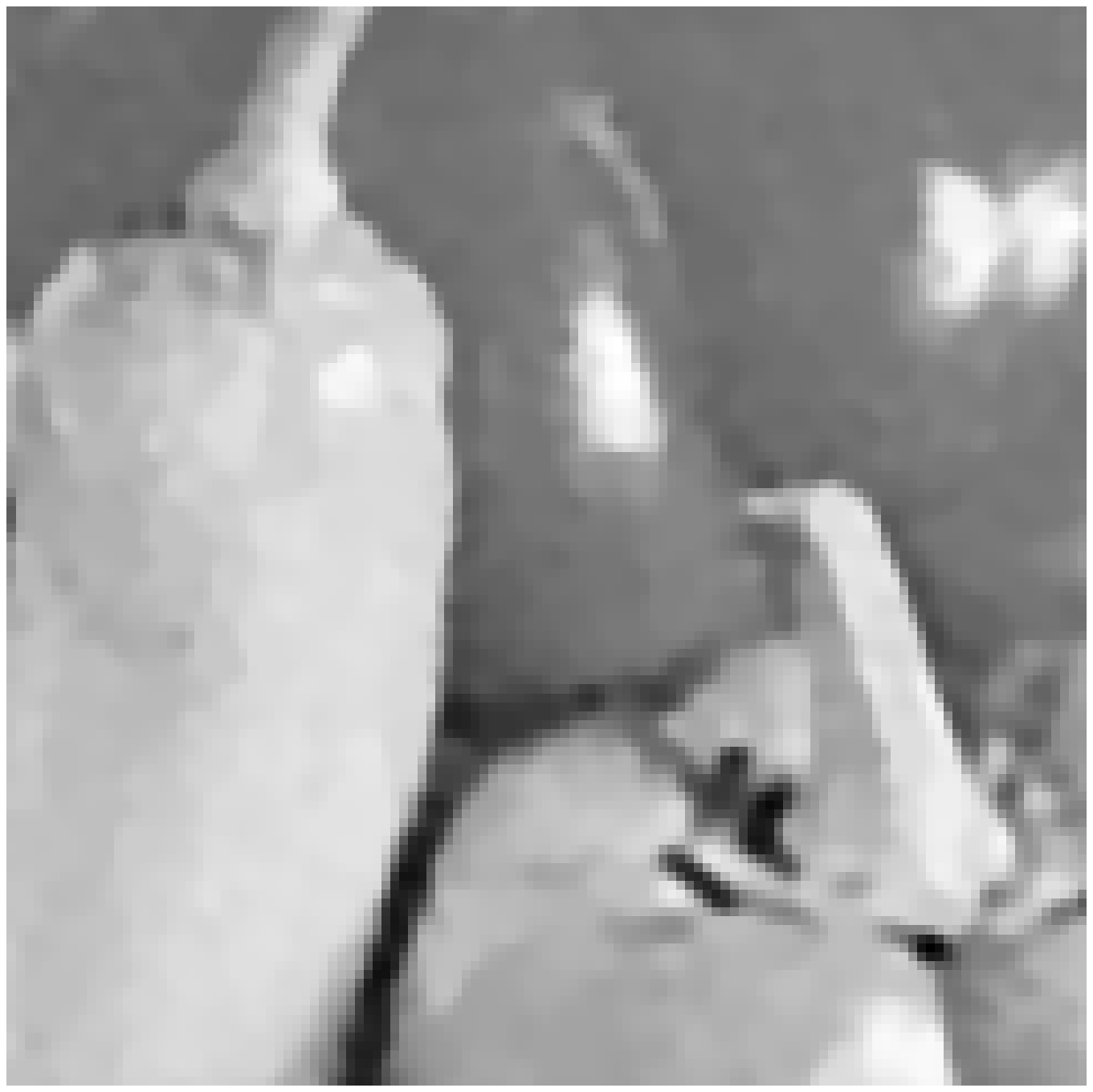}}
\subfigure{\includegraphics[width=31mm]{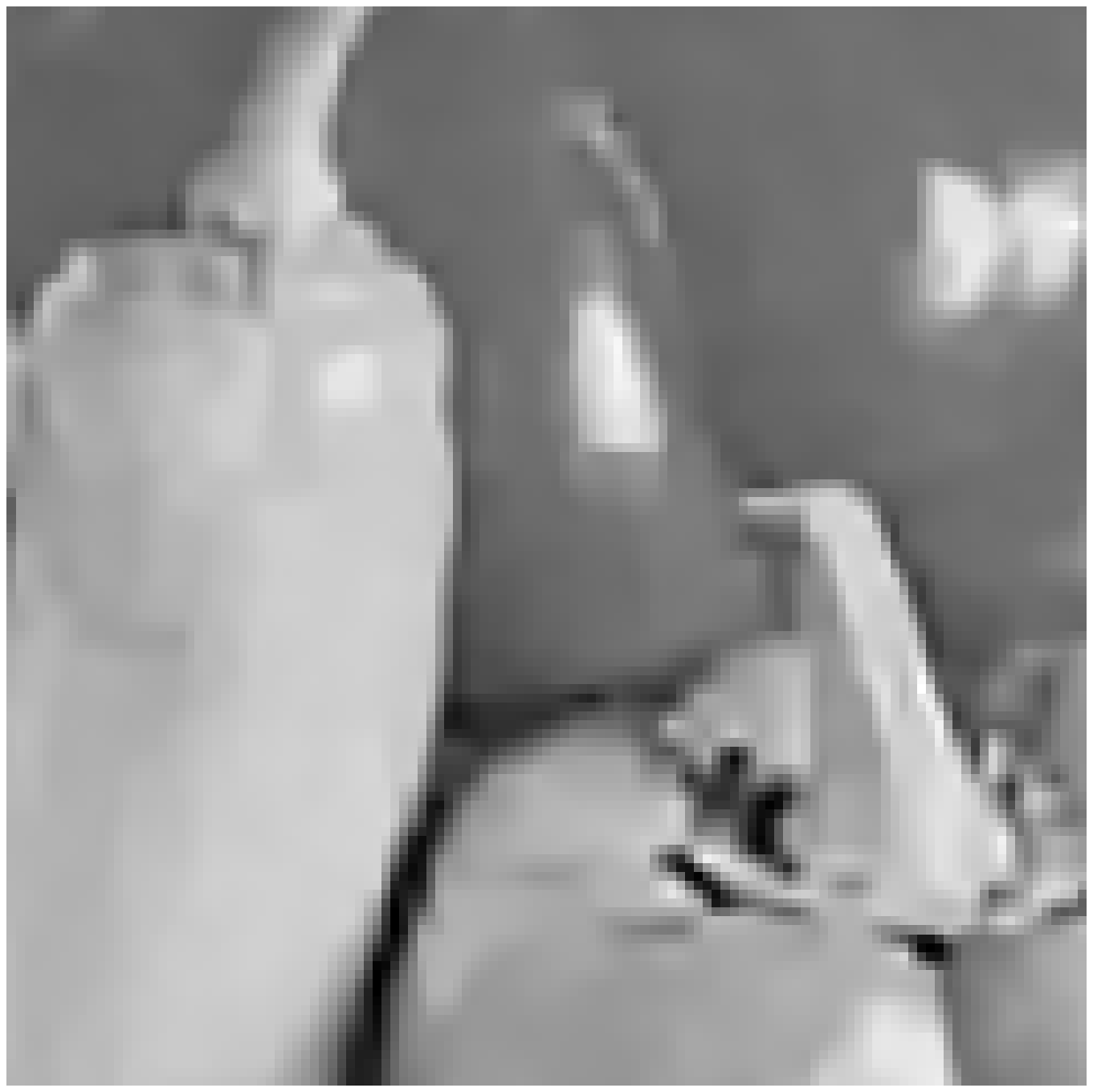}}
\\
\subfigure{\includegraphics[width=31mm]{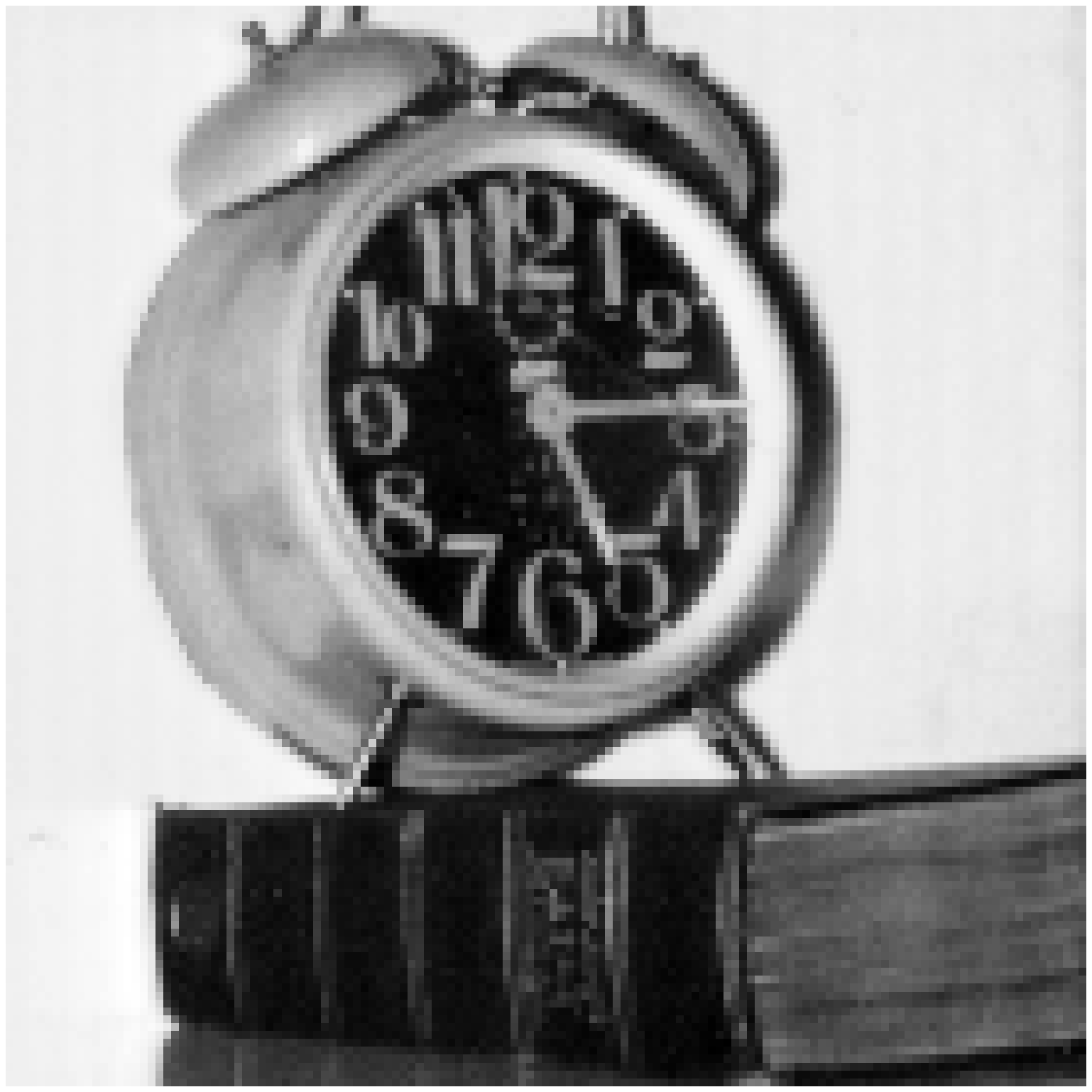}}
\subfigure{\includegraphics[width=31mm]{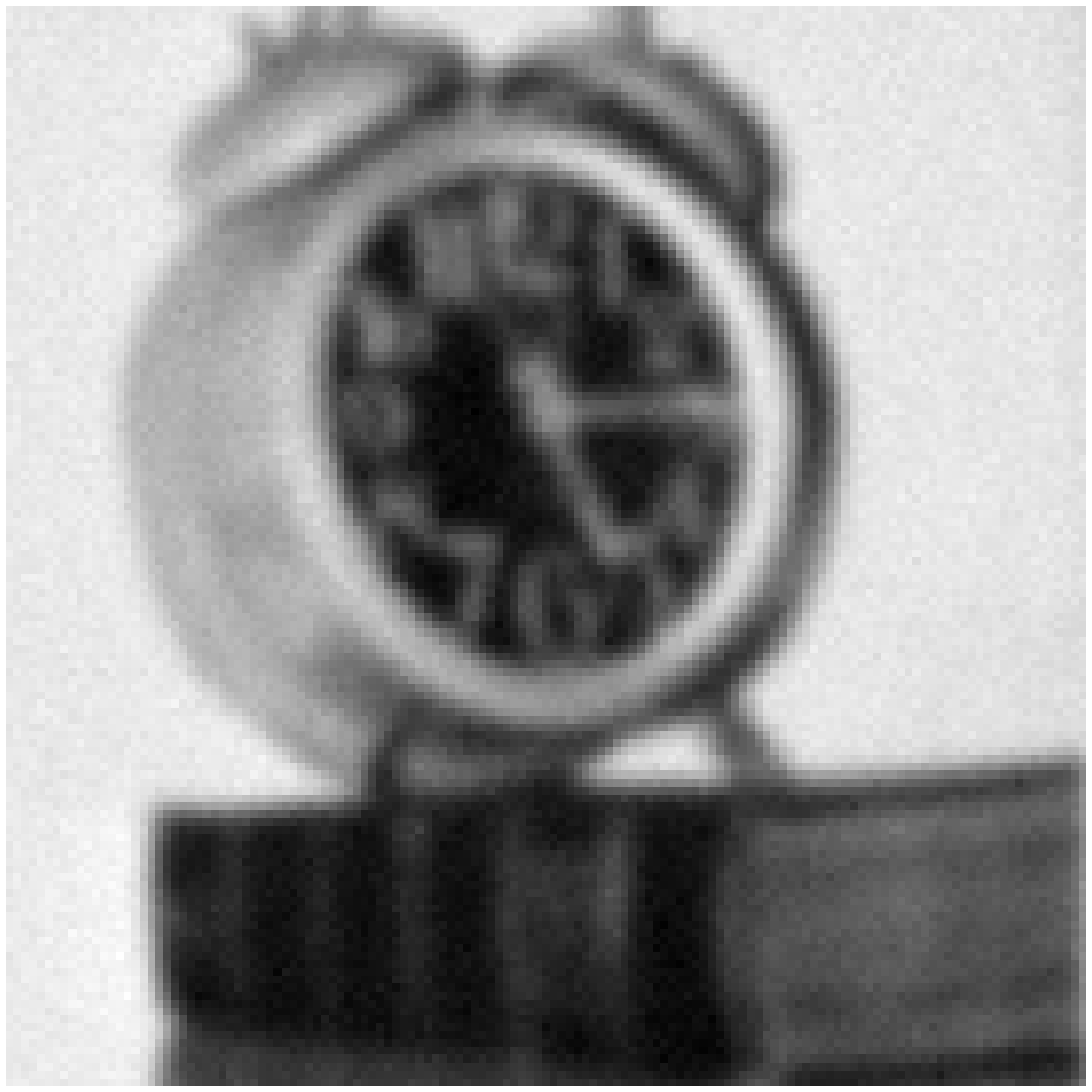}}
\subfigure{\includegraphics[width=31mm]{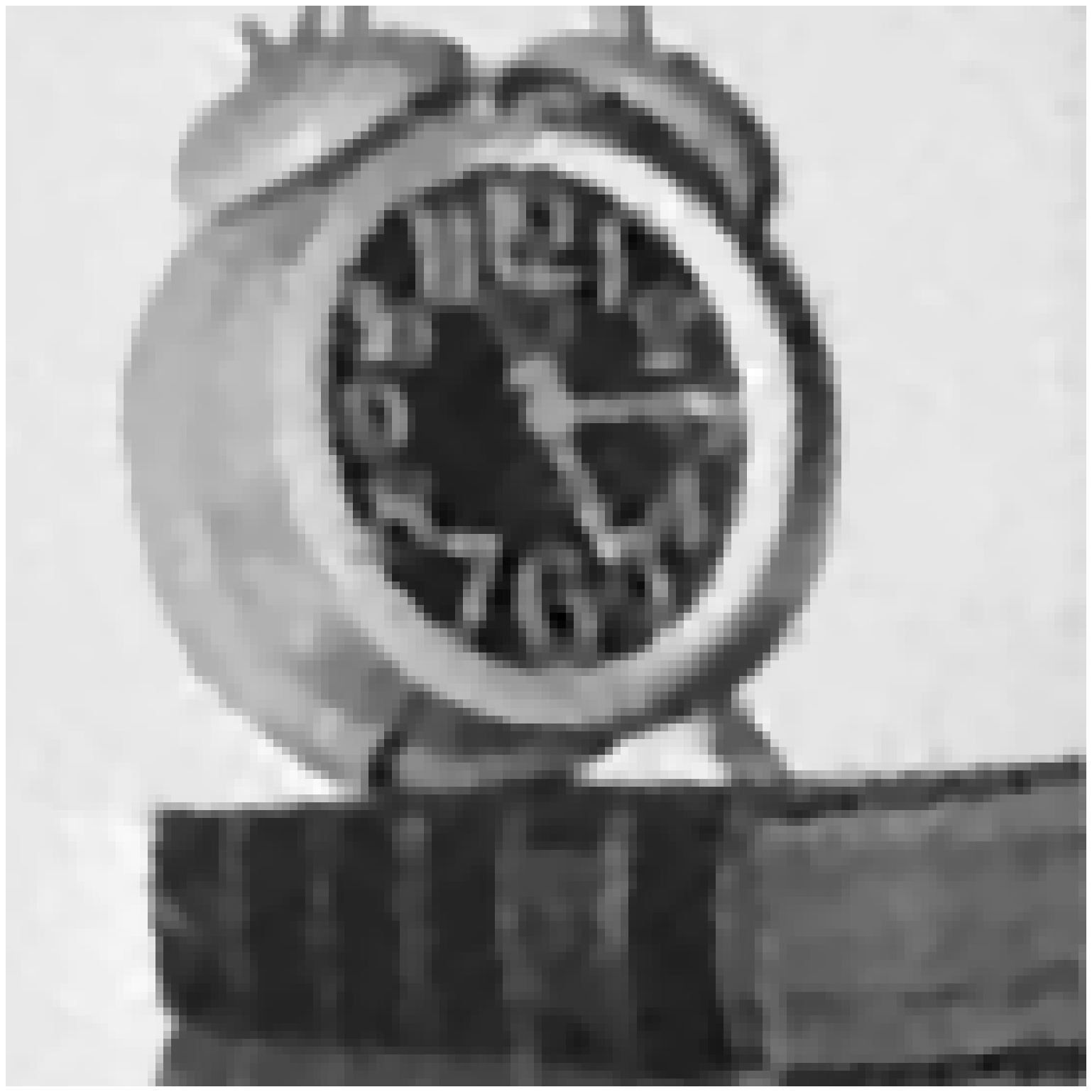}}
\subfigure{\includegraphics[width=31mm]{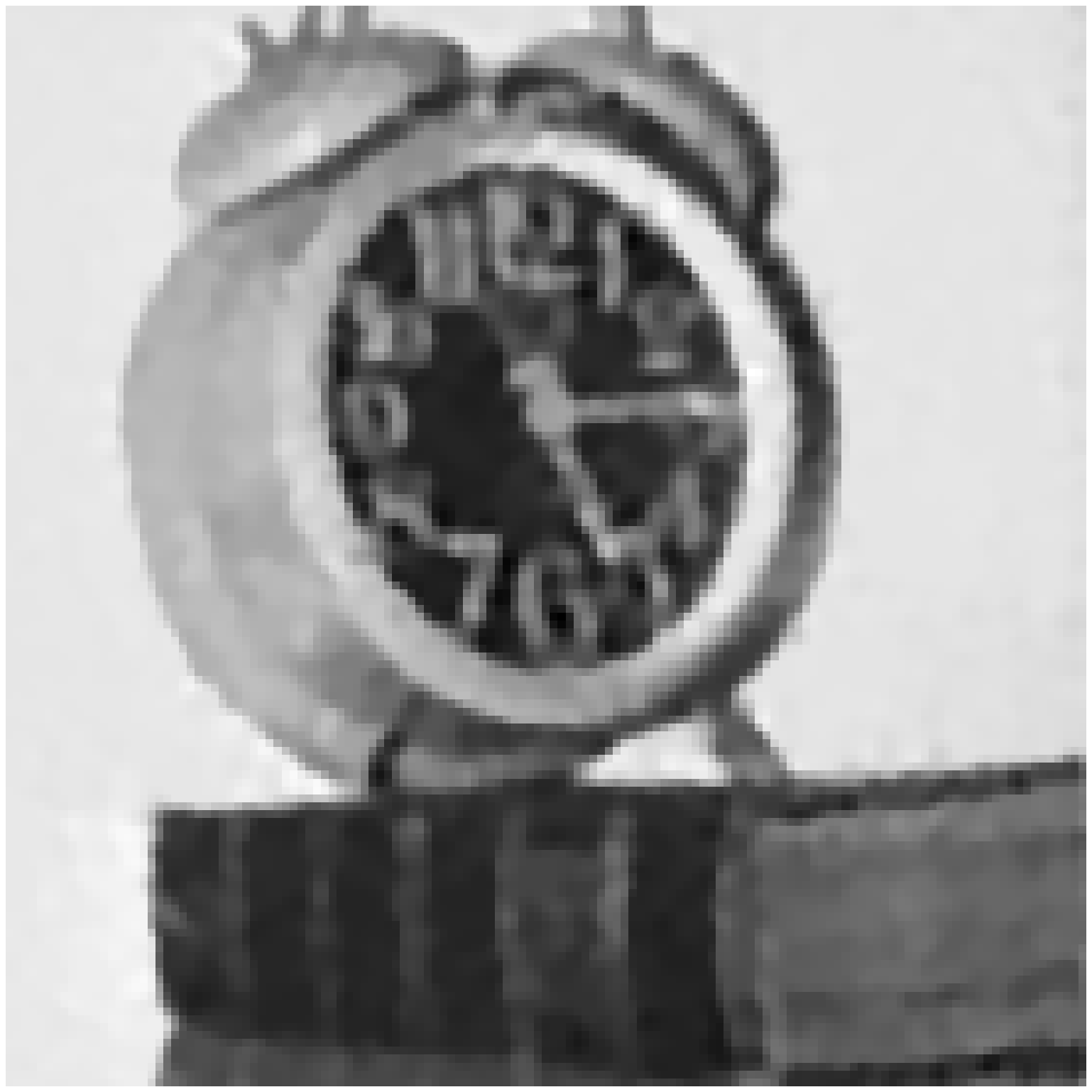}}
\subfigure{\includegraphics[width=31mm]{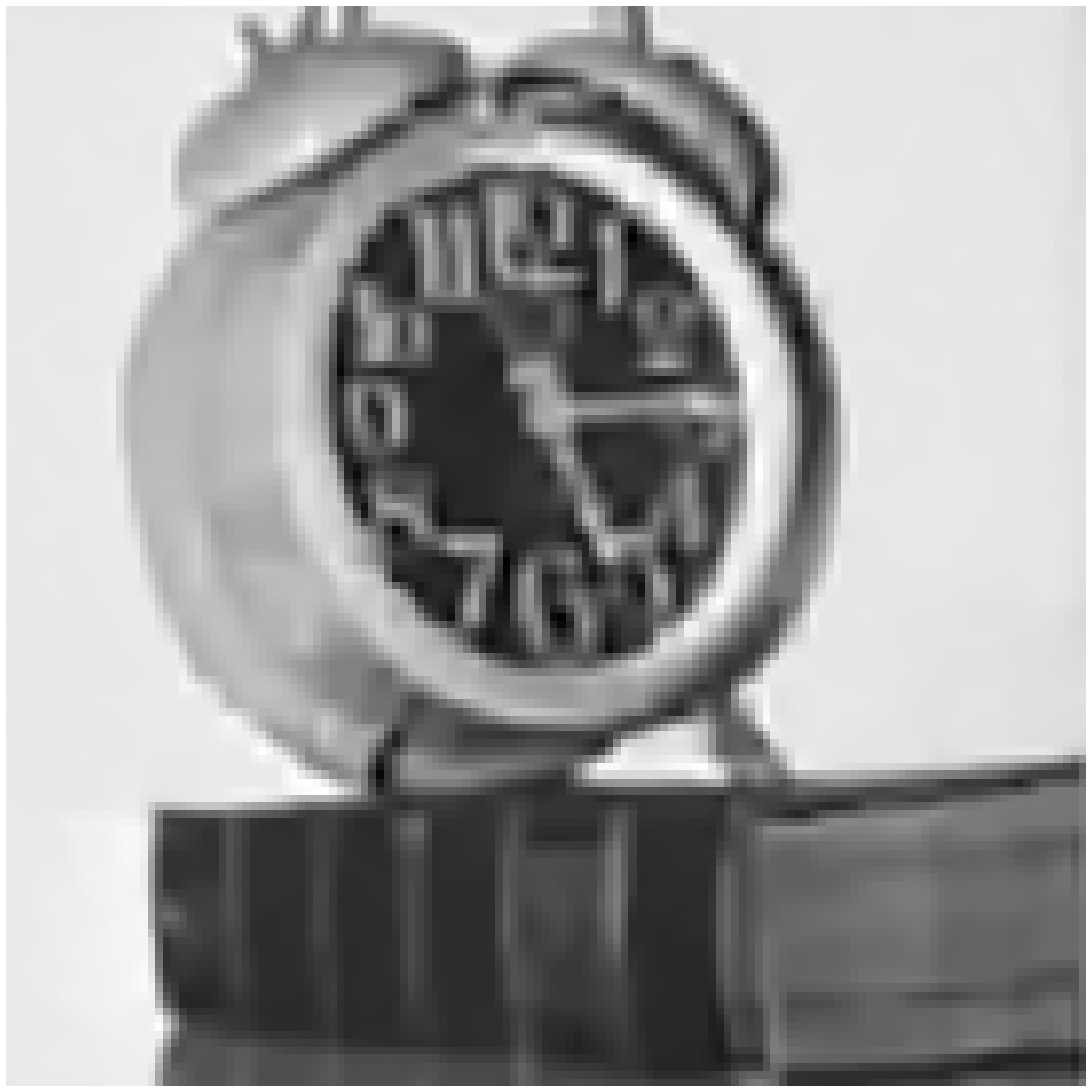}}
\\
\subfigure{\includegraphics[width=31mm]{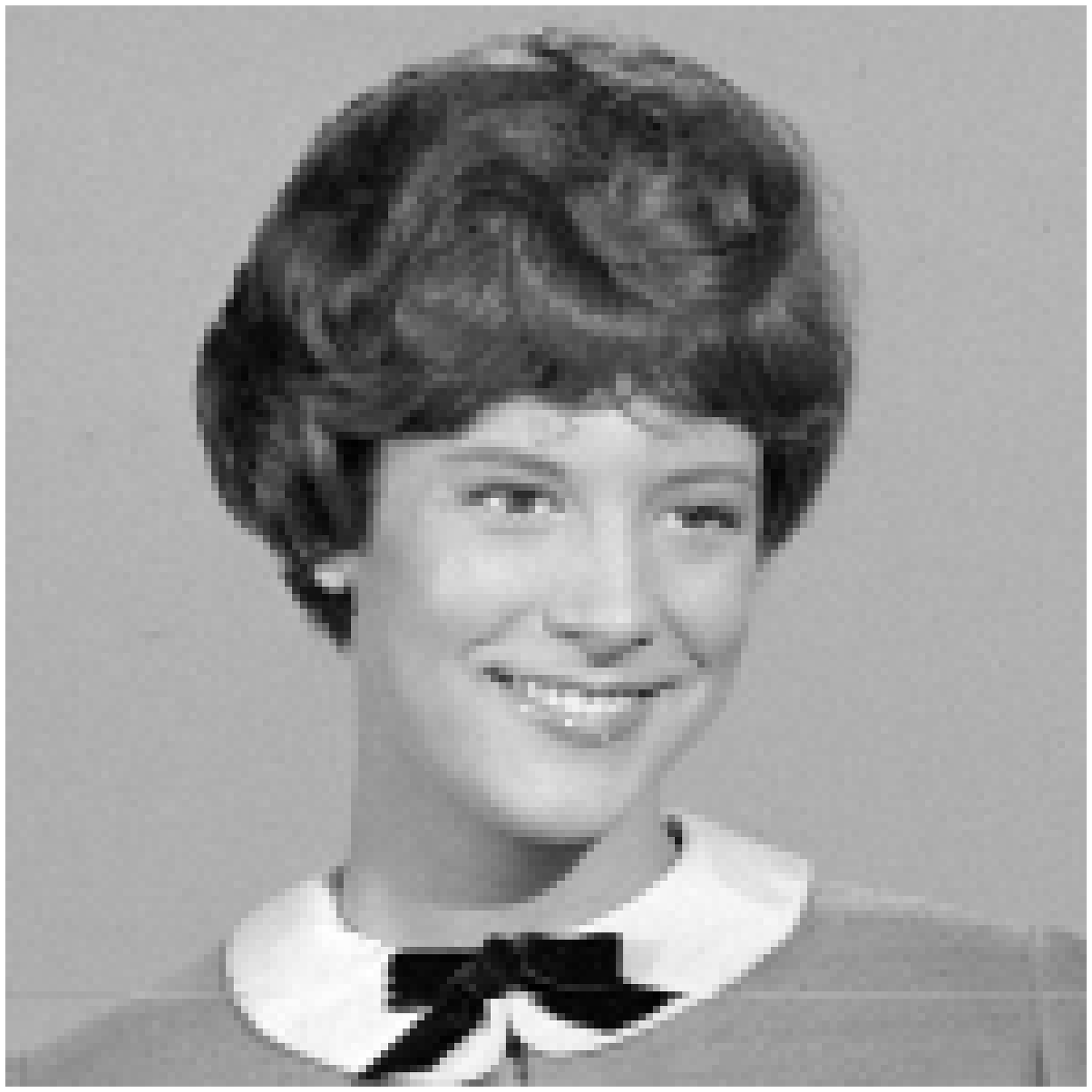}}
\subfigure{\includegraphics[width=31mm]{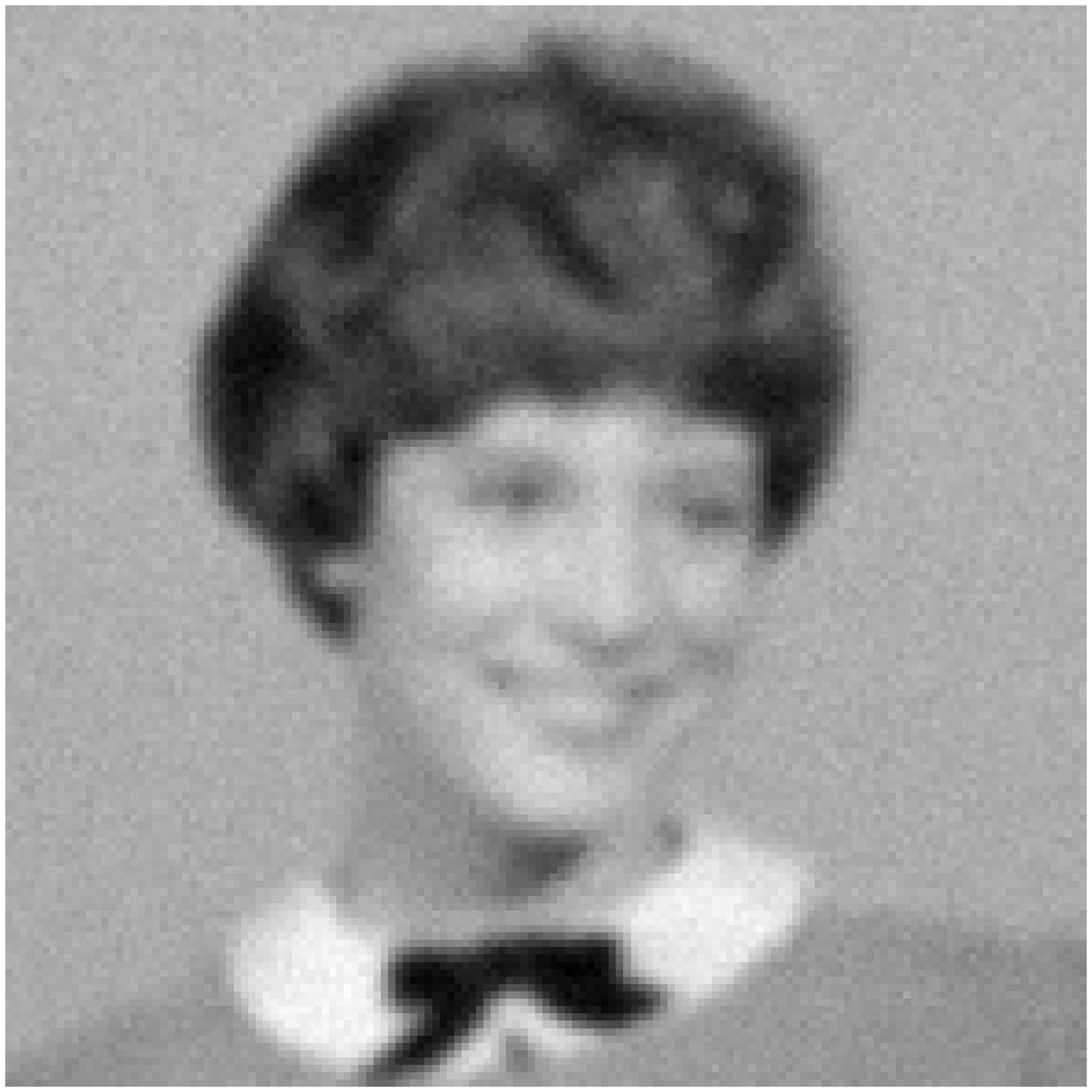}}
\subfigure{\includegraphics[width=31mm]{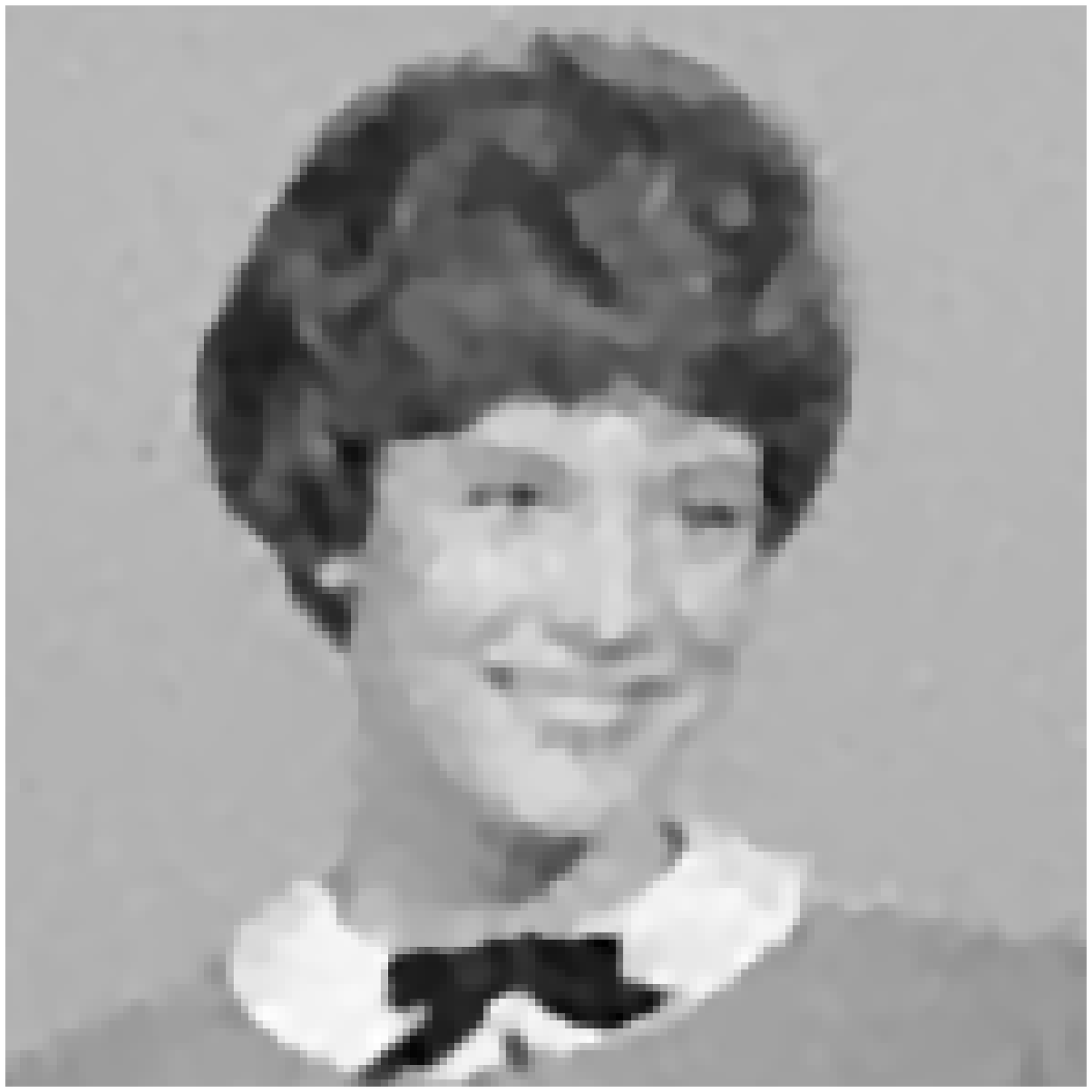}}
\subfigure{\includegraphics[width=31mm]{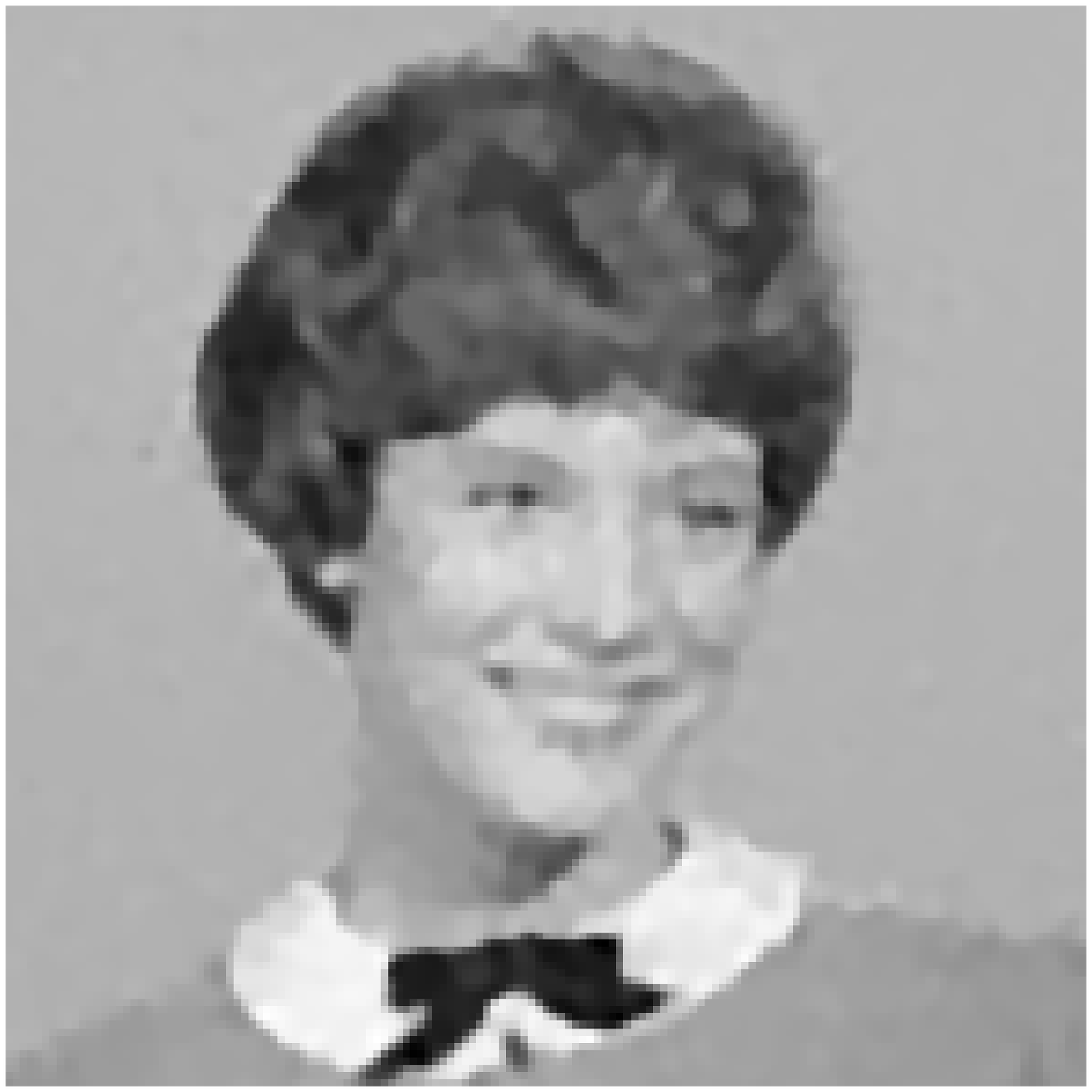}}
\subfigure{\includegraphics[width=31mm]{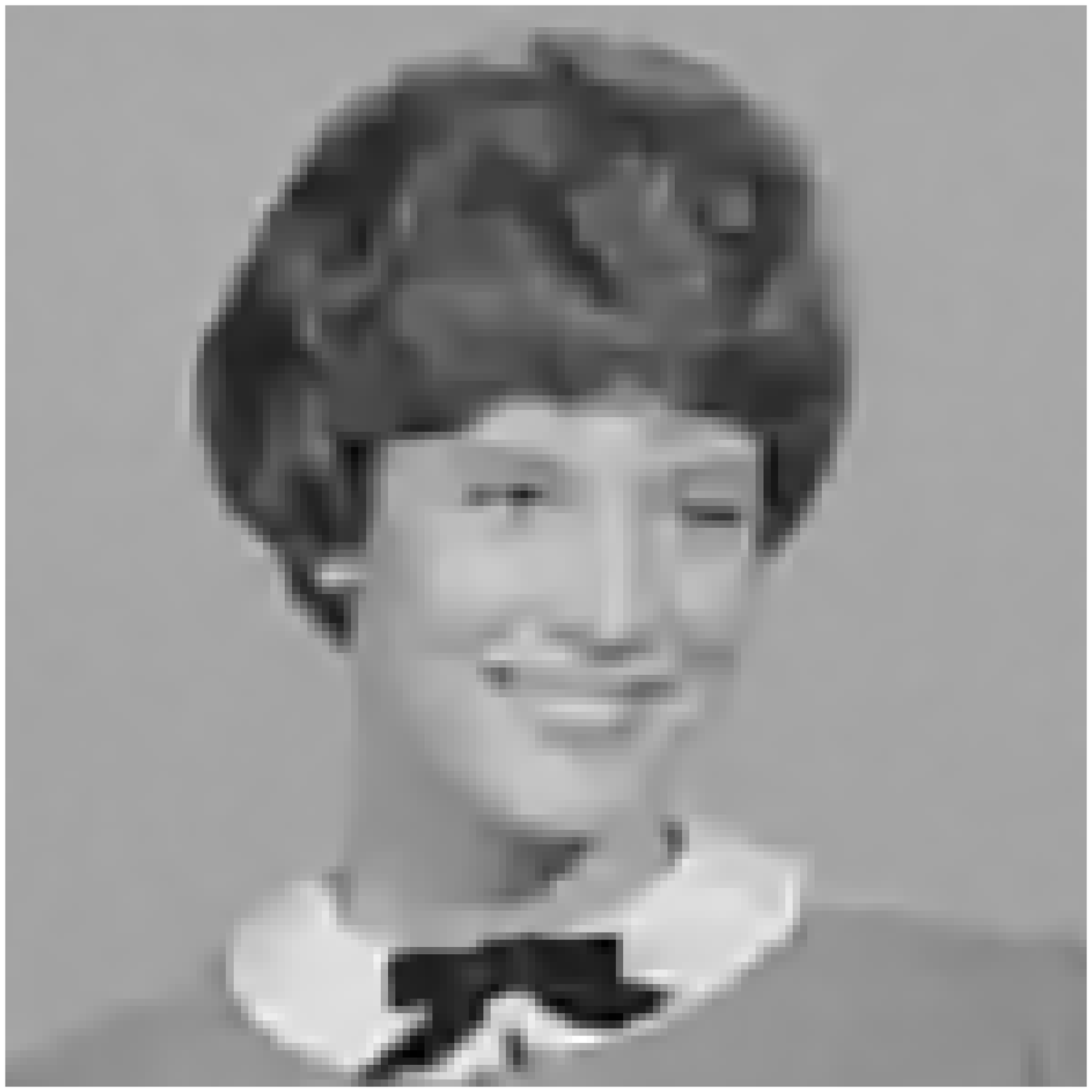}}
\\
\caption {Zoom-in to the texture part of ``downhill'', ``cameraman'',
``bridge'', ``pepper'', ``clock'', and ``portrait I''.
Image from left to right are: original image, observed image, results of
the balanced approach, results of the analysis based approach and
results of the PD method.} \label{image4}
\end{figure}

\begin{figure}[t!]
\centering
\subfigure{\includegraphics[width=31mm]{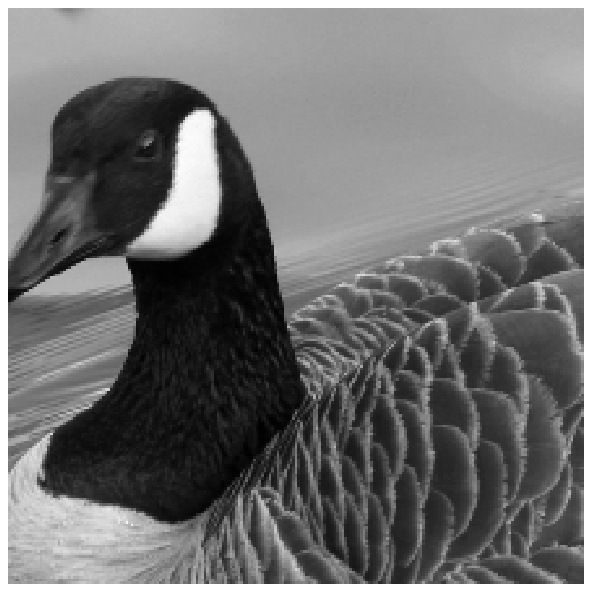}}
\subfigure{\includegraphics[width=31mm]{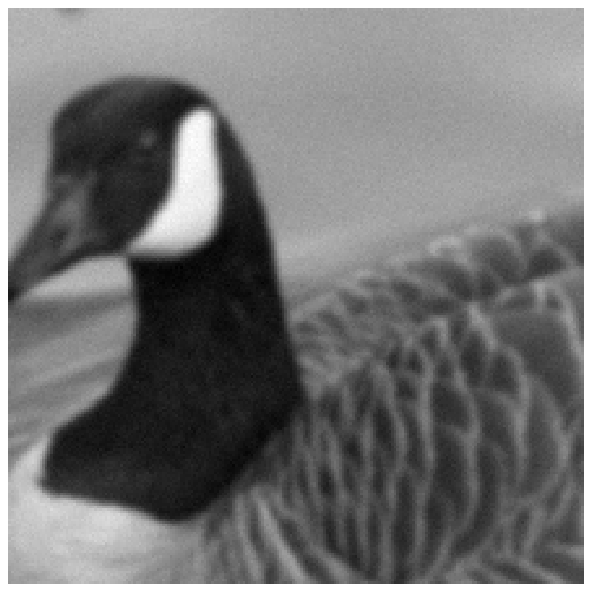}}
\subfigure{\includegraphics[width=31mm]{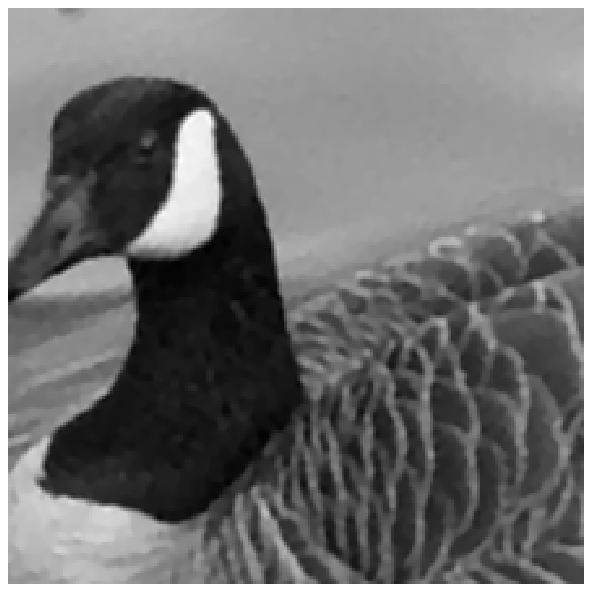}}
\subfigure{\includegraphics[width=31mm]{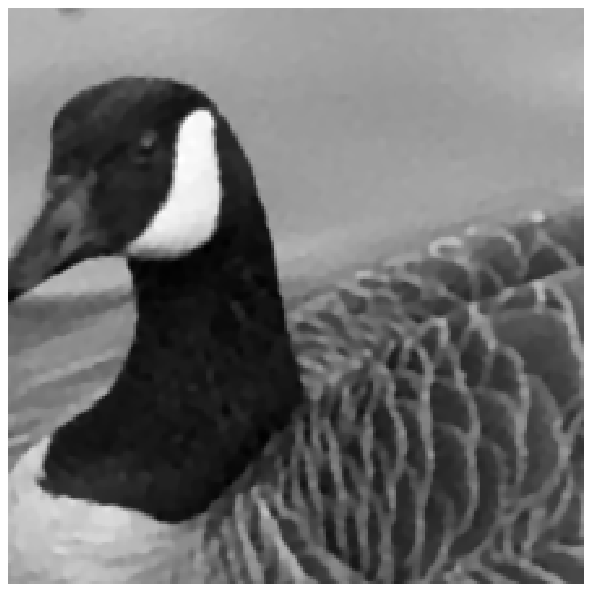}}
\subfigure{\includegraphics[width=31mm]{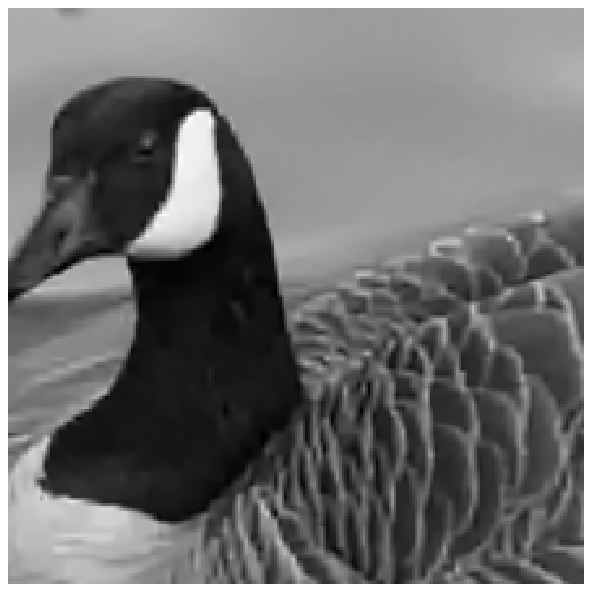}}
\\
\subfigure{\includegraphics[width=31mm]{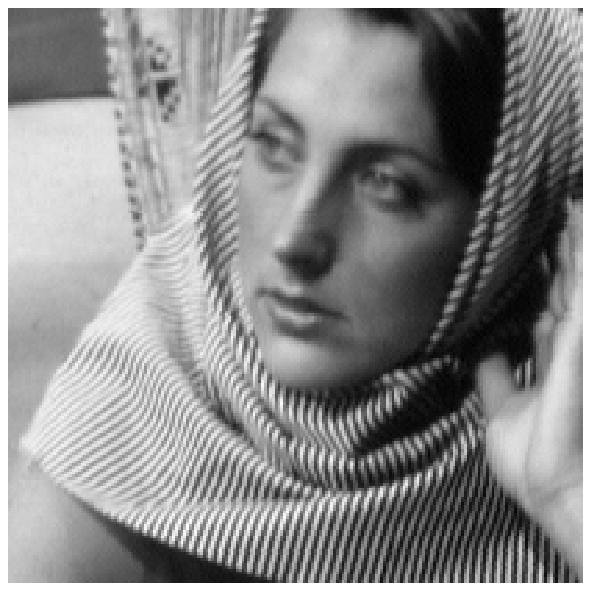}}
\subfigure{\includegraphics[width=31mm]{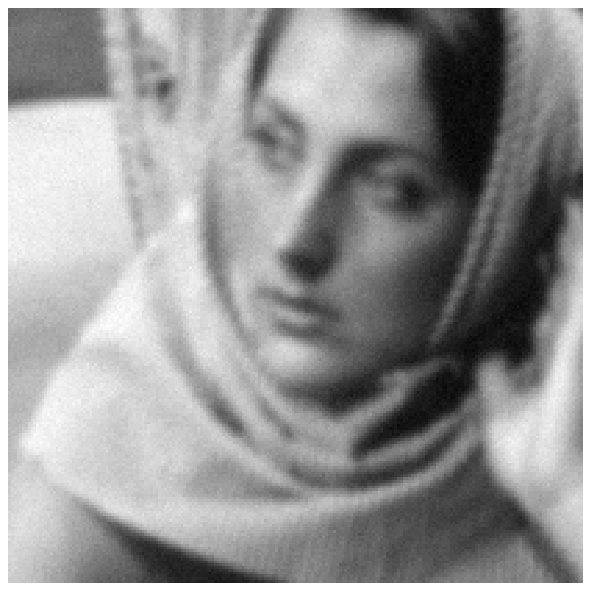}}
\subfigure{\includegraphics[width=31mm]{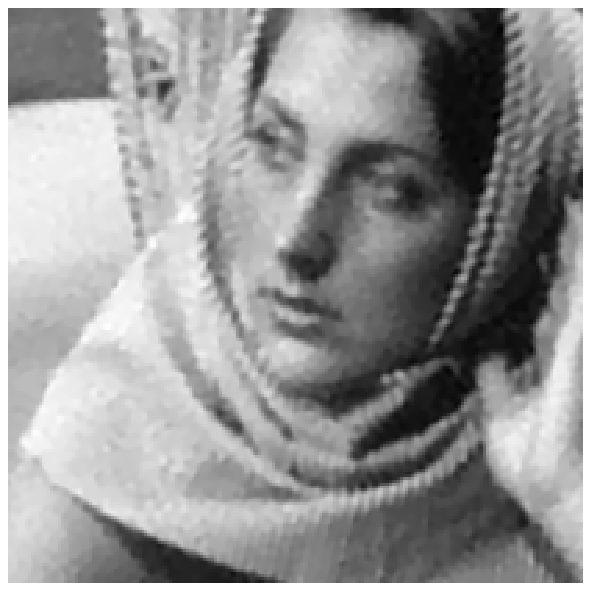}}
\subfigure{\includegraphics[width=31mm]{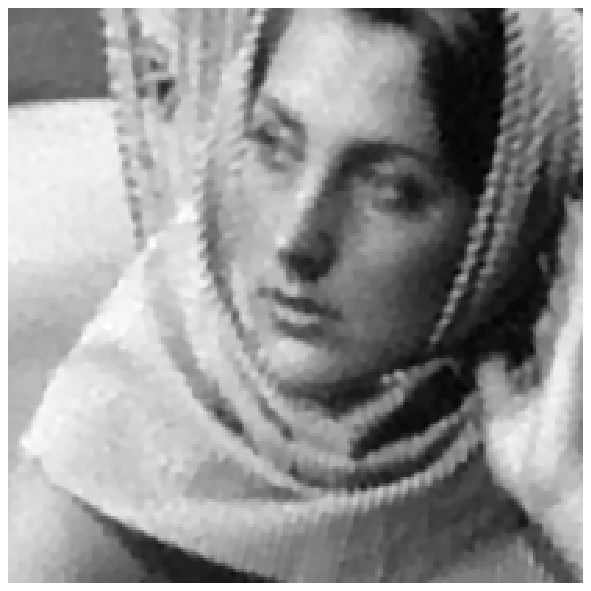}}
\subfigure{\includegraphics[width=31mm]{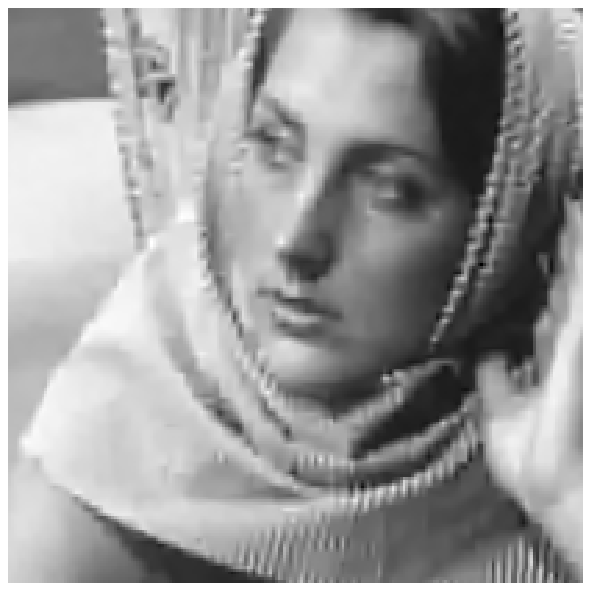}}
\\
\subfigure{\includegraphics[width=31mm]{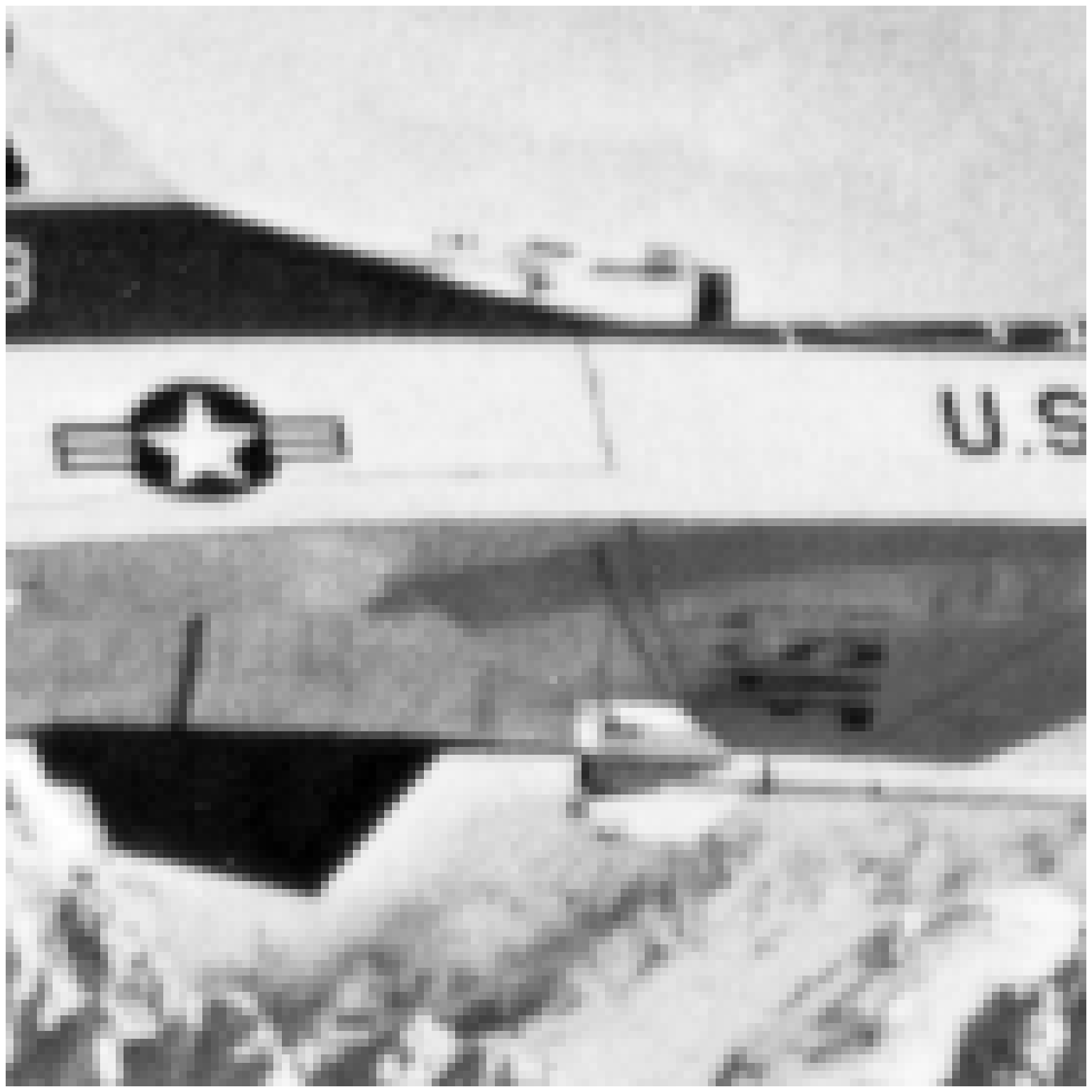}}
\subfigure{\includegraphics[width=31mm]{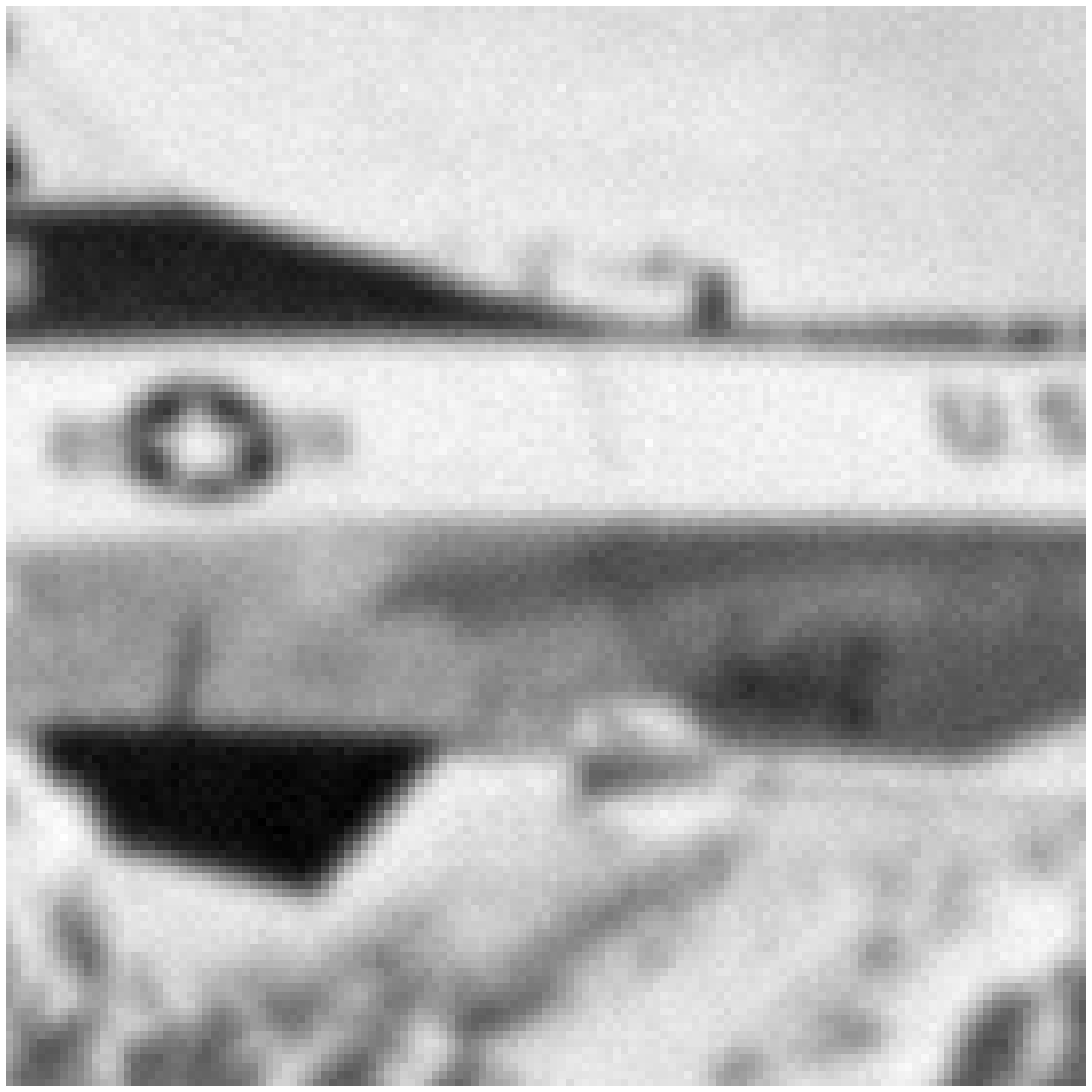}}
\subfigure{\includegraphics[width=31mm]{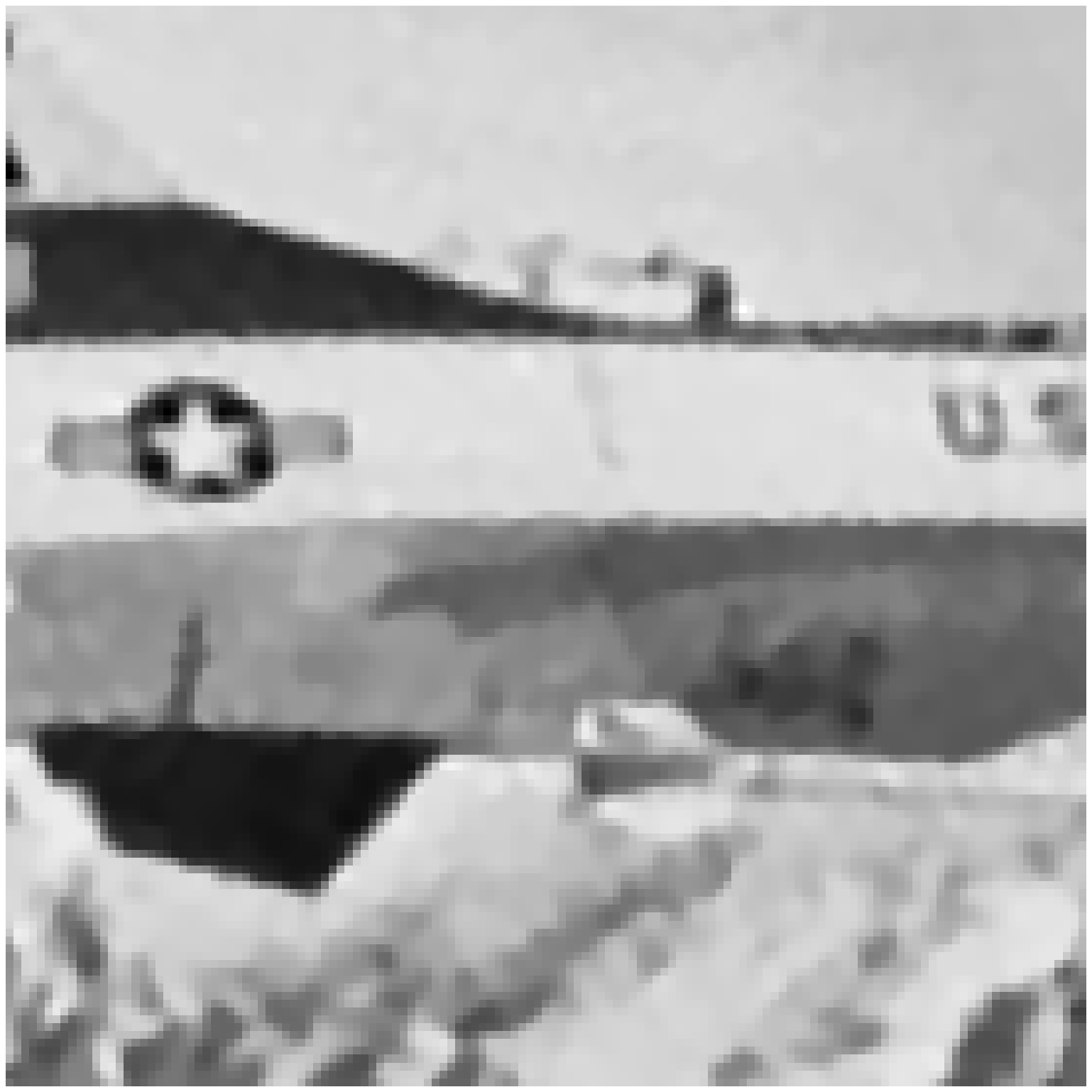}}
\subfigure{\includegraphics[width=31mm]{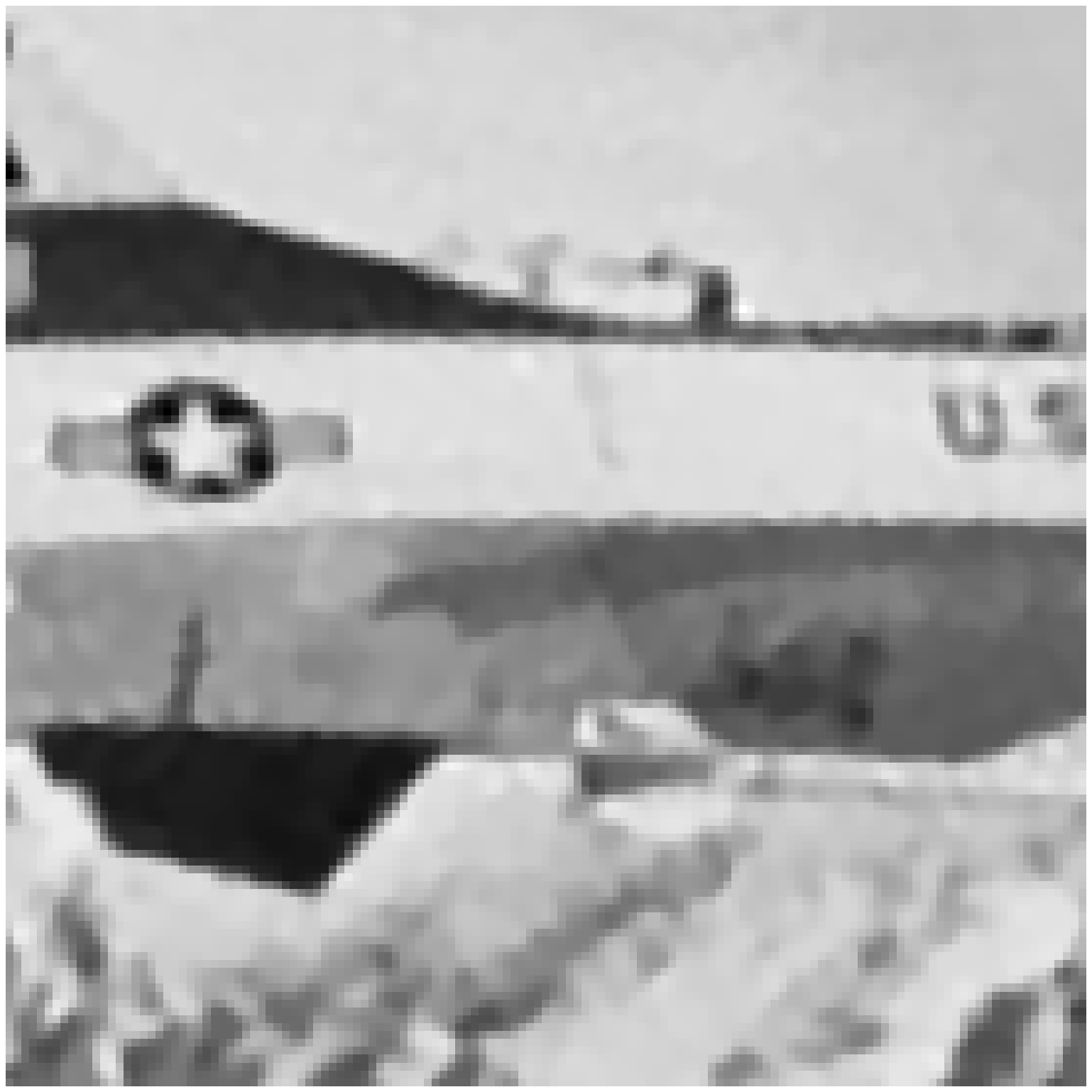}}
\subfigure{\includegraphics[width=31mm]{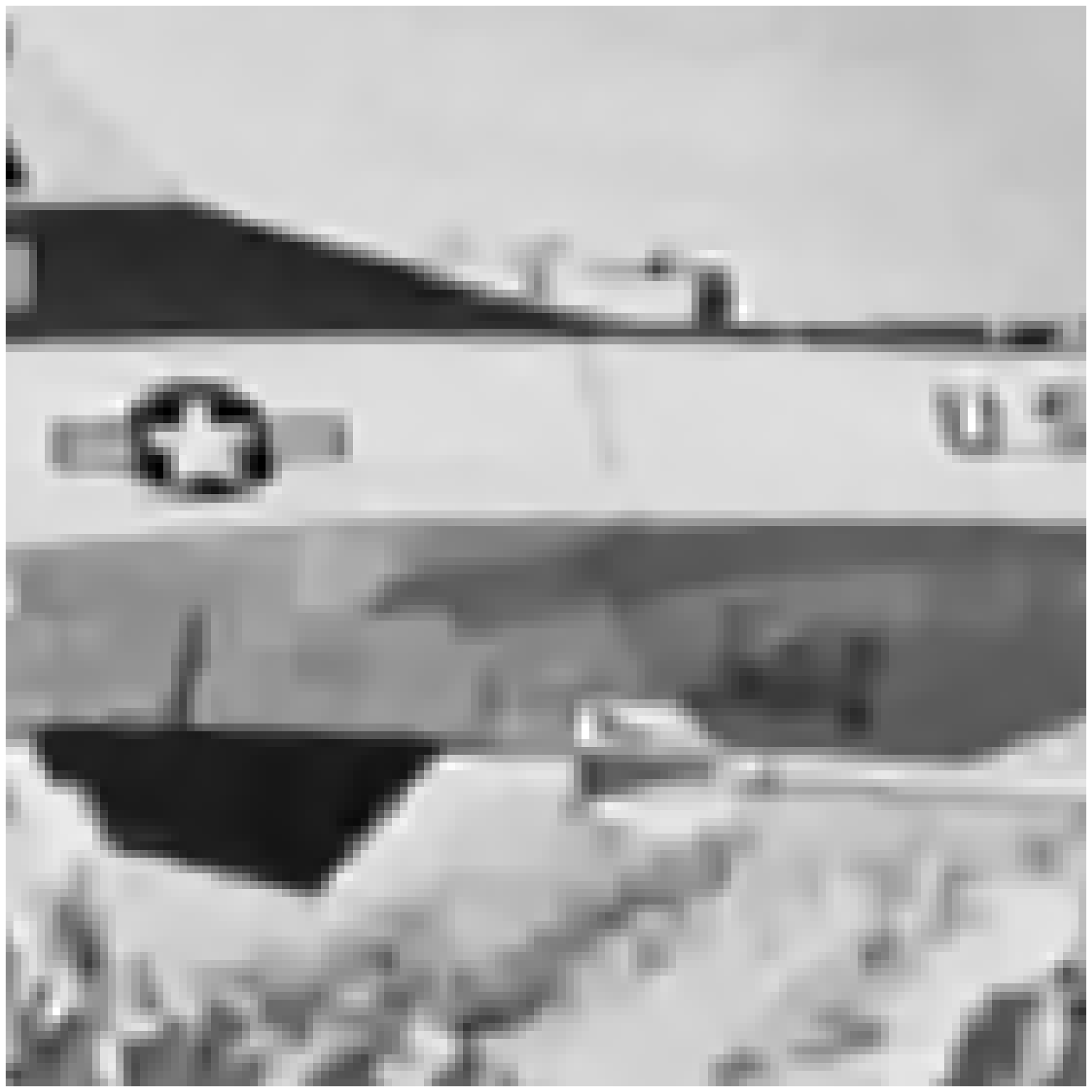}}
\\
\subfigure{\includegraphics[width=31mm]{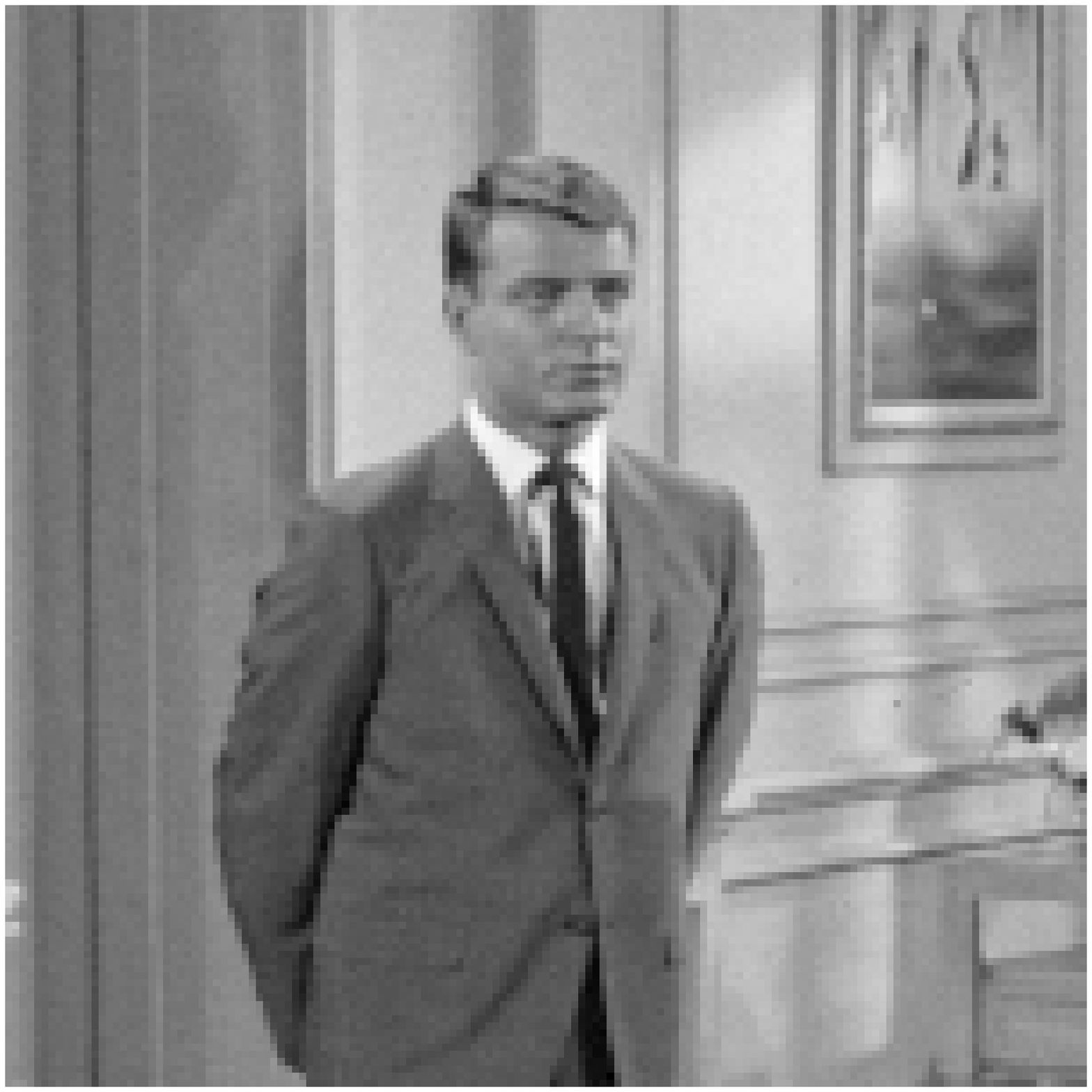}}
\subfigure{\includegraphics[width=31mm]{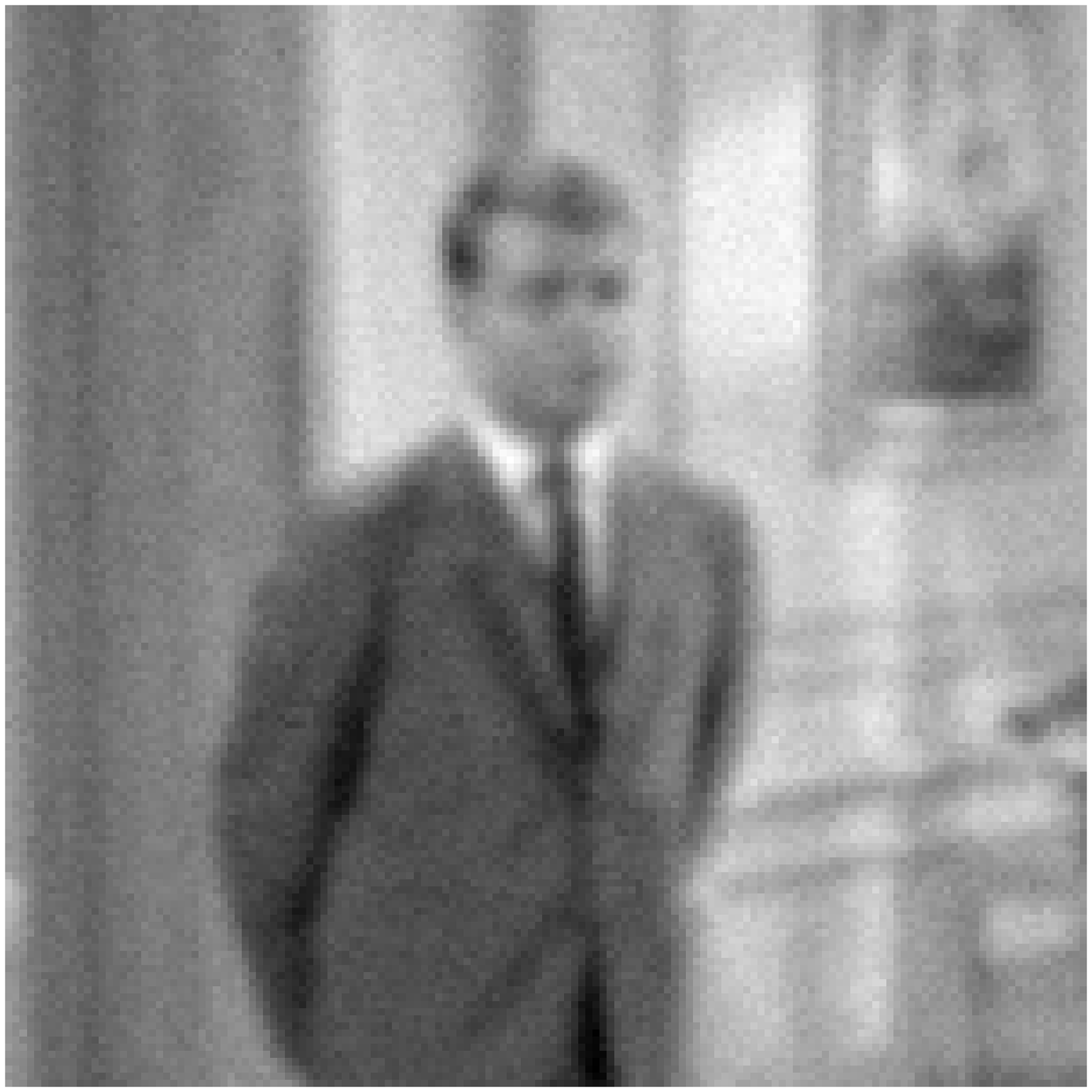}}
\subfigure{\includegraphics[width=31mm]{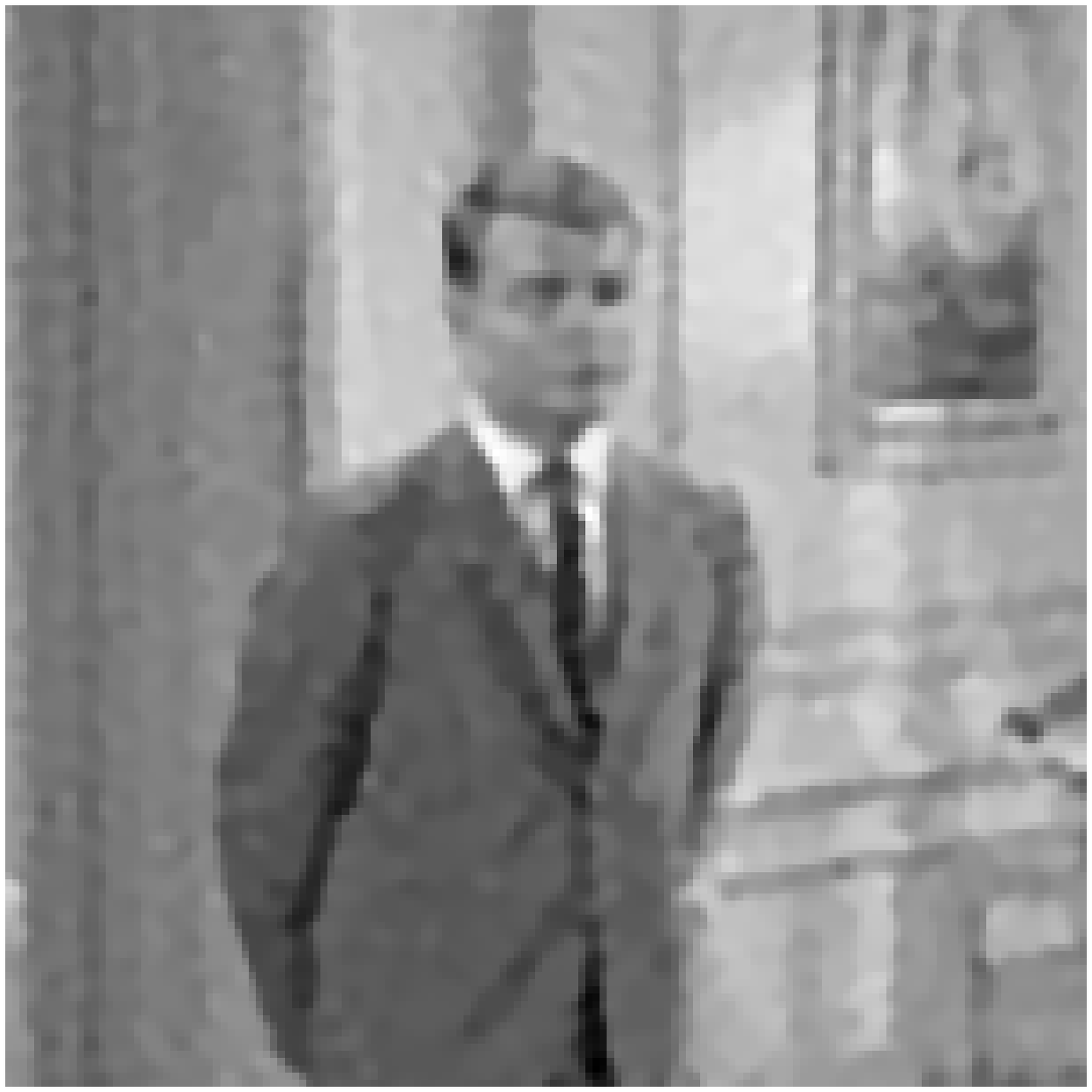}}
\subfigure{\includegraphics[width=31mm]{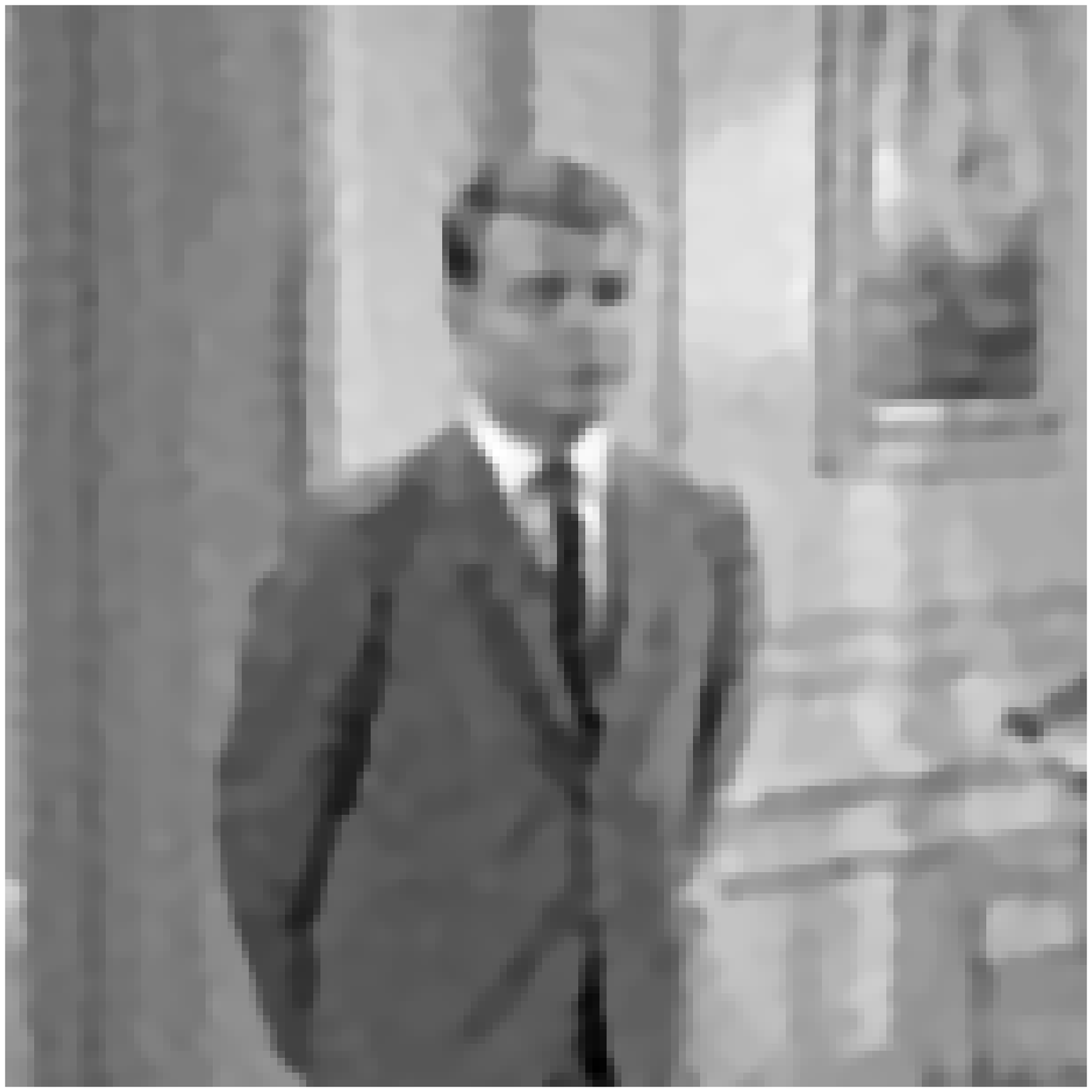}}
\subfigure{\includegraphics[width=31mm]{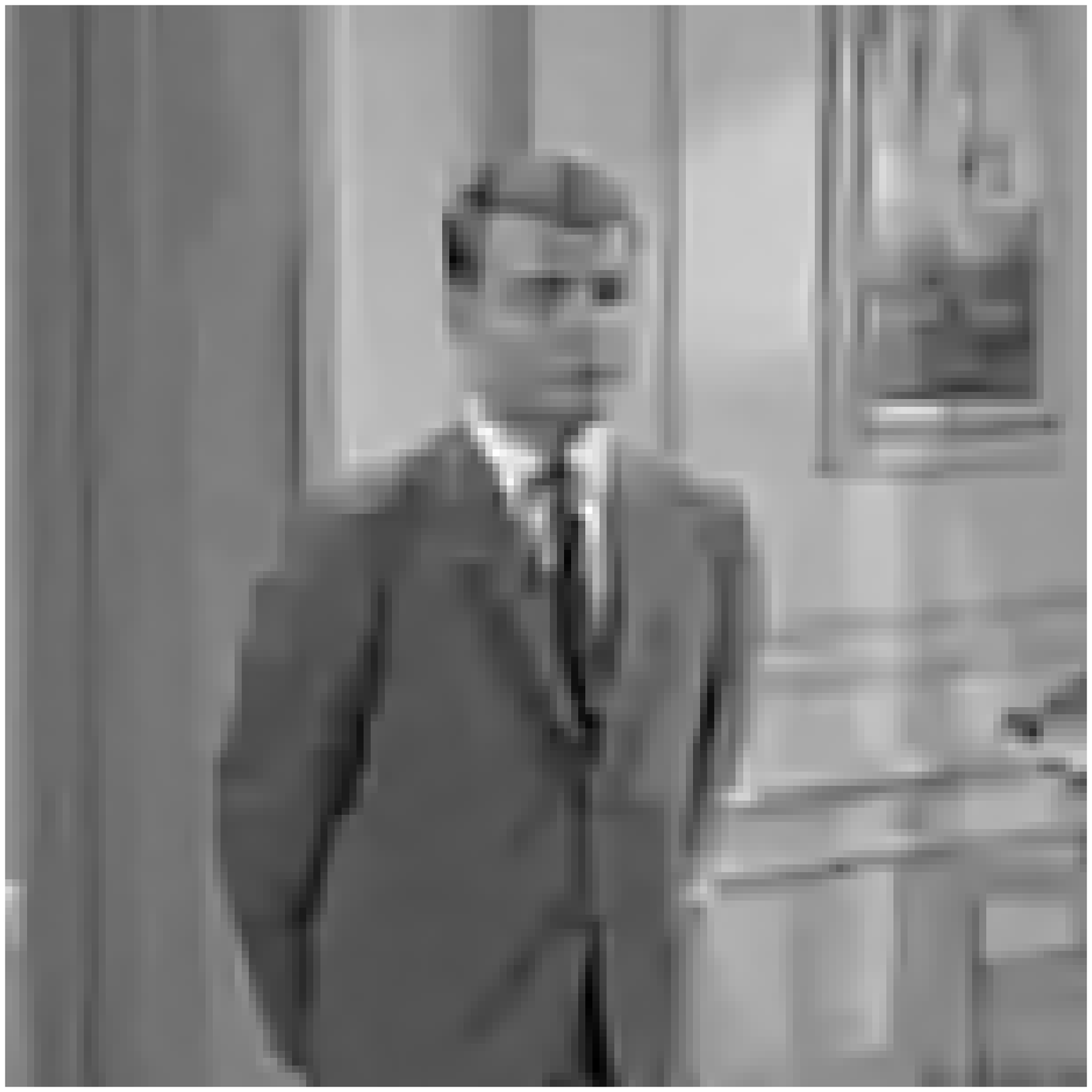}}
\\
\subfigure{\includegraphics[width=31mm]{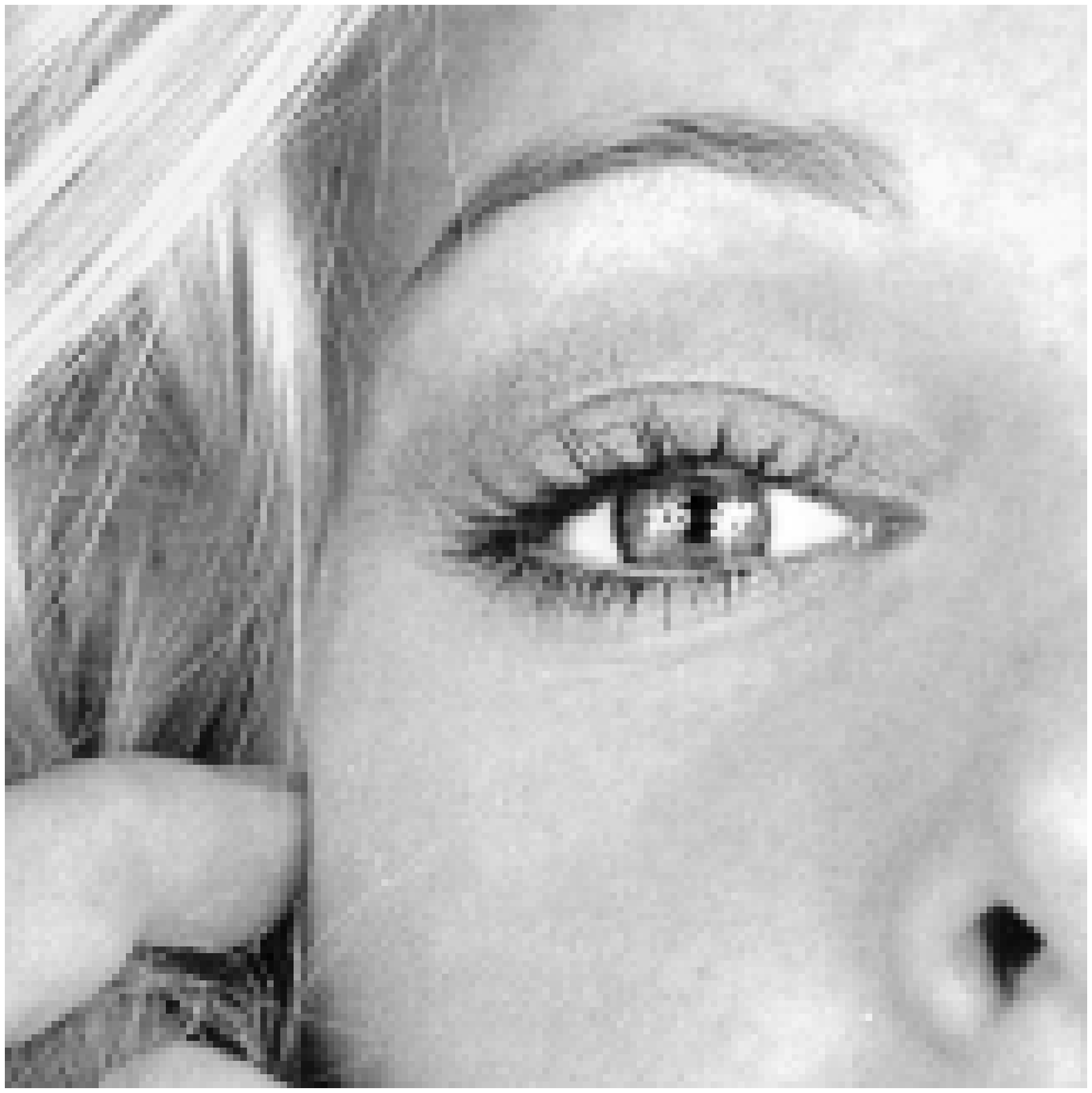}}
\subfigure{\includegraphics[width=31mm]{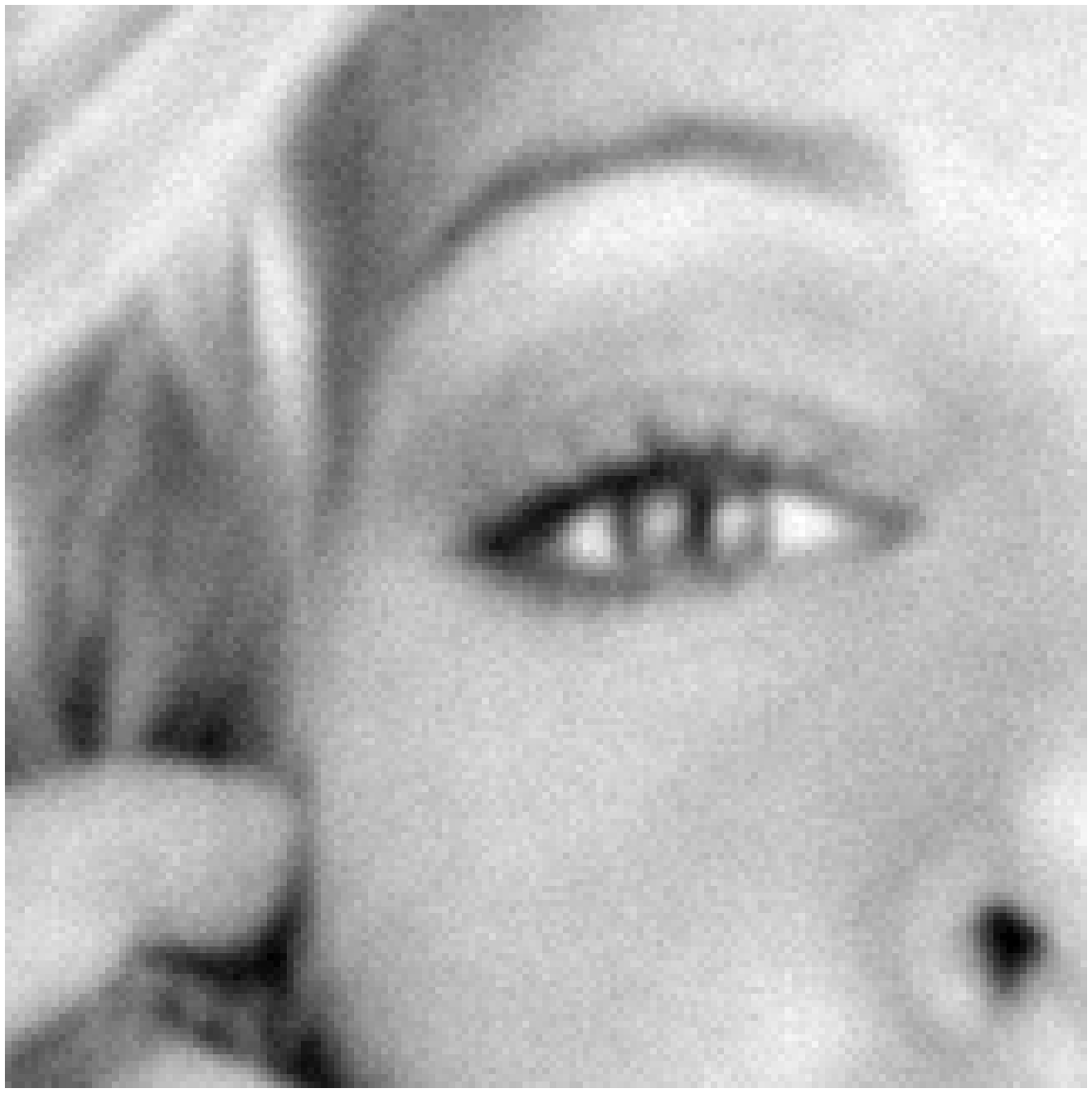}}
\subfigure{\includegraphics[width=31mm]{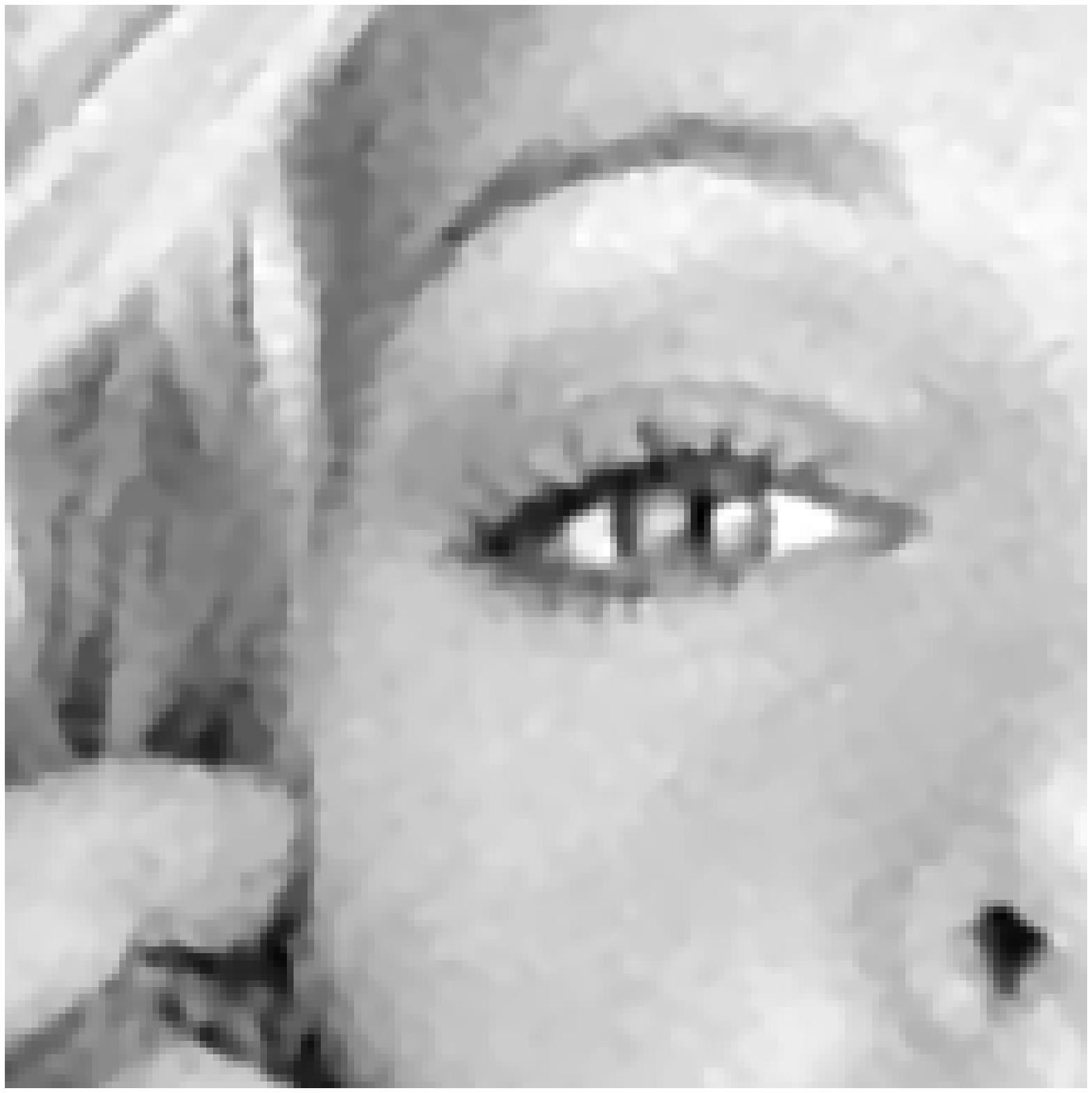}}
\subfigure{\includegraphics[width=31mm]{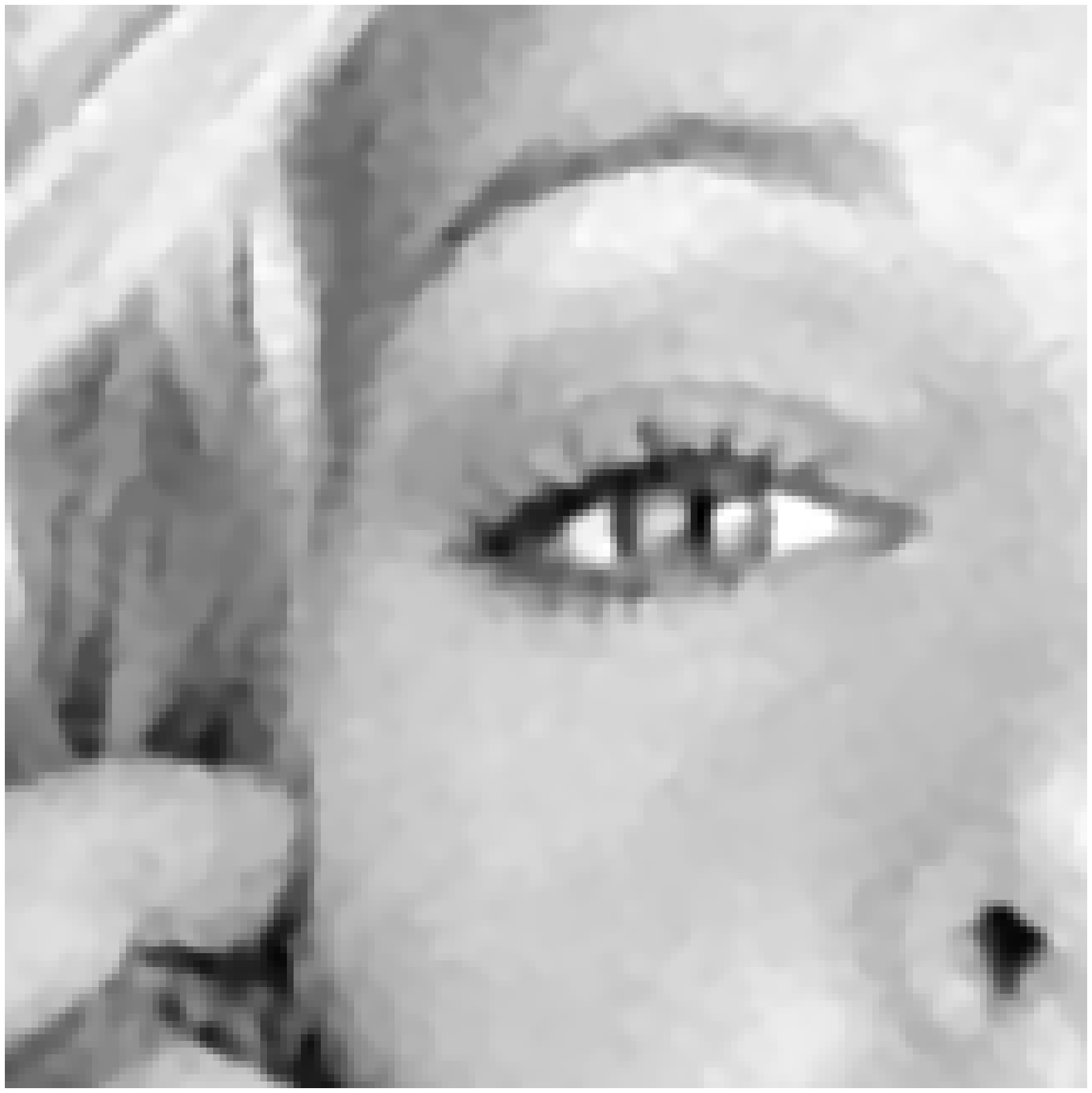}}
\subfigure{\includegraphics[width=31mm]{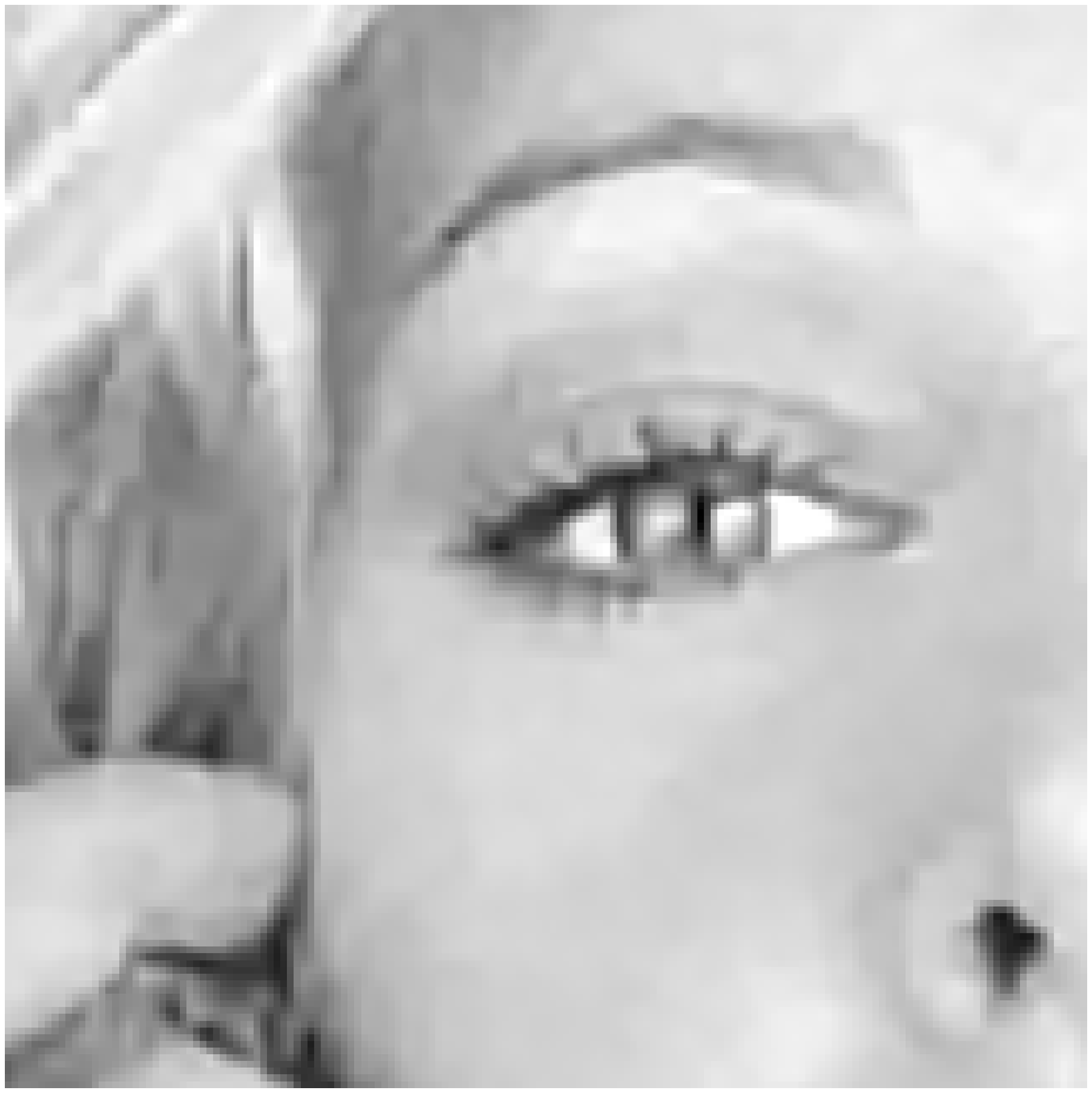}}
\\
\subfigure{\includegraphics[width=31mm]{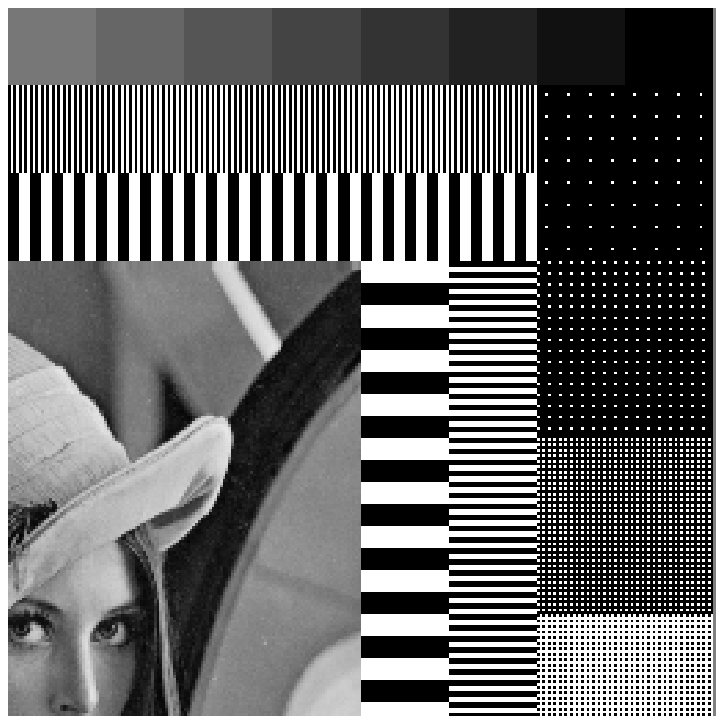}}
\subfigure{\includegraphics[width=31mm]{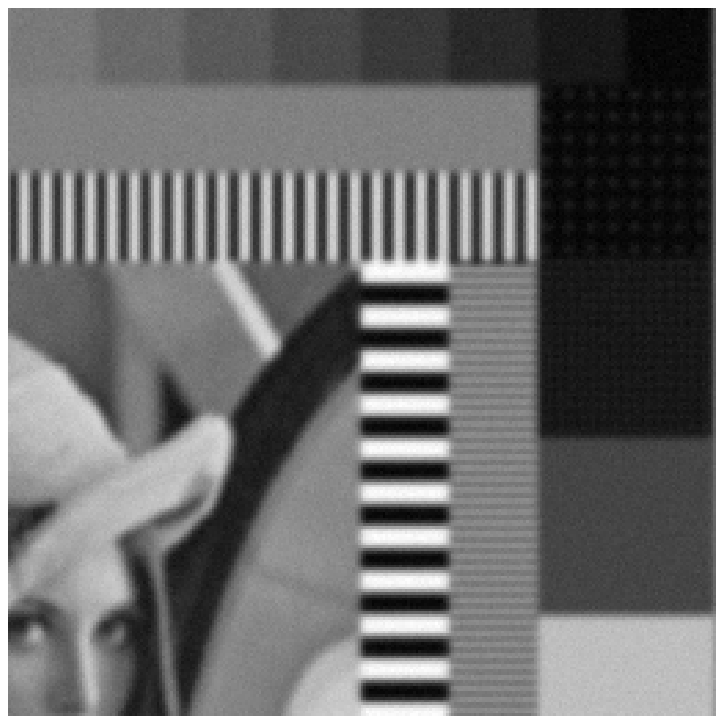}}
\subfigure{\includegraphics[width=31mm]{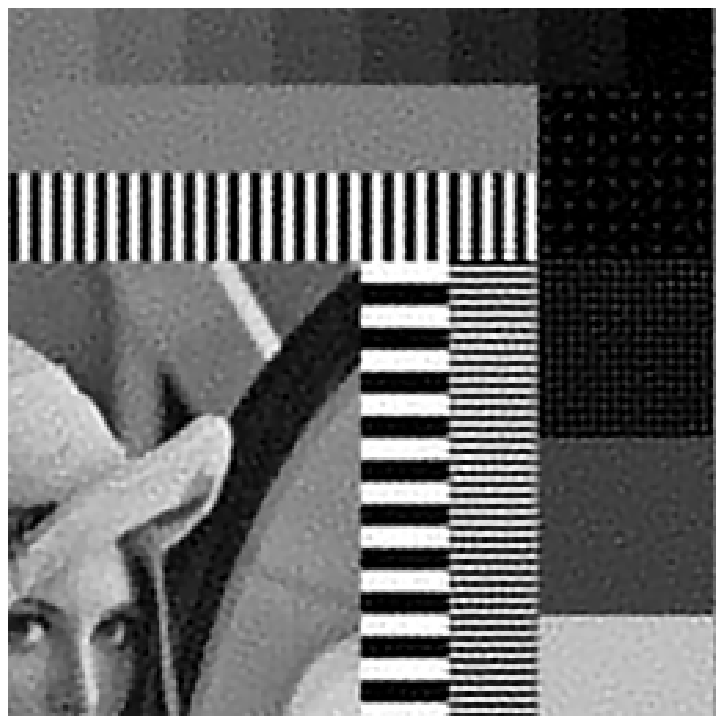}}
\subfigure{\includegraphics[width=31mm]{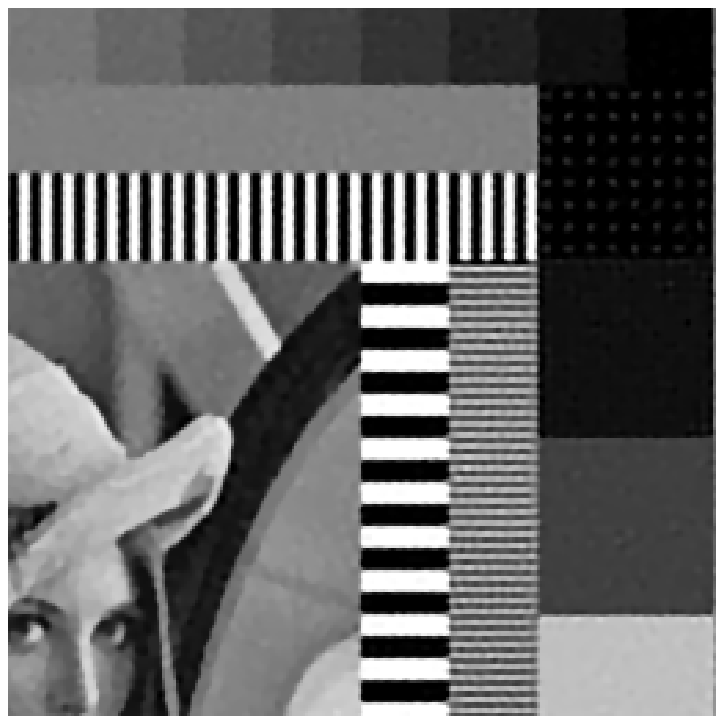}}
\subfigure{\includegraphics[width=31mm]{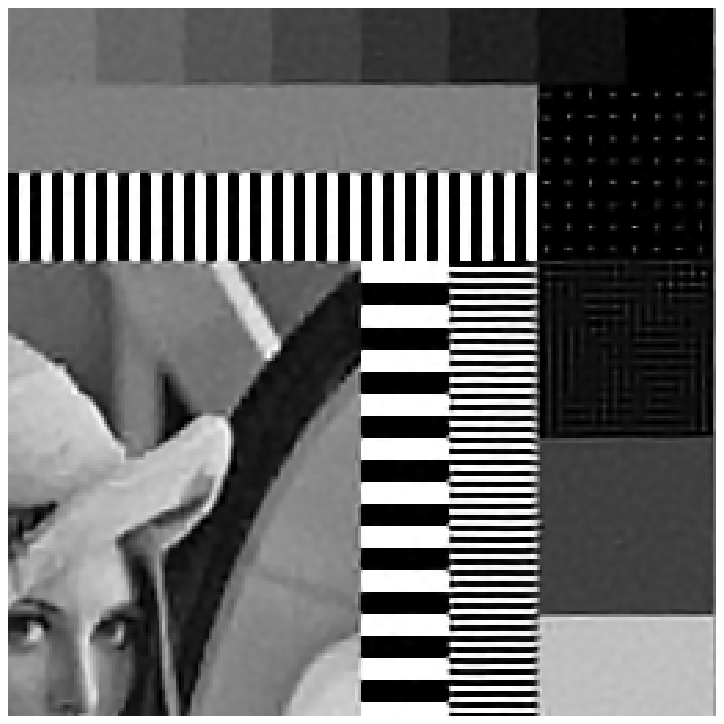}}
\\
\caption {Zoom-in to the texture part of ``duck'', ``barbara'',
``aircraft'', ``couple'', ``portrait II'' and ``lena''. Image
from left to right are: original image, observed image, results of
the balanced approach, results of the analysis based approach and
results of the PD method.} \label{image5}
\end{figure}

\begin{table}[t!]
\caption{Comparisons: image deconvolution} \centering
\label{table2}
\begin{tabular}{|l c||c c|c c|c c|}
\hline
\multicolumn{1}{|l}{} &\multicolumn{1}{c||}{} &
\multicolumn{2}{c|}{Balanced approach} & \multicolumn{2}{c|}{Analysis based approach} & \multicolumn{2}{c|}{PD method}
\\
\multicolumn{1}{|l}{Name} &\multicolumn{1}{c||}{Size} &
\multicolumn{1}{c}{Time} & \multicolumn{1}{c|}{PSNR} &
\multicolumn{1}{c}{Time} & \multicolumn{1}{c|}{PSNR} &
\multicolumn{1}{c}{Time} & \multicolumn{1}{c|}{PSNR}
\\
\hline
Downhill      & 256 & 12.5  & 27.24 & 6.1  & 27.36 & 29.5 & 27.35 \\
Cameraman & 256 & 18.2  & 26.65 & 7.0  & 26.73 & 31.1 & 27.21  \\
Bridge          & 256 & 14.5  & 25.40 & 5.1 & 25.46 & 33.0 & 25.44\\
Pepper         & 256 & 21.6  & 26.82 & 7.5  & 26.63 & 32.1 & 27.29 \\
Clock          & 256  & 17.3  & 29.42  & 19.9 & 29.48 & 22.3 & 29.86 \\
Portrait I     & 256  & 32.7 & 33.93 & 19.3  & 33.98 & 27.1 & 35.44   \\
Duck           & 464 & 30.6  & 31.00 & 16.1 & 31.11 & 72.5 & 31.09 \\
Barbara        & 512 & 38.8  & 24.62 & 12.3 & 24.62 & 77.4 & 24.69 \\
Aircraft       & 512  & 55.9  & 30.75 & 35.1 & 30.81 & 67.5 & 31.29 \\
Couple        & 512  & 91.4 & 28.40 & 41.5 & 28.14 & 139.1 & 29.32 \\
Portrait II     & 512  & 45.2 & 30.23 & 22.1 & 30.20 & 48.9 & 30.90 \\
Lena            & 516  & 89.3  & 12.91 & 31.0 & 12.51 & 67.0 & 13.45 \\
\hline
\end{tabular}
\\
\end{table}

\begin{table}[t!]
\caption{Comparisons among different wavelet representations} \centering
\label{table3}
\begin{tabular}{|l||c c|c c|c c|}
\hline
\multicolumn{1}{|l||}{} &
\multicolumn{2}{c|}{Balanced approach} & \multicolumn{2}{c|}{Analysis based approach} & \multicolumn{2}{c|}{PD method}
\\
\multicolumn{1}{|l||}{Wavelets}  &
\multicolumn{1}{c}{Time} & \multicolumn{1}{c|}{PSNR} &
\multicolumn{1}{c}{Time} & \multicolumn{1}{c|}{PSNR} &
\multicolumn{1}{c}{Time} & \multicolumn{1}{c|}{PSNR}
\\
\hline
Haar                   & 17.9  & 33.63  & 20.2  & 33.80 & 24.3 & 34.68 \\
Piecewise linear & 32.7 & 33.93 & 22.3  & 33.98 & 27.1 & 35.44   \\
Piecewise cubic & 61.0 & 33.95  & 37.3 & 34.00 & 37.8 & 35.20\\

\hline
\end{tabular}
\\
\end{table}

\begin{table}[t!]
\caption{Comparisons among different noise levels for image ``Portrait I"}
\centering \label{table4}
\begin{tabular}{|c||c c|c c|c c|}
\hline
\multicolumn{1}{|c||}{} &
\multicolumn{2}{c|}{Balanced approach} & \multicolumn{2}{c|}{Analysis based approach} & \multicolumn{2}{c|}{PD method}
\\
\multicolumn{1}{|l||}{Variances of noises}  &
\multicolumn{1}{c}{Time} & \multicolumn{1}{c|}{PSNR} &
\multicolumn{1}{c}{Time} & \multicolumn{1}{c|}{PSNR} &
\multicolumn{1}{c}{Time} & \multicolumn{1}{c|}{PSNR}
\\
\hline
$\sigma =3$ & 32.7 & 33.93 & 22.3  & 33.98 & 27.1 & 35.44   \\
$\sigma =5$ & 23.7 & 32.84 & 19.4  & 32.89 & 27.2 & 34.48   \\
$\sigma =7$ & 19.6 & 32.11 & 25.0  & 32.14 & 29.7 & 33.69   \\
\hline
\end{tabular}
\\
\end{table}

\section{Conclusion}
\label{conclude}

In this paper, we proposed a wavelet frame based $\ell_0$
minimization model, which is motivated by the analysis based
approach and balanced approach. The penalty decomposition (PD)
method of \cite{LuZhangTech2010} was used to solve the proposed
optimization problem. Numerical results showed that the proposed
model solved by the PD method can generate images with better
quality than those obtained by either analysis based approach or
balanced approach in terms of restoring sharp features like edges
as well as maintaining smoothness of the recovered images.
Convergence analysis of the sub-iterations in the PD method was
also provided.

\section*{Acknowledgement}
The first author would like to thank Ting Kei Pong for his helpful
discussions on efficiently solving the first subproblem arising in
the BCD method and the convergence results of the BCD method.

\bibliographystyle{plain}
\bibliography{ReferenceLibrary}
\end{document}